	\documentclass[a4paper]{article}
		\usepackage[a4paper, total={15.7cm, 22cm}]{geometry}
	\usepackage[utf8]{inputenc}
	\usepackage[pdftex,colorlinks=true,linkcolor=blue,citecolor=magenta]{hyperref} 
	\usepackage[dvipsnames]{xcolor}

	\usepackage{epsfig}
\usepackage{amssymb}
\usepackage{natbib}
\usepackage{graphicx}

\newcommand{\BlackBox}{\rule{1.5ex}{1.5ex}}  
\ifdefined\proof
\renewenvironment{proof}{\par\noindent{\bf Proof\ }}{\hfill\BlackBox\\[2mm]}
\else
\newenvironment{proof}{\par\noindent{\bf Proof\ }}{\hfill\BlackBox\\[2mm]}
\fi
\newtheorem{example}{Example} 
\newtheorem{theorem}{Theorem}
\newtheorem{lemma}[theorem]{Lemma} 
\newtheorem{proposition}[theorem]{Proposition} 
\newtheorem{remark}[theorem]{Remark}
\newtheorem{corollary}[theorem]{Corollary}
\newtheorem{definition}[theorem]{Definition}

\usepackage{mystyle}

	\usepackage{authblk}
\author[1,2]{Andrea Della Vecchia}
\author[2]{Ernesto De Vito}
\author[3]{Jaouad Mourtada}
\author[4,5,6]{Lorenzo Rosasco}
\affil[1]{EPFL and Swiss Finance Institute, Lausanne, Switzerland}
\affil[2]{MaLGa Center - DIMA - Università di Genova, Italy}
\affil[3]{ CREST, ENSAE - Institut Polytechnique de Paris, France}
\affil[4]{MaLGa Center - DIBRIS, Università di Genova, Genoa, Italy}
\affil[5]{Center for Brains, Minds and Machines, MIT, Cambridge, MA, USA}
\affil[6]{Istituto Italiano di Tecnologia, Genoa, Italy}

	\date{}
\begin{document}
\title{The Nystr\"om method for convex loss functions}
%
%

\maketitle

\begin{abstract}

We investigate an extension of classical empirical risk minimization, where the hypothesis space consists of a random subspace within a given Hilbert space. Specifically, we examine the Nystr\"om method where the  subspaces are defined by a random subset of the data. This approach recovers Nystr\"om approximations used in  kernel methods as a specific case. Using random subspaces naturally leads to computational advantages, but a key question is whether it compromises the learning accuracy. Recently, the tradeoffs between statistics and computation have been explored for the square loss and self-concordant losses, such as the logistic loss. In this paper, we extend these analyses to general convex Lipschitz losses, which may lack smoothness, such as the hinge loss used in support vector machines. Our main results show the existence of various scenarios where computational gains  can be achieved without sacrificing learning performance. When specialized to smooth loss functions, our analysis recovers most previous results. Moreover, it allows to consider classification problems and translate the surrogate risk bounds into classification error bounds. Indeed, this gives the opportunity to compare the effect of Nystr\"om approximations when combined  with different loss functions such as the  hinge or the  square loss. 
\end{abstract}

\tableofcontents

\section{Introduction}
Despite excellent practical performances, state of the art machine learning (ML) methods often require huge computational resources, motivating the search for  more efficient solutions. 
This has led to a number of new results in optimization \citep{johnson2013accelerating,schmidt2017minimizing}, as well as  the development of 
approaches mixing linear algebra and randomized algorithms \citep{mahoney2011randomized,drineas2005nystrom,woodruff2014sketching,calandriello2017distributed}.

While these techniques are applied to empirical objectives,  in the context of learning it is natural to study how different numerical solutions affect statistical accuracy. 
Interestingly, it is now clear  that  there is a whole set of problems and approaches where computational 
savings do not lead to any degradation in terms of learning performance 
\citep{rudi2015less,bach2017equivalence,bottou2008tradeoffs,sun2018but,li2019towards,rudi2017generalization,calandriello2018statistical}.

Here, we follow this line of research  and study an instance of regularized empirical risk minimization where,  given a fixed high or infinite dimensional
hypothesis space, the search  for a solution is restricted to a smaller, possibly random, subspace. This is equivalent to considering sketching operators \citep{kpotufe2019kernel}, or equivalently regularization  with  random  projections \citep{woodruff2014sketching}.
For infinite dimensional hypothesis spaces, it includes Nystr\"om methods used for kernel methods \citep{smola2000sparse} and Gaussian processes \citep{williams2001using}. Recent works in statistical learning analyzed this approach for smooth loss functions \citep{rudi2015less,bach2013sharp,marteau2019beyond}, whereas here we want to extend these results to convex, Lipschitz but possibly non smooth losses.

In fact,  for the square loss, all relevant quantities (i.e. loss function, excess risk) are quadratic, while the regularized estimator has an explicit expression, allowing for an explicit analysis based on linear algebra and matrix concentration \citep{tropp2012user}. Similarly,  for the logistic loss the analysis can be reduced to the quadratic case through a local quadratic approximation based on the self-concordance property \citep{bach2010self}. Instead,  convex,  Lipschitz but non-smooth losses,  such as the hinge loss, do not allow for such a quadratic approximation and we need to combine empirical process theory \citep{boucheron2013concentration} with results for random projections \citep{rudi2015less}. In particular, fast  rates require considering localized complexity measures \citep{steinwart2008support,bartlett2005local,koltchinskii2006local} and sub-gaussian inputs \citep{koltchinskii2014concentration,vershynin2018high}. We note that, related ideas have been used to extend results for random features from the square loss \citep{rudi2017generalization} to general loss functions \citep{li2019towards, sun2018but}.

Our main interest is in characterizing the relation between computational efficiency and statistical accuracy, while giving a unified study of the Nystr\"om method including both smooth and non-smooth losses. We do so studying the interplay between regularization, subspace size and the different parameters describing the hardness of the problem. 
Our results show that also for  convex, Lipschitz losses  there are settings in which the best known statistical bounds  can be obtained while substantially reducing computational requirements. Interestingly, these effects are relevant but also less marked than for smooth losses. 
In particular, some form of adaptive sampling seems to be needed to ensure no loss of accuracy and achieve sharp learning bounds.
More than that, differently from quadratic loss, also a fast eigenvalues decay of the covariance operator is fundamental to have some computational savings. 

As a byproduct of the aforementioned extension,  we consider the Nystr\"om method in the context of 
binary classification when the relevant error measure is the misclassification risk. 
Indeed, in this case convex loss functions are used as surrogate to 0-1 loss function and the corresponding excess risk bounds can be used to derive bounds on the excess misclassification risk \citep{bartlett2006convexity}. Since the latter is a weaker error measure is then natural to ask how the Nystr\"om method can affect the classification accuracy and to compare different loss functions in this context.  Indeed, our results show that the hinge loss can always achieve a better rate than the one derived by smooth loss functions. 
As regards computational savings, the comparison of the two upper bounds for hinge and quadratic losses suggests that hinge loss is cheaper only for certain classes of  \textit{hard} problems, characterized by a high complexity.

We note that a shorter version of the current paper has appeared in \citep{della2021regularized}.
Here we further develop this analysis, in particular extending our results to square and logistic losses, deriving classification risk bounds under margin assumption and finally compare the obtained results.\\

The rest of the paper is organized as follow.
In Section~\ref{sec: setting}, we introduce the setting and the main notation.
In Section~\ref{sec:ERM}, we review the ERM approach and in Section~\ref{sec: ERM on random subspace} we introduce ERM on random subspaces and our setting.
In Section~\ref{sec:statistical}, we present and discuss the main results and defer the proofs to the appendix.
In Section~\ref{sec:other}, we extend our previous results to smooth losses.
In Section~\ref{sec: 0-1 loss} we analyse the \textit{classification risk} with $0-1$ loss and discuss the comparison between the derived classification bounds from hinge and square losses.
In Section~\ref{sec:experiments}, we collect some simple numerical results.

\paragraph{Main contributions} 
In Section \ref{sec:statistic}, Theorem~\ref{thm:excessrisk-erm-standard} provides a finite sample bound of the excess risk for the regularized ERM in the misspecified case. This result was already established for bounded random variables, but had never been proved in the sub-gaussian case. Our proof in Appendix \ref{sec:radem-compl} also holds for the bounded case. In Section \ref{sec:statistics-Nystrom}, Theorem \ref{prop: constr pb} presents the first bound on the excess risk for generic convex loss functions when using the Nyström method. Under certain eigenvalue decay conditions of the covariance operator, this result leads to significant computational benefits while preserving the optimal rates achieved by standard ERM. However, the bound is suboptimal in terms of the required number of Nyström points compared to known bounds for smooth losses. Our refined analysis in Sections \ref{sec:statistical} and \ref{sec:other} overcomes this issue. Theorem \ref{thm: fast rate A(lambda)}, the most novel aspect of our work, provides the first optimal excess risk bound for the Nyström method for generic convex, possibly non-smooth, losses in the sub-gaussian case. We demonstrate that ERM with the Nyström algorithm can achieve fast rates under suitable eigenvalue decay conditions. Our result matches those obtained using random features but experiences a milder saturation effect, allowing for further improvements in the convergence rate with an increased number of sampled points.
 When adapting this result to smooth losses, Theorem \ref{thm:square loss} finally matches the known optimal results in \citep{rudi2015less}, while extending that analysis also to fast rates. In Section \ref{sec: 0-1 loss} we present a first comparison between the obtained results for hinge and square surrogates when considering classification error and under certain low noise condition.

\section{Setting and notations}  \label{sec: setting}
We start by introducing the learning setting and the assumptions we consider.
Let $\X$ be a real separable Hilbert space with scalar product
 $\scal{\cdot}{\cdot}$ and $\Y$ a Polish space, i.e a  separable complete
  metrizable  topological space. Let $(X,Y)$ be a pair of random variables taking value in $\X$ and $\Y$,
respectively, and denote by $P$ their joint distribution defined on the Borel $\sigma$-algebra of $\X\times \Y$.
Let $\ell:{\Y} \times \R\to [0,\infty]$ be a loss function and
\[
L: \X\to [0,\infty) \qquad \LL(w) =
\int_{\X\times\Y}\ell({y},\scal{w}{{x}})dP(x,y) = \E[\ell(Y,\scal{w}{X})] 
\]
the corresponding  expected risk. Given $w\in\X$,
$\ell({y},\scal{w}{{x}})$ can be viewed as the error made in
predicting $y$ with the linear function $\scal{w}{x}$, while $L(w)$ can be interpreted
as the expected loss on future data.

In this setting, we are interested in solving the problem
\begin{equation}
\label{def of the problem}
\inf_{w \in \X} \LL(w),
\end{equation}
when {the distribution} $P$ is only known through a training set $(X_1,Y_1), \dots, (X_n,Y_n)$ composed by $n$ copies of $(X,Y)$. Since the data are finite, we
cannot expect to solve the problem exactly. Given an empirical approximate
solution  $\wh w$, a natural error measure is the the excess risk $$
\LL(\wh w) - \inf_{w\in\X} L(w),$$ which is a random variable through
its dependence on $\wh w$, and hence on the data $
(X_i,Y_i)_{i=1}^{n}$.

We make the following assumptions on the data distributions and the loss.
\begin{ass}  \label{ass: sub-gaussian}
There exists $C >0$ such that $X$ is a $C$-sub-gaussian centered random vector.
\end{ass}
\noindent We recall that a random vector $X$ taking value in a Hilbert space  $\H$ is called $C$-sub-gaussian if
	\begin{equation}
		\label{def: subgauss}
		\|\langle X, u\rangle\|_p \leq
                C\sqrt{p}\|\langle X, u\rangle\|_{2} \quad\quad
                \forall u\in\H, p\geq 2,
              \end{equation}
where $\|\langle X, u\rangle\|_{p}^p= \mathbb E\left[|\langle
  X,u\rangle|^p \right]$ \citep{koltchinskii2014concentration}.  Note that \eqref{def: subgauss} implies
that for any vector $u\in \H$, the 
projection $\scal{X}{u}$ is a real sub-gaussian random variable
\citep{vershynin2018high},  but this latter condition is not
sufficient since the sub-gaussian norm  
\begin{equation}\label{eq:13}
	\Vert\scal{X}{u}\Vert_{\psi_2}=\sup_{p\geq 2} \frac{\|\langle X, u\rangle\|_p}{\sqrt{p}}
\end{equation}
should be bounded from above by the $L_2$-norm $\|\langle X,
u\rangle\|_{2}$. In particular, we note that, in general,
bounded random vectors in $\H$ are not sub-gaussian. 

Under the above conditions, $\E[ \nor{X}^2]$ is finite, so that 
the (non-centered) covariance operator
\[ \Sigma:\X\to \X \qquad \Sigma=\E[X\otimes X] \] is a trace-class 
positive operator. We define the effective rank of $\Sigma$ as
\begin{equation}
  \label{eq:1}
  r_\Sigma=\frac{\tr{\Sigma}}{\|\Sigma\|}
\end{equation}
where  $\operatorname{Tr}
\Sigma= \E[ \nor{X}^2]$ is the trace of $\Sigma$.

We define the so-called effective
dimension \citep{zhang2005learning,caponnetto2007optimal}, for $\alpha>0$, as
\begin{align}
&d_\alpha= \text{Tr}((\Sigma+\alpha I)^{-1}\Sigma) =\sum_j
                \frac{\sigma_j}{\sigma_j+\alpha}  
\end{align}
where  $(\sigma_j)_j$ are the strictly positive eigenvalues of
$\Sigma$,  with   eigenvalues  counted with respect to their
multiplicity and ordered in a non-increasing way, and $(u_j)$ is  the
corresponding family of eigenvectors. Note that  $d_\alpha$ is always finite since $\Sigma$ is trace class.

The next assumption is on the loss function. 
\begin{ass}[Lipschitz loss]\label{ass:loss}
The loss function $\ell: \Y \times \R\to {[0,\infty)}$ is convex and Lipschitz in its second argument, namely there exists $G>0$ such that 
\begin{equation}
    \label{eq:5}
  |\ell(y, a)-\ell(y, a')|\le G|a-a'| \quad \forall y \in \Y \quad \text{and}\quad a,a'\in \R.
\end{equation}
We also assume
$\ell_0= \sup_{y\in\Y} \ell(y,0)$ for all $y \in \Y$. 
\end{ass}
 Under the above condition, the expected risk $L(w)$ is finite, convex and Lipschitz.

We next provide some relevant examples.
The classical linear regression problem corresponds to the choice $\X=\R^d$
and $\Y=\R$. Another example is provided by kernel methods \citep{steinwart2008support}.
\begin{example}\label{ex kernel}
The input variable $X$ takes value in an abstract measurable set
$\mathcal X$.  We fix a reproducing
kernel Hilbert space on $\mathcal X$ with (measurable) reproducing kernel $K:
\mathcal X\times \mathcal X\to \R$. By mapping the inputs from ${\mathcal X}$ to
  $\H$ through the feature map 
\[\X\ni x\mapsto  K(\cdot,x)=K_x\in \X,\]
we can always identify $X$ with $K_X$, which is a random
variable taking value in $\X$. 
 \end{example}
 We now provide some examples of loss functions.
\begin{example}
  The main examples are 
  \begin{enumerate}[(a)]
  \item hinge loss:
    \begin{equation}
\ell(y,a)= |1-ya|_+=\max\{0,1-ya\}\qquad \Y = \{ -1, 1 \}\label{eq:9}
\end{equation}
 which is convex, but non-differentiable  with $G=1$ and $\ell_0=1$;
\item logistic loss 
  \begin{equation}
    \label{eq:10}
    \ell(y,a)= \log(1+e^{-ya}) \qquad \Y = \{ -1, 1 \}
\end{equation}
which is convex and differentiable  with $G=1$ and $\ell_0=\log 2$;
\item square loss
  \begin{equation}
    \label{eq:11}
    \ell(y,a)=\left(y-a\right)^2 \qquad \Y\subseteq [-M,M],
  \end{equation}
  which is convex and differentiable with $G_{\text{loc}}=2M$ (locally Lipschitz with $a\in[-M,M]$) and $\ell_0=M^2$.
  \end{enumerate}
\end{example}
 For classification, where $\Y = \{ -1, 1 \}$, a natural loss function is
given by the $0-1$ loss
\[
\ell_{0-1}(y,a):=\mathbb{1}_{(-\infty,0]}(y\;\operatorname{sign}a),
\]
which is not convex.

 In the next subsection we introduce the main notation.

\subsection{Notation}
For the reader's convenience we collect the main notation we introduced in
the paper. We denote with the ``hat'', e.g. $\wh \cdot$, random quantities depending on the data. Given a linear operator $A$ we denote by $A^\top$ its adjoint (transpose for matrices).
For any $n\in \N$, we denote by $\scal{\cdot}{\cdot}_n, \nor{\cdot}_n$ the inner product and Euclidean norm in $\R^n$. 
Given two quantities $a,b$ (depending on some  parameters), the
notation $a \lesssim b$, or $a = O (b)$ means that there exists a
constant $C$ such that $a \leq C b$. We denote by $P_X$ the marginal distribution of $X$ and by
$P(\cdot|x)$ is the conditional distribution of $Y$ given $X=x$. 
The conditional probability is well-defined since $\X$ is separable
and $\Y$ is a Polish space  \citep{steinwart2008support}. 
%
\begin{table}[h]\label{tab:1}
	\caption{Definition of the main quantities used in the paper}
	\centering
	\begin{tabular}{ll}
		\toprule
		\multicolumn{1}{c}{}   &\multicolumn{1}{c} {Definition}\\
		\cmidrule(r){1-1}	\cmidrule(r){2-2}
		$L(w)$				& $\;\int_{\X\times\Y}\ell({y},\scal{w}{{x}})dP(x,y)$\\
		$L_\la (w)$		& $\; L(w)+\la\|w\|^2$\\
		$\wh L (w)$		& $\; n^{-1} \sum_{i=1}^n \ell(y_i, \scal{w}{x_i})$\\
		$\wh L_\la (w)$		& $\; \wh L(w)+\la\|w\|^2$\\
		$w_*$                & $\;\argmin_{w\in\X}L(w)$                          \\
		$w_\la$              & $\;\argmin_{w\in\X}L_\la(w)$                      \\
		$\wh w_\la$          & $\;\argmin_{w\in\X}\wh L_\la(w)$                  \\
		$f_*(x)$             & $\;\argmin_{a\in \R}\int_{\Y} \ell(y, a) dP(y|x)$\\
		$\BB_m$ & $\;\mathrm{span}\{\wt  x_1, \dots, \wt  x_m\}$\\
				$\beta_{\la,m}$    & $\;\argmin_{\beta\in\BB_m}L_\la(\beta)$             \\
		$\wh\beta_{\la,m}$ & $\;\argmin_{\beta\in\BB_m}\wh L_\la(\beta)$         \\
		$\mathcal{P}_m$ &  \;projection operator onto $\BB_m$\\
		\bottomrule
	\end{tabular}
\end{table}

\section{Empirical risk minimization}\label{sec:ERM}

A classical  approach to derive empirical  solutions is based on replacing the expected risk with the empirical risk $\wh L:\X\to [0,\infty)$ defined for all $w\in \X$ as
$$
\wh L (w) = \frac 1 n \sum_{i=1}^n \ell(y_i, \scal{w}{x_i}).
$$
We consider the (regularized) empirical risk minimization (ERM) based on the solution of the problem,
\begin{equation}\label{perm}
\min_{w\in \X} \wh L_\la (w), \qq \qq\wh L_\la (w)= \wh L (w)+\la \nor{w}^2,
\end{equation}
where $\la>0$ is a positive regularization parameter. Since  $\wh L_\la:\H \to \R$ is continuous and  strongly convex,
there exists a  unique minimizer $\wh w _\la$  and, by the representer
theorem \citep{wahba1990spline,scholkopf2001generalized}, there exists $c
 \in \R^n$ such that
\begin{equation}\label{repre}
\wh w_\la = \wh  X^{\top} c \in  \mathrm{span}\{x_1,\dots, x_n\},
\end{equation}
where $\wh  X:\X\to \R^n$ denotes the input data matrix
\[
  (\wh X w)_i = \scal{w}{x_i} \qquad i=1,\ldots,n,\quad w\in\X.
  \]
The explicit form of the coefficient vector $c$  depends on
the considered loss function. In Section~\ref{sec:comp} we briefly
recall some possible  approaches to compute $c$, whereas in
Section~\ref{sec:statistic} we analyze the statistical properties of the above estimator.

	\begin{example}[Representer theorem for kernel machines]
		\label{ex: kernel span}
		In the context of kernel methods, see Example~\ref{ex kernel}, the above discussion, and in particular~\eqref{repre}, can be easily adapted. Indeed, the parameter $w$ corresponds to a function $f\in\H$ in the RKHS, while the norm $\|\cdot\|$ is the RKHS norm ${\|\cdot\|_\H}$. Eq. \eqref{repre} simply states that there exists constants $c_i$ such that
		 the solution of the regularized ERM can be written as
		 $\wh f_\lambda(x)=\sum_{i=1}^n K(x,x_i) c_i\in \mathrm{span}\{K_{x_1},\dots, K_{x_n}\}$.
	\end{example}
      
\subsection{Computational  aspects }\label{sec:comp}
Problem~\eqref{perm} can be solved in many ways and  we provide below some basic considerations.  If $\X$ is finite dimensional, gradient methods can be used. 
For example, the  subgradient method \citep{boyd2004convex} applied to~\eqref{perm} gives, for some suitable $w_0$ and step-size sequence $(\eta_t)_t$, 
\begin{equation}\label{subgsvm}
w_{t+1}= w_t - \eta_t \left (\frac 1 n  \sum_{i=1}^n y_{i} x_{i} 
g_i (w_t)
+2\la w_t\right),
\end{equation}
where for all ${i=1,\dots,n}$, {$g_i(w)\in \partial \ell(y_i, \scal{w}{x_i})$ is the subgradient of the  map $a\mapsto \ell(y_i,a)$ evaluated at $a=\scal{w}{x_i}$, see  \citep{rockafellar1970convex}.}
The corresponding per iteration cost is $O(nd)$ in time and memory. 
A more refined accelerated version of this algorithm can be found in \citep{tanji2023snacks}.
 When $\X$ is infinite dimensional a  different approach is possible, provided $\scal{x}{x'}$ can be computed for all $x,x'\in \X$. For example, it is easy to prove by induction  that the iteration in~\eqref{subgsvm}  satisfies  $ w_{t}=\wh X^\top c_{t+1}$, with
\begin{equation}\label{subgsvm1}
c_{t+1}= c_t - \eta_t \left (\frac 1 n  \sum_{i=1}^n y_{i} e_{i}
 g_i(\wh X^\top c_t)
 +2\la c_t \right),
\end{equation}
and where $e_1, \dots, e_n$ is the canonical basis in $\R^n$. The cost of the above iteration is $O(n^2  C_K)$ for computing $g_i(w) \in \partial \ell\left(y_i, \scal{\wh X ^\top c_t}{x_{i}}\right)= \partial \ell\left(y_i,\sum_{j=1}^n \scal{x_j}{x_i}(c_t)_i\right)$, where $C_K$ is the cost of evaluating the inner product. Also in this case, a number of approaches can be considered, see e.g. \citep[Chap.11]{steinwart2008support} and references therein. We illustrate the above ideas for the hinge loss.
\begin{example}[Hinge loss \& SVM]
 Problem~\eqref{perm} with the hinge loss corresponds to 
support vector machines for classification. With this choice, $\partial \ell(y_i,\scal{w}{x_i})=0$ if $ y_i\scal{w}{x_i} > 1$,  $\partial \ell(y_i,\scal{w}{x_i})=  [-1,0] $ if $ y_i\scal{w}{x_i} = 1$ and  $\partial \ell(y_i,\scal{w}{x_i})=-1$ if $ y_i\scal{w}{x_i} < 1$.
%
%
%
In particular,  in~\eqref{subgsvm1} we can take
 $
 g_i (w)=-
 \indic{y_i\scal{w}{x_i}\le 1}.
 $
\end{example}

\subsection{Statistical analysis}\label{sec:statistic}
In this section, we summarize the main statistical properties of the
regularized ERM under the sub-gaussian hypothesis in Assumption~\ref{ass: sub-gaussian}. In the following theorem
we provide a finite sample bound on the excess 
risk of $\wh w_\lambda$  without assuming the existence of $w^*$ (which will instead be assumed in Theorem~\ref{thm:regularized-full-space} via Assumption~\ref{ass:best}). Towards this end,
we introduce the approximation error,
\begin{align}\label{eq:24} 
	\mathcal{A} (\lambda) 
	&= \inf_{w \in \H} [L (w) + \lambda \| w \|^2] - \inf_{w\in\H} L(w) 
	\,.
\end{align}
Note that,  if $w_*$ exists, then ${\cal A}(\la)\leq\la \nor{w_*}^2$. More generally, the
approximation error decreases with $\la$ and learning rates can be
derived assuming a suitable decay.
\begin{theorem}\label{thm:excessrisk-erm-standard}
Under Assumptions~$\ref{ass: sub-gaussian}$ and~$\ref{ass:loss}$, fix 
$\la>0$ and $0<\delta<1$. Then, 
  with probability at least  $1- \delta$,
		\begin{align}
				\label{eq:excessrisk-standard-general}
	L (\wh w_\lambda) - \inf_{w\in\H} L(w)
	< &2 \mathcal{A} (\lambda) + \frac{D^2G^2C^2 \|\Sigma\|((\sqrt{r_\Sigma}+K)^2+(  \sqrt{r_\Sigma}+\sqrt{\log (1/\delta)})^2)}{4\lambda n}
	+ \nonumber\\
	&+ \frac{DGC(\sqrt{r_\Sigma}+K) \|\Sigma\|^{\frac{1}{2}} +
	D\ell_0(K+\sqrt{\log (1/\delta)})}{\sqrt{n}} .
\end{align}
where $C$ and $G$ are the constants defined respectively in~\eqref{def: subgauss} and~\eqref{eq:5}, $D$ is a numerical constant and
  \[
K = K_{\lambda, \delta} = \sqrt{\log (1 + \log_2 ( 3 +
	{\ell_0/\lambda} ) ) + \log (1/\delta)}
    = O ( \sqrt{\log \log (3 + \ell_0 / \lambda) + \log (1/\delta)} ).
  \]

\end{theorem}
 The theorem can be easily extended to non-centered
sub-gaussian variables. Notice that the same result is well known for bounded random
variables; see,
for example \citep{steinwart2008support,shalev2010learnability}. We are not aware of a reference for the sub-gaussian case. In Appendix~\ref{sec:radem-compl} we provide a
simple self-contained proof,  which also holds true for the bounded
case~\citep{della2021regularized}. It is based on the fact that  the excess
risk bound for regularized ERM arises from a trade-off between an
estimation and an approximation error. Similar bounds in
high-probability for ERM constrained to the ball of radius $R \geq \|
w_* \|$ can be obtained through a uniform convergence argument over
such balls, see
\citep{bartlett2002rademacher,meir2003generalization,kakade2009complexity}. 
To apply this line of reasoning to regularized ERM, one could in principle use
the fact that by Assumption~\ref{ass:loss}, $\| \wh w_\lambda \| \leq
\sqrt{\ell_0/\lambda}$ (see Appendix) \citep{steinwart2008support}, but
this would yield a suboptimal dependence in $\lambda$.
Finally, a similar rate, though only in
expectation, can be derived through a stability argument
\citep{bousquet2002stability,shalev2010learnability}.

 The bound~\eqref{eq:14} shows that the learning rate depends on some
a-priori assumption on the distribution that allows control of the approximation error
$\mathcal{A} (\lambda) $. The simplest assumption is that the
  best in the model exists.   
\begin{ass}\label{ass:best}
  There exists $w_*\in \X$ such that $\displaystyle{L(w_*)= \min_{w\in \X} L(w)}$.
\end{ass} 
 Under the above condition, we have the following result.

\begin{theorem}\label{thm:regularized-full-space}
  Under  Assumption~\ref{ass: sub-gaussian},~\ref{ass:loss},
  and~\ref{ass:best}, take $\la>0$ and $0<\delta<1$, then
with probability at least $1-{\delta}$:
\begin{align}
		\label{eq:excessrisk-standard-best}
	L (\wh w_\lambda)-L(w_*) 
	<& \lambda \| w_* \|^2 + 		\frac{D^2G^2C^2(\sqrt{r_\Sigma}+K)^2 \|\Sigma\|}{4\la n} +\nonumber\\
	&+\frac{DGC(\sqrt{r_\Sigma}+K) \|\Sigma\|^{\frac{1}{2}} +
		D\ell_0 (K+\sqrt{\log (8/\delta)})}{\sqrt{n}} +\nonumber\\
	&+ \frac{DG C  \|\Sigma\|^{\frac{1}{2}} \nor{w_*}\big(  \sqrt{r_\Sigma}+\sqrt{\log (8/\delta)}
		\big)}{\sqrt{n}}.
\end{align}
Hence, let $\lambda=\la_n \asymp (DGC\nor{\Sigma}^{1/2} / \| w_* \|)
\sqrt{\log (1/\delta) / n}$ with high probability:
\begin{equation}
  \label{eq:17}
  L (\wh w_{\lambda_n})- L(w_*) =   O ( \| w_* \|\sqrt{{\log(1/\delta)}/n} ),
\end{equation}
up to a  $\log \log n$ terms.
 \end{theorem}  
 As above, the proof  is given in
Appendix~\ref{sec:radem-compl}. In a nutshell, what Theorem.~\ref{thm:regularized-full-space}  shows is
that, with high probability,
\[
L(\wh w_\la) -\inf_{w\in\H}L(w)\lesssim \frac{1}{\la
	n } + \la \nor{w_*}^2, 
\]
provided that the  best in model $w_*\in \X$ exists.
With the choice  $\lambda \asymp \sqrt{ 1/ n}$ it holds that 
\begin{equation}
	L (\wh w_\lambda)- \inf_{w\in \H} L(w) =   O ( \sqrt{1/n} ) \label{eq:38},
\end{equation}
which provides a benchmark for the results in the next sections. 

\begin{remark}
	Note that for all $w \in \H$ with $\| w \| \leq R$, 
	\[\mathcal{A} (\lambda) \leq L (w) + \lambda \| w \|^2 - \inf_\H L
	\leq L (w) - \inf_\H L + \lambda R^2\]
	hence
	$\mathcal{A} (\lambda) \leq \inf_{\| w \| \leq R} L (w) - \inf_\H L
	+ \lambda R^2$ and
	\begin{align*}
		L (\wh w_\lambda) - \inf_{w\in\H} L(w)
		<& 2 \Big( \inf_{\| w \| \leq R} L (w) - \inf_\H L \Big) + 2
		\lambda R^2 + \\
		&+   \frac{D^2G^2C^2 \|\Sigma\|((\sqrt{r_\Sigma}+K)^2+(  \sqrt{r_\Sigma}+\sqrt{\log (8/\delta)})^2)}{4\lambda n}+\\
		&+
		\frac{DGC(\sqrt{r_\Sigma}+K) \|\Sigma\|^{\frac{1}{2}} +
			DK\ell_0+D\ell_0 \sqrt{\log (8/\delta)}}{\sqrt{n}}	.
	\end{align*}
	Letting $\lambda \asymp 1/ (R \sqrt{n})$, this gives
	$L (\wh w_\lambda) - \inf_{w\in\H} L(w) \leq 2 (\inf_{\| w \| \leq
		R} L (w) - \inf_\H L) + O ( R/\sqrt{n} )$ with high probability.
\end{remark}

\section{ERM on random subspaces} \label{sec: ERM on random subspace}

As explained in the introduction, though the ERM estimator $\wh
w_{\lambda}$ achieves sharp rates, from a computational point of
view, it can be very expensive for large datasets. To overcome this issue,
we study a variant of ERM  based on considering a subspace $\BB\subset \X$ and the corresponding regularized ERM problem,
\begin{equation}\label{sperm}
\min_{\beta \in \BB}\wh L_\la (\beta),
\end{equation}
with $\wh\beta_\lambda$ as the unique minimizer.
As clear from~\eqref{repre}, choosing $\BB= \X_n = \mathrm{span}\{ x_1, \dots, x_n\}$ is not a restriction and yields the same solution  as considering~\eqref{perm}. 
From this observation, a natural choice is to consider for $m\le n$,
\begin{equation}\label{randspace}
\BB_m=\mathrm{span}\{\wt  x_1, \dots, \wt  x_m\}
\end{equation}
with $\{\wt  x_1, \dots, \wt  x_m\}\subset \{ x_1, \dots,  x_n\}$ being a
subset of the input points, called the Nystr\"om points. We  denote by ${\mathcal P}_m= {\mathcal P}_{{\BB_m}}$ the corresponding projection and by  $\wh{\beta}_{\la,m}$ the unique minimizer of $\wh L_\la$ on $\BB_m$, i.e.
\begin{equation}
  \label{eq:37}
\wh{\beta}_{\la,m}=\operatornamewithlimits{argmin}_{\beta \in \BB_m}\wh L_\la (\beta).
\end{equation}
 In the rest of the paper, all the results are valid when the Nystr\"om points are selected using approximate leverage scores (ALS) sampling. 
Recall that leverage scores are defined as 
\citep{drineas2012fast}:
\begin{equation}
	{l_i}(\alpha)=\scal{x_i}{(\wh  X\wh  X ^\top x+\alpha In )^{-1} x_i}
	\qquad  i=1, \dots, n\label{eq:3}
\end{equation}
where $\alpha>0$. Since in practice the leverage scores $l_i(\alpha)$ are expensive to compute, approximations have been considered
\citep{drineas2012fast,cohen2015uniform,alaoui2015fast,rudi2018fast}. In particular,
we consider approximations of the form described in the following definition.

\begin{definition}[Approximate leverage scores
	sampling (ALS)]\label{def:approx_lev_scores} 
	Let $(l_i(\alpha))_{i=1}^n$ be the leverage scores~\eqref{eq:3}.  Given $ \alpha_0>0$ and $T\geq 1$, we say
	that a family $(\hat{l}_i (\alpha))_{i=1}^n$ is $(T,\alpha_0)$-approximate leverage
	scores with confidence $\delta\in (0,1)$ if  
	\begin{equation}
		\frac{1}{T} l_i(\alpha)\leq \hat{l}_i(\alpha) \leq T
		l_i(\alpha), \hspace{0.6cm} \forall i \in
		\{1,\dots,n\},\;\;\alpha\geq  \alpha_0 ,
	\end{equation}
	with probability at least $1-\delta$. Under this condition, the approximate leverage   scores (ALS)   sampling 
	selects the Nystr\"om points 
	$\{\tilde{x}_1,\dots,\tilde{x}_m\}$  from the training set $\{ x_1, \dots,  x_n\}$ 
	independently with replacement  and with probability
	$Q_\alpha(i)=\hat{l}_i(\alpha)/\sum_j \hat{l}_j(\alpha)$.  
\end{definition}

 We now focus  on the computational benefits of considering
ERM on random subspaces and we analyze the corresponding statistical
properties in
Section~\ref{sec:statistics-Nystrom}. 

\subsection{Computational aspects} 
The choice of $\BB_m$ as in eq.~\eqref{randspace} allows for improved computations with respect to eq.~\eqref{repre}. Indeed,  $\beta\in \BB_m$ if and only if $\exists b\in \R^m $ such that. $\beta= \wt X^\top b$, with $\wt X: \X\to \R^m$ being the matrix with rows the chosen Nystr\"om points. Then, we can replace the problem in~\eqref{sperm} with
\begin{equation}
\label{eq:erm_nystrom_pr}
\min_{b\in \R^m}\frac 1 n \sum_{i=1}^n \ell\left(y_i, \scal{\wt X^\top b}{x_i}\right) +\la \scal{b}{\wt X\wt X^\top b}_m,
\end{equation}
where $\scal{\cdot}{\cdot}_m$ is the scalar product in $\R^m$.
Further, since $\wt X\wt X^\top\in \R^{m\times m}$ is symmetric and positive semi-definite, we can derive a formulation close to that in~\eqref{perm}, considering 
the reparameterization $a= (\wt X\wt X^\top)^{1/2} b$ which leads to,
\begin{equation}\label{embedsvm}
\min_{a\in \R^m}\frac 1 n \sum_{i=1}^n \ell\left(y_i, \scal{a}{ \xx_i }_m\right) +\la \nor{a}_m^2,
\end{equation}
where for all $i=1, \dots, n$,  we defined the embedding $x_i \mapsto \xx_i= ((\wt X\wt X^\top)^{1/2})^\dagger \wt X x_i$ and with $\|\cdot\|_m$ we denote the Euclidean norm in $\R^m$. Note that the computation of the embedding $x_i\to \xx_i$ only involves the inner product in $\X$ and can be computed in 
$O(m^3 + n m^2 C_K)$ time.  The subgradient method for~\eqref{embedsvm} has a cost of $O(nm)$ per iteration. 
In summary, we obtained that the cost for ERM on subspaces is
$
O(n m^2 C_K+ nm\cdot \# \text{iter})
$
and should be compared with the cost of solving~\eqref{subgsvm1} which is 
$
O( n^2 C_K+ n^2 \cdot \# \text{iter}).
$
The corresponding costs to predict new points are $O(mC_K)$ and $O(n C_K)$, while the memory requirements are $O(mn)$ and $O(n^2)$, respectively. Clearly, memory requirements can be reduced by recomputing things on the fly. 
As clear from the above discussion, computational savings can be drastic as long as $m<n$, and the question arises of how this affects the corresponding statistical accuracy. The next section is devoted to this question.
\begin{remark}
Formulations in eq.~\eqref{eq:erm_nystrom_pr} and eq.~\eqref{embedsvm} of the reduced problem can be equivalently written as problem~\eqref{perm} using the projector operator $\mathcal{P}_m$, i.e. 
$\min_{w \in \H}	\wh L(\mathcal{P}_m w)+\la\|w\|^2$. Note that when transitioning to the approximated problem, the Lipschitz constant can only decrease, while the smoothness does not change for sufficient large values of $\lambda$. Consequently, solving the approximated problem in \eqref{embedsvm} through subgradient descent requires a number of iterations which is smaller than or equal to those needed for solving~\eqref{perm}. For simplicity, this was not considered when discussing the computational benefits of the approximated method, but it further strengthens our point.
\end{remark}
\begin{example}[Kernel methods and Nystr\"om approximations]
  \label{ex kernel emb}
  Again, following Example~\ref{ex kernel} and Example \ref{ex: kernel span}, our setting can be easily specialized to kernel methods, where $\beta \in \BB_m=\mathrm{span} \{\wt x_1,\dots,\wt x_m\}$ is replaced by $\wt f(x)=\sum_{i=1}^m K(x,\wt x_i)\wt c_i \in \mathrm{span} \{K_{\wt x_1},\dots, K_{\wt x_m}\}$, while the embedding $x_i \mapsto \xx_i= ((\wt X\wt X^\top)^{1/2})^\dagger \wt X x_i$ becomes $x_i \mapsto \xx_i= (\wt K^{1/2})^\dagger (K(\wt x_1,x_i),\dots, K(\wt x_m,x_i))^\top$, with ${\wt K_{i,j}= K(\wt x_i, \wt x_j)}$.
\end{example}

\subsection{Statistical analysis}\label{sec:statistics-Nystrom}

In this section, we will show, under a suitable polynomial (or exponential) decay condition on the spectrum of $\Sigma$ (see~\eqref{eq:2}), that,
\[
L(\wh{\beta}_{\lambda,m})-L(w_*) \lesssim \frac{\sqrt{\log(1/\delta)}}{\sqrt{n}},
  \]
provided that the best in model  $w_*\in \X$ exists, see Assumption~\ref{ass:best}, and, up to log terms, 
  \[
\lambda\asymp\frac{1}{\sqrt{n}}, \qquad m\gtrsim n^{p},
    \]
    where the exponent $p$ controls how strong the polynomial decay condition is (see~\eqref{eq:2}).
Compared to the results for exact ERM in~\eqref{eq:38}, we get the same convergence rate up to a log factor, but the computational complexity of the algorithm  is dramatically reduced. For example, if $p=1/2$ we only need $m\simeq\sqrt{n}$ Nystr\"om points.  A similar result is obtained for exponential decay in which case we can take $m\simeq \log^2n$ Nystr\"om points.  
We observe that under the above decay conditions on the spectrum of $\Sigma$, classical ERM algorithm achieves fast rates. In Section~\ref{sec:statistical}, we will show that also randomized ERM can achieve fast rates, but this will require a more refined analysis.

 We now state the detailed results. We recall that the Nystr\"om
points are sampled according to ALS, see Definition~\ref{def:approx_lev_scores}.

\begin{theorem}\label{thm:1}
Under Assumption~\ref{ass: sub-gaussian},~\ref{ass:loss} and~\ref{ass:best}, fix $\alpha,\la,\delta>0$. Then, with probability at least $1-\delta$:
	\begin{align}
	\label{reg pb f_H main bound}
	L(\wh{\beta}_{\la,m})-L(w_*) & \lesssim
	\frac{\log (1/\delta)}{\lambda n}+\frac{ \|w_*\|\sqrt{\log (1/\delta)}}{\sqrt{n}}+\sqrt{\alpha} \|w_*\|+\lambda\Vert  w_*\Vert^2
\end{align}
up to $\log( \log(1/\lambda))$ terms and 
provided that $n\gtrsim d_\alpha\vee \log(1/\delta)$ and $ m\gtrsim d_\alpha\log(\frac{2n}{\delta})  .$
\end{theorem}
 The proof of Theorem~\ref{thm:1} with explicit constants is given
in Appendix~\ref{proofthmbasic}, here we only add some comments.  Note that
    \begin{align}
	d_\alpha & =\int \scal{w}{(\Sigma+\alpha I)^{-1}w} dP_X(w)  \leq
	\int \nor{w}^2 \nor{(\Sigma+\alpha I)^{-1}} dP_X(w) \leq
	\alpha ^{-1}\E[\nor{X}^2]\lesssim \alpha^{-1}, \label{eq:6bis}
\end{align}
using the fact that the second moment of a sub-gaussian variable is finite. Using the above bound, we get that, up to log terms, 
with high probability
$$
L(\wh{\beta}_{\la,m})-L(w_*) \lesssim
\frac{\log (1/\delta)}{\lambda n}+\frac{ \|w_*\|\sqrt{\log (1/\delta)}}{\sqrt{n}}+\sqrt{\alpha} \|w_*\|+\lambda\Vert  w_*\Vert^2, 
$$ 
provided that $m\gtrsim \alpha^{-1}$. 
With the choice
\[
\lambda \asymp \frac{1}{\nor{w_*}\sqrt n}\;, \qquad \alpha \asymp 1/ n 
  \]
  we get that with high probability
	\begin{equation}
	L(\wh{\beta}_{\la_n,m})-L(w_*)\lesssim \frac{\Vert w_*
          \Vert\sqrt{\log(1/\delta)}}{\sqrt{n}} 
      \end{equation}
up to log factors in $n$ and with $m\gtrsim n$.  

 Despite of the fact that the rate is optimal (up to the logarithmic
term), the required number of subsampled 
points is $m\gtrsim n$, so that the procedure is not
effective. However, the following proposition shows that under  a
decay conditions on the spectrum of the covariance operator $\Sigma$, the
ALS method becomes computationally efficient.  
We assume one of the following two conditions:
  \begin{enumerate}[a)]
  \item polynomial decay: there exists $p\in(0,1)$ such that 
      \begin{equation}
      \sigma_j\lesssim
      j^{-\frac{1}{p}} \label{eq:2}
    \end{equation}
  \item exponential decay: there exists $\beta>0$ such that 
    \begin{equation}
      \label{eq:19}
      \sigma_j\lesssim e^{-\beta j}.
    \end{equation}
  \end{enumerate}
Under the above conditions, we following result holds.
\begin{theorem}
	\label{prop: constr pb}
 Under the assumptions of Theorem~$\ref{thm:1}$, fix $\delta>0$,  with probability at
 least $1-\delta$:
   		\begin{align}
 	L(\wh{\beta}_{\la,m})-L(w_*) & \lesssim
 \frac{\log (1/\delta)}{\lambda n}+\frac{ \|w_*\|\sqrt{\log (1/\delta)}}{\sqrt{n}}+\sqrt{\alpha} \|w_*\|+\lambda\Vert  w_*\Vert^2
   \end{align}
     and, with the choice
     	\begin{enumerate}[(a)]
     		 \item for the polynomial decay~\eqref{eq:2} 
     		 \[
     		 \lambda\asymp \frac{\Vert w_* \Vert\sqrt{\log(1/\delta)}}{\sqrt{n}}, \qquad
     		 \alpha\asymp\frac{\log(1/\delta)}{n}, \qquad m\gtrsim n^{p},
     		 \]
     		 \item for the exponential decay~\eqref{eq:19}
     		  \[
     		 \lambda\asymp \frac{\Vert w_* \Vert\sqrt{\log(1/\delta)}}{\sqrt{n}}, \qquad
     		 \alpha\asymp\frac{\log(1/\delta)}{n},  \qquad m\gtrsim\log^2 n,
     		 \]
     		\end{enumerate}
     		 then, it holds that
     		 \begin{equation}
     		   \label{eq:12}
     		   	L(\wh{\beta}_{\la_n,m})-L(w_*) \lesssim \frac{\Vert w_* \Vert\sqrt{\log(1/\delta)}}{\sqrt{n}}.
     		 \end{equation}

\end{theorem}
The proof of the above result is given in
Appendix~\ref{proofthmbasic}.
Theorem~\ref{prop: constr pb} is already known for square loss \citep{rudi2015less} and for smooth loss functions \citep{marteau2019beyond} under the assumption that the input $X$ is bounded. 
However, note that our bound on the number of Nystr\"om points is, in the case of square loss, worse than the bound in~\citep{rudi2015less}.
In Section~\ref{sec:other}, by specializing the analysis for smooth losses and exploiting the special structure of the quadratic loss, we obtain the right estimate of Nystr\"om points matching the result in~\citep{rudi2015less}. 

Theorem~\ref{prop: constr pb} shows that for an arbitrary convex, possibly non-smooth, loss function, leverage scores sampling can lead to better results depending on the spectral properties of the covariance operator.
Indeed, if there is a fast eigendecay, then using leverage scores  and a subspace of dimension $m<n$, one can achieve the same rates as exact ERM.
For fast eigendecay ($p$ small), the subspace dimension can decrease dramatically.
For example,  considering $p=1/2$, then the choice $m\simeq \sqrt{n}$ is enough.
These observations are consistent with recent results for random features \citep{bach2017equivalence,li2019towards,sun2018but}, while they seem new for ERM on random subspaces.
Compared to random features, the proof techniques present similarities but also differences due to the fact that in general random features do not define subspaces.
Finding a unifying analysis would be interesting, but it is left for future work.
Also, we note that uniform sampling can have the same properties as leverage scores sampling, if $d_\alpha \asymp d_{\alpha, \infty}$, where $d_{\alpha,\infty}:= \sup_{w\in\supp(P_X)} \scal{w}{(\Sigma+\alpha I)^{-1}w}$, see \citep{rudi2015less}.
This happens under strong assumptions on the eigenvectors of the covariance operator, but can also happen in kernel methods with  kernels corresponding to Sobolev spaces \citep{steinwart2009optimal}.
With these comments in mind, next, we focus on random subspaces defined by leverage scores sampling and show that the assumption on the eigendecay not only allows  for smaller subspace dimensions, but can also lead to faster learning rates. 

\begin{remark} \label{remark:proj}
Following \citep{rudi2015less},  other choices of $\BB\subseteq \H$ are possible.
Indeed, for any $q\in \N$ and   $z_1, \dots, z_q\in  \H$ we could
consider $\BB= \mathrm{span}\{ z_1, \dots,  z_q\}$ and derive a
formulation as in~\eqref{embedsvm} replacing $\wt X$ with the matrix
$Z$ with rows $z_1, \dots, z_q$.  We leave this discussion  for future
work. We simply state the following result where
\begin{equation}\label{proj}
  \mu_{\BB}= \nor{\Sigma^{1/2}(I-{\mathcal P})},
\end{equation}
and ${\mathcal P}$ is the projection onto $\BB$. 
\begin{theorem}\label{firstmain}
  Choose $\BB\subseteq \H$. 
  Under Assumptions~\ref{ass: sub-gaussian},~\ref{ass:loss},~\ref{ass:best}, fix $\la>0$ and
$0<\delta<1$,  with probability at least $1- \delta $: 
\begin{equation*}
L(\wh{\beta}_{\la}) -L(w_*)\lesssim \frac{ 
 {\log(1/\delta)}}{\la n } + \la\nor{w_*}^2 + \sqrt{\mu_\BB} \nor{w_*}.
\end{equation*}
\end{theorem}
\noindent Compared to Theorem~\ref{thm:regularized-full-space}, the above result 
shows that there is an extra approximation error term due to considering a subspace. The coefficient $\mu_\BB$
appears in the analysis also for other loss functions, see e.g.  \citep{rudi2015less,marteau2019beyond}. Roughly speaking,  it captures how well the subspace $\BB$ is adapted to the problem.
\end{remark}

\section{Fast rates}\label{sec:statistical}

In this section, we prove that Nystr\"om algorithm achieves fast rates
under a Bernstein condition on the loss function, see
Assumption~\ref{ass:berstein}, which is quite standard in order to have fast
rates for regularized ERM  \citep{steinwart2008support,
  bartlett2005local}. To state the results, we recall some definitions
and basic facts, see \citep[Chapter~6]{steinwart2008support}.  

Given a threshold parameter $M>0$, for any $a\in \R$,  $a^{cl}$ denotes the clipped value of $a$ at $\pm M$
  \begin{align*}
    &a^{cl}=-M \qq\text{if} \;\;a\leq -M, \qq \qq
      a^{cl}=a \qq\text{if}\;\; a\in[-M\text{,}M], \qq\qq
      a^{cl}=M \qq\text{if}\;\;  a\geq M.
  \end{align*}
We say that the loss function $\ell$ can be {\it clipped} at $M>0$ if  for all
  ${y\in\Y},a\in \R$,
  \begin{equation} \label{eq: clipping} \ell(y,a^{cl})\leq
    \ell(y,a),
  \end{equation}
For convex loss functions, as considered in this paper, the above
definition is equivalent to the fact that  for all $y\in\Y$, there
exists $a_y\in [-M,M]$ such that
\[
\ell(y,a_y)=\min_{a\in\R} \ell(y,a) ,
  \]
  see  \citep[Lemma 2.23]{steinwart2008support}. Furthermore, Aumann’s
  measurable selection principle \citep[Lemma
  A.3.18]{steinwart2008support} implies that there exists a measurable
  map $\varphi:\Y\to \R$ such that 
  \[
 \ell(y,\varphi(y))=\min_{a\in\R} \ell(y,a),  \qquad |\varphi(y)|\leq M
    \]
and we can set
\begin{equation}\label{target_def}
f_*(x) = \int_{\Y} \ell(y, \varphi(x)) dP(y|x),
\end{equation}
for $P_X$-almost all $x\in \X$.  The function $f_*$ is the target
function since
\[
L(f_*) = \inf_f L(f),
  \]
where the infimum is taken over all the measurable functions
$f:\X\to\R$. It  easy to check that hinge loss  and  square loss with
bounded outputs can be clipped. Even if the logistic loss can not be clipped, we will show in Section \ref{sec: logistic} how we can easily bypass this issue with an ad hoc fix. We also introduce the following notation, for all  $w\in\X$, we set
  \[
    w^{cl}:\X\to \R \qquad w^{cl}(x) = \scal{w}{x}^{cl}.
 \]
In the following we assume the conditions below.

\begin{ass}[Clippability]\label{ass:clipped}
There exists $M>0$ such that the loss function can be clipped at $M$. 
\end{ass}

\begin{ass}[Universality]\label{ass:universal}
  One has
  \begin{equation}\label{universal}
    \inf_{w\in \X} L(w)= L(f_*).
  \end{equation}
\end{ass}
Recalling that the target function $f_*$ is the minimizer of the expected error
over all possible functions $f$, condition~\eqref{universal} means
that  $f_*$ can be arbitrarily well approximated by a linear function
$\scal{w}{x}$ for some $w\in\X$.  When considering the square loss, this condition
is equivalent to the fact that $\X$ is dense in $L^2(\X,P_X)$ and, in
the context of kernel methods, see 
Example~\ref{ex kernel} it is satisfied by universal kernels \citep{steinwart2008support}.
Condition~\eqref{universal}  may be
relaxed at the cost of an additional approximation term, but the analysis is just lengthier and it won't be discussed in here.  A sufficient stronger
condition is provided by assuming the target function to be linear (well-specified model).
\begin{ass}[Well specified model]\label{target}
There exists $w_*\in \X$ such that
  \[
  f_*(x)= \scal{w_*}{x} 
\]
  for $P_X$-almost $x\in\X$.
\end{ass}
We further assume the
following condition. 
\begin{ass}[Bernstein condition]\label{ass:berstein}
  There exist constants $B>0$, $\theta\in[0,1]$ and $V\geq B^{2-\theta}$, such  that for all $w\in\H$, the following inequalities hold almost surely:
	\begin{align}
	&\ell({Y},\scal{w}{X}^{cl})\leq B, \label{supremum bound}\\
	&\label{variance bound}
          \E \big[ \big\{ \ell(Y,\scal{w}{X}^{cl})- \ell(Y,f_*(X)) \big\}^2 \big]
          \leq
          V(\E[\ell(Y,\scal{w}{X}^{cl})- \ell(Y,f_*(X))])^\theta \\
   &\label{variance bound2}
     \E \big[ \big\{ \ell(Y,\scal{w}{X})- \ell(Y,f_*(X)) \big\}^2 \big]
     \leq
     V(\E[\ell(Y,\scal{w}{X})- \ell(Y,f_*(X))])^\theta 
   	\end{align}
      \end{ass}
Condition~\eqref{supremum bound} is called supremum bound \citep{steinwart2008support} and, thanks to the clipping, it is 
satisfied by Lipschitz loss functions. Condition~\eqref{variance
  bound} is called variance bound \citep{steinwart2008support} and the optimal exponent
corresponds to the choice $\theta=1$. For the square loss with bounded
output, the variance bound always  holds true with $\theta=1$, see \citep[Example
7.3]{steinwart2008support} . For other loss functions the above
condition is hard to verify for all distributions. For
classification, the variance bound is implied by so called margin conditions (see Section \ref{sec: 0-1 loss} and Theorem 8.24 in \citep{steinwart2008support}),
and the parameter $\theta$ characterizes how easy or hard the
classification problem is \citep{steinwart2008support}.
With respect to~\citep{steinwart2008support}, condition~\eqref{variance bound2} is a technical one that we need in the
proof.

To state our result, we will make use again of the approximation error $\A(\la)$ defined in~\eqref{eq:24}.
The following theorem provides
fast rates for Nystr\"om algorithm, where we recall the Nystr\"om
points are sampled according to ALS, see Definition~\ref{def:approx_lev_scores}.

\begin{theorem}\label{thm: fast rate A(lambda)}
Under  Assumptions~\ref{ass: sub-gaussian},~\ref{ass:loss},~\ref{ass:clipped},~\ref{ass:berstein}, let 
fix $0<\delta<1$,   then, with probability at least $1-2\delta$:
 \begin{enumerate}[(a)]
 \item for the polynomial decay condition~\eqref{eq:2}
	\begin{align}
		\label{fast rate A(lambda)}
		L(\wh{\beta}_{\lambda,m}^{cl})-L(f_*)
		&\lesssim \Big(\frac{1}{\lambda^pn}\Big)^{\frac{1}{2-p-\theta+\theta p}}+ \sqrt{\frac{\alpha
				\mathcal A(\la)}{\lambda}}+\Big(\frac{\log(3/\delta)}{n}\Big)^{\frac{1}{2-\theta}}+
		\frac{\log(3/\delta)}{n} \sqrt{\frac{\mathcal A(\la)}{\lambda}} +\A(\la)
\end{align}
   provided that
   \begin{align*}
     \alpha\gtrsim n^{-1/p}, \qquad n \gtrsim d_\alpha\vee \log(1/\delta), \qquad m \gtrsim d_\alpha\log(\frac{2n}{\delta}),
   \end{align*}
 \item for the exponential decay condition~\eqref{eq:19}
\begin{align*}
		L(\wh \beta^{cl}_{\la,m}) -
		L(f_*)
		&\lesssim 
		\frac{\ln^2(1 / \lambda)}{n}+\sqrt{\frac{\alpha\mathcal{A}(\la)}{\lambda}}
		+\Big(\frac{\log(3/\delta)}{n}\Big)^{\frac{1}{2-\theta}}+\frac{\log(3/\delta)}{n} 
		\sqrt{\frac{\mathcal A(\la)}{\lambda}}+{\cal A}(\la)
	\end{align*}
  provided that
   \begin{align*}
     \alpha & \gtrsim e^{-n}, \qquad n\gtrsim d_\alpha\vee \log(1/\delta), \qquad m \gtrsim d_\alpha\log(\frac{2n}{\delta}).
   \end{align*}  
 \end{enumerate}
\end{theorem}
The proof of Theorem~\ref{thm: fast rate A(lambda)} is given in
Appendix~\ref{app:theorem 4}. Notice that a faster decay condition on the spectrum of $\Sigma$ leads to improvements in both the excess risk bound and the parameters' choices. As regards the learning rate, under exponential decay in \textit{(b)}, first term of \eqref{fast rate A(lambda)} improves to $1/n$ up to logarithmic factors. At the same time, the range of admissible $\alpha$ gets larger while the control on the effective dimension gets tighter. Let us comment these results more precisely in the following.

\subsection{Polynomial decay of $\Sigma$} 
\label{subsec:poly}
In this section we assume the polynomial  decay~\eqref{eq:2} condition on the
spectrum of $\Sigma$. By omitting numerical constants, logarithmic and higher
order terms,  Theorem~\ref{thm: fast rate A(lambda)} implies that with high probability
\[
L(\wh \beta^{cl}_{\la,m}) -
 L(f_*) \lesssim 
\left(
\frac{1}{\la^p n} 
\right)^{\frac{1}{2-p-\theta+\theta p}}+ \sqrt{\frac{\alpha {\cal A}(\la)}{\la}}
 +\frac{\log(3/\delta)}{n} 
 \sqrt{\frac{\mathcal A(\la)}{\lambda}}+{\cal A}(\la). 
  \]
To have an explicit rate, we further assume that there exists $r \in (0,1]$ such that  
\[\mathcal  A(\la)\lesssim \la^{r}.\]
Under this condition, with the choice
\begin{align*}
 & \lambda_n \asymp
            n^{-\min\{\frac{2}{r+1},\frac{1}{r(2-p-\theta+\theta
            p)+p}\}} \\
&\alpha_n \asymp n^{-\min\{2,\frac{r+1}{r(2-p-\theta+\theta p)+p}\}}\\
&m \gtrsim  n^{\min\{2p,\frac{p(r+1)}{ r(2-p-\theta+\theta
            p)+p}\}}\log n
\end{align*}
then with high probability
	\begin{align}\label{eq:25}
	L(\wh{\beta}_{\lambda_n,m}^{cl})-L(f_*)
          &\lesssim 
            n^{-\min\{\frac{2r}{r+1},\frac{r}{r(2-p-\theta+\theta p)+p}\}}.
	\end{align}
The above bound further simplifies when the variance
bound~\eqref{variance bound} holds true  with the optimal paratemer
$\theta=1$ and  the model is well-specified  as in \eqref{target}
since we can set $r=1$.  Under these conditions, we get that
\begin{align}  \label{rate2}    
L(\widehat{\beta}_{\lambda_n,m}^{cl})-L(w_*)
    \lesssim  
  n^{-\frac{1}{{1+p}}}.
\end{align}
with the choice
\begin{equation}
	\label{params_choice}
  \lambda_n \asymp n^{-\frac{1}{1+ p}}, \quad
  \alpha_n \asymp n^{-\frac{2}{1+p}}, \quad
  m \gtrsim n^{\frac{2p}{1+p}}\log n.
\end{equation}
By comparing~bound~\eqref{rate2} with~\eqref{eq:12}, the assumption on the spectrum also leads to
an improved estimation error bound and hence improved learning rates.  
In this sense, these are the \textit{correct} error estimates since the decay of the eigenvalues is used both for the subspace approximation error and the estimation error.
As it is clear from~\eqref{rate2}, for fast eigendecay, the obtained rate goes from $O(1/\sqrt{n})$ to $O(1/n)$. Taking again, $p=1/2$ leads to a rate $O(1/n^{2/3})$ which is better than the one in~\eqref{eq:12}. In this case, the subspace defined by leverage scores needs to be chosen of dimension at least $O(n^{2/3})$. 

For arbitrary $\theta$ and $r$, bound~\eqref{eq:25}  is harder to parse. 
For $r \to 0$ the bound become vacuous and there are not enough
assumptions to derive a bound \citep{devroye2013probabilistic}. Note
that large values of $\la$ are prevented, 
indicating  a saturation effect (see
\citep{vito2005learning,mucke2019beating}). As discussed before, the bound 
improves when there is a fast eigendecay. Smaller values of $\theta$
and $r$ leads to worse bounds than~\eqref{rate2}, which is the best
rate in this context. 
 Since, given any acceptable choice of $p,r$ and $\theta$, the
 quantity $\min\{2p,\frac{p(r+1)}{ r(2-p-\theta+\theta p)+p}\}$ takes
 values in $(0,1)$, the best rate, that differently from before can
 also be slower than $\sqrt{1/n}$, can always be achieved choosing
 $m<n$ (up to logarithmic terms).

\subsection{Exponential decay of $\Sigma$}
\label{subsec exponential}
We can further improve the bounds above assuming an exponential decay~\eqref{eq:2} condition on the
spectrum of $\Sigma$. By omitting numerical constants, logarithmic and higher
order terms,  Theorem~\ref{thm: fast rate A(lambda)} implies that with high probability
\[
	L(\wh \beta^{cl}_{\la,m}) -
		L(f_*)
		\lesssim 
		\frac{\ln^2(1 / \lambda)}{n}+\sqrt{\frac{\alpha\mathcal{A}(\la)}{\lambda}}
		+\Big(\frac{\log(3/\delta)}{n}\Big)^{\frac{1}{2-\theta}}+\frac{\log(3/\delta)}{n} 
		\sqrt{\frac{\mathcal A(\la)}{\lambda}}+{\cal A}(\la). 
  \]
Under an exponential decay condition, it is reasonable to modify
the source condition controlling the behaviour of the approximation error $\A(\la)$ from polynomial to logarithmic.
We therefore assume that
\[\A(\la)\lesssim\log^{-1}(1/\la) \]
and,  with the choice
	\begin{equation}
		\label{params choice exp}
		\lambda_n \asymp \log n/n^2, \quad
		\alpha_n \asymp 1/n^2, \quad
		m \gtrsim \log^2 n ,
              \end{equation}
              with high probability,
	\begin{align*}  
		L(\widehat{\beta}_{\lambda_n,m}^{cl})-L(f_*)
		\lesssim  		1/\log n.
	\end{align*}
If the model is well-specified as in~\eqref{target} and $\theta=1$, we
get
$$
	L(\widehat{\beta}_{\lambda,m}^{cl}) -L(w_*)\lesssim \frac{\ln^2(1 / \lambda)}{n}+\la\nor{w_*}^2+\sqrt{\alpha}\nor{w_*}
	$$
 provided that $n$ and $m$ are large enough, and $\alpha\gtrsim
 e^{-n}$.  With the choice
	\begin{equation*}
		\lambda_n \asymp 1/n, \quad
		\alpha_n \asymp 1/n^2, \quad
		m \gtrsim \log^2 n ,
	\end{equation*}
	with high probability 
	\begin{align*}  
		L(\widehat{\beta}_{\lambda_n,m}^{cl})-L(w_*)
		\lesssim 	1/n.
	\end{align*}

 \begin{remark}\label{why}
Whereas the results of Section~\ref{sec:statistics-Nystrom} also hold
true for bounded inputs $X$, to have fast rates we are forced to assume the
sub-gaussianity of $X$. Under this latter condition in fact,
Lemma~\ref{lem:id_min_proj_lev_subgauss} requires \textit{only} that $\alpha\gtrsim
n^{-1/p}$ for polynomial decay and $\alpha\gtrsim e^{-n}$ for exponential
decay. These ranges are compatible with  the choices \eqref{params_choice} and \eqref{params choice exp}, which provide the optimal  convergence rates.  Under
        the assumption that $X$ is bounded,
        Lemma~\ref{lem:id_min_proj_lev_subgauss}  is replaced by Lemma
        7 in \citep{rudi2015less}, which requires instead that $\alpha\gtrsim
        n^{-1}$ both for polynomial and exponential decay, which is
        not compatible with~\eqref{params_choice}
        and \eqref{params choice exp}.
 \end{remark}

 \subsection{Comparison with Random Features}
 We begin by comparing our results with those obtained using random features, as presented in \citep{sun2018but}. Random features is a well-known technique for efficiently approximating the kernel matrix without computing it in full. Introduced in \citep{rahimi2008random}, this method maps the data into a finite-dimensional feature space, providing a random approximation of the RBF kernel feature space. By employing explicit finite-dimensional feature vectors, the original kernel support vector machine (KSVM) is converted into a linear support vector machine (LSVM). This conversion facilitates faster training algorithms, as shown in \citep{shalev2011pegasos} and \citep{hsieh2014fast}, and allows for constant-time testing relative to the number of training samples.\\
 Specifically, their Theorem~1 is based on similar assumptions as our result in eq. \eqref{rate2}, i.e. the surrogate loss is the hinge loss (Lipschitz, convex, non-differentiable, see our Assumption~\ref{ass:loss}), the Bayes predictor belongs to the RKHS ({realizable case}, see Assumption \ref{target}), Massart's low-noise condition is assumed (which implies our variance condition in Assumption \ref{ass:berstein} with $\theta=1$, see Section \ref{sec: 0-1 loss}), and the spectrum of the covariance operator decays polynomially: $\sigma_i \asymp i^{-1/p}$, $0<p<1$ (see eq.~\ref{eq:2}). Under these assumptions they obtain a rate of $n^{-1/(2p+1)}$ using $n^{2p/(2p+1)}$ random features.
 We can obtain the same rate with the same number of Nystr\"om points, but our analysis also provides an improved rate of $n^{-1/(p+1)}$ using $n^{2p/(p+1)}$ Nystr\"om points; this improvement is due to our refined analysis, allowing to consider smaller values of $\alpha$ in \eqref{params_choice}.
 We do not know
 whether this improvement comes from a better adaptivity of Nystr\"om sampling, or it is a byproduct of our analysis.
Regarding \citep{li2019towards}, comparison with their fast rates is more difficult, as they assume that the Bayes predictor belongs to the random space spanned by random features.
We do not make this strong assumption, and indeed controlling the approximation error of the random subspace is one of the key challenges in our work.

The following table provides a comparison (up to logarithmic factors) among the various rates for
the hinge loss discussed above.
\begin{table}[h]
	\caption{Comparison among the different regimes using hinge loss.}
	\centerline{%
	\begin{tabular}{lllllllll}
		\toprule
		\multicolumn{1}{c}{}   &\multicolumn{1}{c} {Assumptions} & \multicolumn{1}{c} {Eigen-decay} & \multicolumn{1}{c} {Rate} & \multicolumn{1}{c} {$m$} \\
		\cmidrule(r){1-1}	\cmidrule(r){2-2} \cmidrule(r){3-3} \cmidrule(r){4-4} \cmidrule(r){5-5} 
		Theorem \ref{thm:regularized-full-space} & \ref{ass: sub-gaussian},\ref{ass:loss},\ref{ass:best}      &             \;         /                           & $n^{-1/2}$                                                       &                     \;/                                      \\ 
		Eq.~\eqref{eq:12} & \ref{ass: sub-gaussian},\ref{ass:loss},\ref{ass:best}     & $\sigma_j\lesssim j^{-\frac{1}{p}}$    & $n^{-1/2}$                                  & $n^p$                                               \\
		Eq.~\eqref{eq:12}		& \ref{ass: sub-gaussian},\ref{ass:loss},\ref{ass:best}  & $\sigma_j\lesssim e^{-\beta j}$ & $n^{-1/2}$                                  &  $\log^2 n$          \\
		Eq: \eqref{rate2} & \ref{ass: sub-gaussian},\ref{ass:loss},\ref{target},\ref{ass:berstein}\tablefootnote{$\theta=1$} &$\sigma_j\lesssim j^{-\frac{1}{p}}$     & $n^{-\frac{1}{{1+p}}}$                                           & $n^{\frac{2p}{1+p}}$                                \\ 
		Eq: \eqref{eq:25} & \ref{ass: sub-gaussian},\ref{ass:loss},\ref{ass:berstein} & $\sigma_j\lesssim j^{-\frac{1}{p}}$    & $n^{-\min\{\frac{2r}{r+1},\frac{r}{r(2-p-\theta+\theta p)+p}\}}$ & $n^{\min\{2p,\frac{p(r+1)}{ r(2-p-\theta+\theta p)+p}\}}$ \\
		RF\tablefootnote{Here $m$ is number of random
          features}    \citep{sun2018but}       &
                                                 $\cdot$\tablefootnote{$X$ bounded},\ref{ass:loss},\ref{target},\ref{ass:berstein}$^*$  & $\sigma_j\lesssim j^{-\frac{1}{p}}$    & $n^{-\frac{1}{2p+1}}$                                            & $n^{\frac{2p}{2p+1}}$                        \\ 
		\bottomrule
	\end{tabular}
}
\end{table}

%
%
%
%
%
%


\section{Differentiable loss functions}\label{sec:other}
In this section we specify the above results to differentiable losses and, in particular, to quadratic and logistic losses. In both cases, we will provide for this setting equivalent bounds of the ones presented in Theorem \ref{thm: fast rate A(lambda)}.

\subsection{Square loss}
Next, we specialized the analysis to square loss defined
by~\eqref{eq:11} under the assumption that $\Y\subset [-1,1]$.
The interval $[-1,1]$ can be replaced by $[-M,M]$, but we take $M=1$
since, in the
following section, we will consider binary classification. It is easy to see that 
$$\ell(y,t)\leq 4, \qquad y,t\in[-1,1],$$
and $\ell$ can be clipped at $1$. A well known variance bound
for least squares loss gives that
$$
\begin{aligned}
	\left(\ell(y, f^{cl}(x))-\ell\left(y,
            f^{*}(x)\right)\right)^{2}
        &=\left(\left(f^{cl}(x)+f^{*}(x)-2
            y\right)\left(f^{cl}(x)-f^{*}(x)\right)\right)^{2} \\ & \leq
        16 \left(f^{cl}(x)-f^{*}(x)\right)^{2}, 
\end{aligned}
$$
so that variance bound~\eqref{variance bound}  holds for $V=16$ and $\theta=1$.

Finally, the least squares loss restricted to $[-1, 1]$ is Lipschitz continuous, that is
$$
\left|L(y, t)-L\left(y, t^{\prime}\right)\right| \leq 4 \left|t-t^{\prime}\right|
$$
for all $y \in[-1, 1]$ and $t, t^{\prime} \in[-1, 1]$.

The following theorem specializes to the square loss the previous states, see Appendix~\ref{app: square} for the proof. As usual the Nystr\"om
points are sampled according to ALS, see Definition~\ref{def:approx_lev_scores}. 
\begin{theorem}\label{thm:square loss}  
 Under Assumption~\ref{ass: sub-gaussian} and the polynomial decay
 condition~\eqref{eq:2}, fix $\la>0$, $\alpha\gtrsim n^{-1/p}$ and $0<\delta<1$.
	then with probability at least $1-2\delta$:
	\begin{align*}
		L(\wh \beta^{cl}_{\la,m}) -
		L(f_*)
		&\lesssim 
		\frac{1}{\la^p n}+ \frac{\alpha {\cal A}(\la)}{\la}
		+\frac{\log(3/\delta)}{n} 
		\sqrt{\frac{\mathcal A(\la)}{\lambda}}+{\cal A}(\la). 
	\end{align*}
	Furthermore, if there exists $r \in (0,1]$ such that  
	${{\cal  A}(\la)\lesssim \la^{r}}$, then 
	\begin{align*}
		 \lambda_n \asymp
		n^{-\min\{\frac{2}{r+1},\frac{1}{r+p}\}}, \qquad \alpha_n \asymp n^{-\min\{\frac{2}{r+1},\frac{1}{r+p}\}}, \qquad m \gtrsim  n^{\min\{\frac{2p}{r+1},\frac{p}{ r+p}\}}\log n
	\end{align*}
	with high probability
	\begin{align*}
		L(\wh{\beta}_{\lambda_n,m}^{cl})-L(f_*)
		&\lesssim 
		n^{-\min\{\frac{2r}{r+1},\frac{r}{r+p}\}}.
	\end{align*}
\end{theorem}

Comparing the above bound and the one in~\eqref{eq:25} with $\theta=1$, we get
the same convergence rates, but the number $m$ of Nystr\"om
points reduces from $ n^{\min\{2p,\frac{p(r+1)}{ r+p}\}}\log n$
to  $ n^{\min\{\frac{2p}{r+1},\frac{p}{ r+p}\}}\log n$, matching the
bound in \citep{rudi2015less}.

As already observed in Remark~\ref{why} we are able to prove the above results only under the assumption that $X$
sub-gaussian. However, it is possible to show that in the \textit{well
  specified case}, see Assumption~\ref{target}, corresponding to the choice
$r=1$,  the above result holds true also for bounded inputs $X$. This
is due to the additional square we get in the projection term thanks
to the quadratic properties of the loss, namely
\[L(\mathcal{P}_mw_* )-L(w_*)=\nor{\Sigma^{1/2}
    (I-\mathcal{P}_m)w_*}^2\]
so that condition $\alpha\gtrsim
n^{-1}$ in Lemma~7 in \citep{rudi2015less} can still be fulfilled for our choice of the parameter $\alpha$.  We state the result without reporting the  proof,
        which is a variant of the proof of Theorem~\ref{thm:square loss} taking into account the above remark. 
\begin{corollary}\label{thm:square loss basic}
Assume that $X$ is bounded almost surely, under
Assumption~\ref{target} and polynomial decay of the spectrum~\eqref{eq:2}, fix $\la>0$, $\alpha\gtrsim 1/n$, and $0<\delta<1$.
Then, with probability at least $1- 2{\delta}$:
	$$
	L(\widehat{\beta}_{\lambda,m}^{cl}) -L(w_*)\lesssim \frac{1}{\la^p n } + \la\nor{w_*}^2 + \alpha\nor{w_*}^2
	$$ 
	provided that $n$ and $m$ are large enough.
	Further, for ALS sampling with the choice
	\begin{equation}
		\lambda \asymp n^{-\frac{1}{1+ p}}, \quad
		\alpha \asymp n^{-\frac{1}{1+p}}, \quad
		m \gtrsim n^{\frac{p}{1+p}}\log n , 
	\end{equation}
	with high probability,
	\begin{align}      
		L(\widehat{\beta}_{\lambda,m}^{cl})-L(w_*)
		\lesssim  
		n^{-\frac{1}{{1+p}}}.
	\end{align}
\end{corollary}

\begin{table}[h]
	\caption{Comparison among the different regimes with square loss}
	\centering
	\begin{tabular}{lllllllll}
		\toprule
		\multicolumn{1}{c}{}   &\multicolumn{1}{c} {Assumptions} & \multicolumn{1}{c} {Eigen-decay} & \multicolumn{1}{c} {Rate} & \multicolumn{1}{c} {$m$} \\
		\cmidrule(r){1-1}	\cmidrule(r){2-2} \cmidrule(r){3-3} \cmidrule(r){4-4} \cmidrule(r){5-5} 
		Corollary \ref{thm:square loss basic} & \ref{ass: sub-gaussian},\ref{target}      &               $\sigma_j\lesssim j^{-\frac{1}{p}}$  & $n^{-\frac{1}{{1+p}}}$   &   $ n^{\frac{p}{1+p}}$     \\ 
		\citep{rudi2015less} & $X$ bounded, \ref{target} & $\sigma_j\lesssim j^{-\frac{1}{p}}$     & $n^{-\frac{1}{{1+p}}}$                                           & $n^{\frac{p}{1+p}}$                                \\ 
		Theorem \ref{thm:square loss} & \ref{ass: sub-gaussian}& $\sigma_j\lesssim j^{-\frac{1}{p}}$    & $n^{-\min\{\frac{2r}{r+1},\frac{r}{r+p}\}}$ & $n^{\min\{\frac{2p}{r+1},\frac{p}{ r+p}\}}$ \\

		\bottomrule
	\end{tabular}
\label{tab: square}
\end{table}


\begin{remark}[Comparison with \citep{rudi2015less}]
The comparison makes sense only when choosing $s=0$ in the source condition $\left\|\Sigma^{-s} w_*\right\|_{\mathcal{H}}<R$ in \citep{rudi2015less}. The reason is that while in \citep{rudi2015less} they study the problem in the well-specified case --improving the result when $w_*$ belongs to subspaces of $\X$ that are the images of the fractional compact operators $\Sigma^s$-- here instead we go in the opposite direction studying the case where $w_*$ does not exists and the approximation error must be introduced. The only intersection is for $s=0$ where it is reasonable to compare their bound with our Theorem \ref{thm:square loss basic}. As detailed in Table~\ref{tab: square} the two works return exactly the same rate and the same requirement for $m$.
\end{remark}

Our analysis can easily be adapted to sketching techniques other than Nystr\"om sampling as shown  in Remark~\ref{remark:proj}. For example, in the so-called Gaussian sketching the random subspace $\mathcal{B}_m$ is defined as 
\[\mathcal{B}_m=\operatorname{span} \{\sum_{j=1}^{n} G_{ij}x_j: \; 1\leq i\leq m\},\] 
where $G \in \mathbb{R}^{m \times n}$ is a random matrix with i.i.d. entries drawn from a Gaussian distribution.
To extend our results to Gaussian sketching, it is sufficient to bound the projection error term~\eqref{proj} by using Lemma 13 and the proof of Corollary 4 in \citep{lin2018optimal}, instead of  Lemma 7 in \citep{rudi2015less}.
However, we note that, when the inputs $X$ are bounded, the results in \citep{lin2018optimal} require the condition $\alpha \gtrsim n^{-1}$ on the projection parameter $\alpha$, which prevents the application of their analysis to extend our bounds in Section~\ref{sec:statistical} to Gaussian sketching.  We leave it to future work to explore whether  this condition can be relaxed by assuming sub-Gaussian inputs, as we did in Lemma~\ref{lem:id_min_proj_lev_subgauss} for Nystr\"om sampling. 

Here, as done for Nystr\"om sampling, we overcome the problem for bounded inputs by considering the square loss and the well-specified setting, namely Assumption~\ref{target} with $r = 1$ (compare with the discussion preceding Corollary~\ref{thm:square loss basic}). Under these assumptions, it is possible to prove the following corollary for Gaussian sketching, which can be compared with the results in \citep{lin2018optimal}, where their parameter $\zeta$ in their source condition with the square loss is related to our parameter $r$ by means of $2\zeta = r$ (with $0 \leq \zeta \leq 1/2$). 
	
	\begin{corollary}
	Let  $X$ be  bounded almost surely. Suppose that Assumption~\ref{target} and the polynomial decay condition~\eqref{eq:2} on of the spectrum hold true. Fix $\la>0$, $\alpha\gtrsim 1/n$, and $0<\delta<1$. Let $\mathcal{B}_m=\mathrm{span} \{\sum_{j=1}^{n} G_{ij}x_j: \; 1\leq i\leq m\}$, where $G\in \R^{m\times n}$ is a randomized matrix with i.i.d. Gaussian entries. 
	Then, with probability at least $1- 2{\delta}$:
	$$
	L(\widehat{\beta}_{\lambda,m}^{cl}) -L(w_*)\lesssim \frac{1}{\la^p n } + \la\nor{w_*}^2 + \alpha\nor{w_*}^2
	$$ 
	provided that $n$ and $m$ are large enough.
	Further, if 
	\begin{equation}
		\lambda \asymp n^{-\frac{1}{1+ p}}, \quad
		\alpha \asymp n^{-\frac{1}{1+p}}, \quad
		m \gtrsim n^{\frac{p}{1+p}}\log n , 
	\end{equation}
	with high probability,
	\begin{align}      
		L(\widehat{\beta}_{\lambda,m}^{cl})-L(w_*)
		\lesssim  
		n^{-\frac{1}{{1+p}}}.
	\end{align}
	\end{corollary}

The above result provides the same rate as in Corollary 4 in \citep{lin2018optimal} and our Corollary~\ref{thm:square loss basic}  for the Nystr\"om setting.  However,  we note that Gaussian sketching involves multiplying the full $n\times n$ Gram matrix by a random Gaussian matrix, which can be impractical when dealing with large datasets. In contrast, Nystr\"om sampling avoids constructing the full Gram matrix, requiring only a random subsampling of its columns.

\subsection{Logistic loss}
\label{sec: logistic}
As already mentioned, logistic loss defined by~\eqref{eq:10}
cannot be clipped according to~\eqref{eq: clipping} \citep{steinwart2008support}. Nevertheless, we can still clip our loss $\ell(y,a)$ at $M=\log n$ so that for all $y\in \Y$, $a\in\R$ it is easy to verify that 
\begin{equation}\label{eq: clipping_logistic}
	\ell(y,a^{cl})\leq \ell(y,a)+\frac{1}{n},
\end{equation}
where $a^{cl}$ denotes the clipped value of $a$ at $\pm \log(n)$, that is
\begin{align*}
	&a^{cl}=-\log(n) \hspace{0.6cm}\text{if} \;\;a\leq -\log(n),\\
	&a^{cl}=y \hspace{1,7cm} \text{if}\;\; a\in[-\log(n)\text{,}\log(n)], \\
	&a^{cl}=\log(n) \hspace{1cm}\text{if}\;\;  a\geq \log(n).
\end{align*}
The key point here is that, even though the loss is not always reduced
by clipping, i.e. $\exists$ $y\in\Y$, $a\in \R$
s.t. $\ell(y,a^{cl})\nleq \ell(y,a)$, it can only increase at most of
$1/n$. This is important since it does not affect the resulting bounds
on the excess risk. In particular, we recover the same excess risk bounds of the square loss in Theorem~\ref{thm:square
  loss} and Corollary~\ref{thm:square loss basic}  for the logistic
loss. The simple adaptation of proofs is given in Appendix \ref{app: logistic}.

\section{From surrogates to classification loss}\label{sec:01loss}
\label{sec: 0-1 loss}
In this section, we consider a classification task, so that $\Y=\{\pm
1\}$ and  the natural way of measuring performances is by using the 0-1 loss, i.e. $\ell_{0-1}(y,a):=\mathbb{1}_{(-\infty,0]}(y\;\text{sign}(a))$. Through out this section, we study how the previous bounds for surrogate losses relate to the 0-1 classification risk. In the following, we will indicate with $L_{0-1}$, $L_{hinge}$, $L_{square}$ and $L_{logistic}$ the risks associated respectively with 0-1, hinge, square and logistic losses. Similarly, we define $L_{0-1}^*:=\inf_f L_{0-1}(f)$, $L_{hinge}^*:=\inf_f L_{hinge}(f)$, $L_{square}^*:=\inf_f L_{square}(f)$, $L_{logistic}^*:=\inf_f L_{logistic}(f)$, where the infimum is taken over all the measurable functions
$f:\X\to\R$.

A key role will be played by the well-known low noise condition~\citep{mammen1999margin,tsybakov2004aggregation,massart2006risk}. The following definition is taken from \citep{tsybakov2004aggregation}:
\begin{definition}
	Distribution $P$ has noise exponent $ 0\leq \gamma <1$ if it satisfies one of the following conditions:
	\begin{itemize}
		\item $N_\gamma$: for some $c>0$ and all measurable $f: \X \rightarrow\{\pm 1\}$,
		\begin{equation}\label{ass:noise cond}
			\operatorname{Pr}[f(X)(2\eta(X)-1)<0] \leq c \left(L_{0-1}(f)-L_{0-1}^{*}\right)^{\gamma};
		\end{equation}
	\item  $M_{\frac{\gamma}{1-\gamma}}$: for some $c>0$ and all $\epsilon>0$,
	\begin{equation}
		\operatorname{Pr}\left[0<\left|2\eta(X)-1\right| \leq \epsilon\right] \leq c \epsilon^\frac{\gamma}{1-\gamma};
	\end{equation}
	\end{itemize}
where $\eta(X)=\operatorname{Pr}(Y=1|X)$ and for $\gamma=1$ we have that $M_\infty$ is equivalent to $N_1$.
\end{definition}

We will assume the following low-noise condition:
\begin{ass}[Low-noise condition] The distribution $P$ has noise exponent $\gamma\in [0,1]$.
	\label{ass: low-noise}
\end{ass}

Using Lemma~\ref{lem: class risk wirh gamma} in Appendix~\ref{app:
  known results}, when dealing with the square loss, there is a standard way of transforming its excess risk bound into the following
bound on the classification risk: 
\begin{lemma}[Square loss] \label{cor: square to 0-1} 
	Under Assumption~\ref{ass: low-noise}, there is a $c>0$ such that for any measurable $f: \mathcal{X} \rightarrow \mathbb{R}$ we have:
\begin{equation}
	L_{0-1}(f)-L_{0-1}^*\lesssim \left(L_{square}(f)-L_{square}^*\right)^\frac{1}{2-\gamma}.
\end{equation}
\end{lemma}
\noindent It is not hard to see that an analogous bound can be obtained for logistic loss.

For the hinge loss, the bound given by Lemma \ref{lem: from0-1_to_surr}
in Appendix~\ref{app: known results} cannot be improved even under
low noise in Assumption~\ref{ass: low-noise}. In fact, the low-noise assumption is directly connected with the variance bound \eqref{variance bound} through Theorem 8.24 in \citep{steinwart2008support} (see Lemma~\ref{lem: lownoise_bernst} in
Appendix~\ref{app: known results}). In particular, if we assume a low
noise condition with parameter $\gamma$, then the variance bound in
Assumption~\ref{ass: low-noise} is always satisfied for the hinge
loss with $\theta=\gamma$.

\subsection{From square and logistic losses to classification loss}
Starting from Theorem \ref{thm:square loss}, we can now
derive an upper bound for the classification risk using the results
obtained for the surrogate square loss. We assume low-noise condition
and exploit Lemma \ref{cor: square to 0-1} to obtain the following
theorem, where $\mathcal A_{\text{square}}(\la)$ is the approximation error,
see~\eqref{eq:24}, with respect the square loss and the Nystr\"om
points are sampled, as always, accordingly to ALS, see Definition~\ref{def:approx_lev_scores}.
\begin{theorem}\label{thm:square loss 01}
	Under Assumptions~\ref{ass: sub-gaussian} and~\ref{ass:
          low-noise} and the polynomial decay
        condition~\eqref{eq:2}, fix $\la>0$, $\alpha\gtrsim n^{-1/p}$ and $0<\delta<1$,
	then with probability at least $1-2\delta$:
	\begin{align*}
		L_{0-1}(\wh \beta^{cl}_{\la,m}) -
		L_{0-1}^*
		&\lesssim 
		\left(\frac{1}{\la^p n}+ \frac{\alpha \mathcal A_{\text{square}} (\la)}{\la}
		+\frac{\log(3/\delta)}{n} 
		\sqrt{\frac{\mathcal A_{\text{square}}(\la)}{\lambda}}+\mathcal A_{\text{square}} (\la)\right)^\frac{1}{2-\gamma}.
	\end{align*}
Furthermore, if there exists $r \in (0,1]$ such that  
	$\mathcal A_{\text{square}} (\la)\lesssim \la^{r}$ and choosing
	\begin{align*}
		\lambda \asymp
		n^{-\min\{\frac{2}{r+1},\frac{1}{r+p}\}}, \qquad
		\alpha \asymp n^{-\min\{\frac{2}{r+1},\frac{1}{r+p}\}},\qquad
		m \gtrsim  n^{\min\{\frac{2p}{r+1},\frac{p}{ r+p}\}}\log n,
	\end{align*}
	then, with high probability
	\begin{align*}
		L_{0-1}(\wh{\beta}_{\lambda,m}^{cl})-L_{0-1}^*
		&\lesssim 
		n^{-\min\{\frac{2r}{(2-\gamma)(r+1)},\frac{r}{(2-\gamma)(r+p)}\}}.
	\end{align*}
\end{theorem}
Once again analogous bounds hold for logistic loss, up to constant or negligible terms.

\subsection{From hinge loss to classification loss}
Starting from Theorem \ref{thm: fast rate
  A(lambda)}, we can derive another upper bound for the classification
risk but using as surrogate the hinge loss. Under the low noise
assumption and exploiting Lemma \ref{lem: lownoise_bernst} we obtain the following theorem, where $\mathcal A_{\text{hinge}}(\la)$ is the approximation error, 
see~\eqref{eq:24}, with respect the hinge loss.
\begin{theorem}\label{thm:hinge loss 0-1}
Under Assumptions~\ref{ass: sub-gaussian},~\ref{ass: low-noise} and under polynomial decay 
	condition~\eqref{eq:2}, fix $\la>0$, $\alpha\gtrsim n^{-1/p}$ and $0<\delta<1$,
	then with probability at least $1-2\delta$:
	\begin{align*}
		L_{0-1}(\wh \beta^{cl}_{\la,m}) -
		L_{0-1}^*
		&\lesssim 
		\left(
		\frac{1}{\la^p n} 
		\right)^{\frac{1}{2-p-\gamma+\gamma p}}+ \sqrt{\frac{\alpha \mathcal A_{\text{hinge}}(\la)}{\la}}
		+\frac{\log(3/\delta)}{n} 
		\sqrt{\frac{\mathcal A_{\text{hinge}}(\la)}{\lambda}}+{\mathcal A_{\text{hinge}}(\la)}. 
	\end{align*}
	Furthermore, if there exists $r \in (0,1]$ such that  
	${\mathcal A_{\text{hinge}}(\la)\lesssim \la^{r}}$ and choosing
	\begin{align*}
		\lambda \asymp
		n^{-\min\{\frac{2}{r+1},\frac{1}{r(2-p-\gamma+\gamma
				p)+p}\}}, \qquad
		\alpha \asymp n^{-\min\{2,\frac{r+1}{r(2-p-\gamma+\gamma p)+p}\}},\qquad
		m \gtrsim  n^{\min\{2p,\frac{p(r+1)}{ r(2-p-\gamma+\gamma
				p)+p}\}}\log n,
	\end{align*}
	then, with high probability
	\begin{align*}
		L_{0-1}(\wh{\beta}_{\lambda,m}^{cl})-L_{0-1}^*
		&\lesssim 
		n^{-\min\{\frac{2r}{r+1},\frac{r}{r(2-p-\gamma+\gamma p)+p}\}}.
	\end{align*}
\end{theorem}	

\begin{table}[h]
	\caption{Comparison between the $0-1$ classification risk derived from square, logistic and hinge loss under low noise condition}
	\centerline{%
	\begin{tabular}{lllllllll}
		\toprule
		\multicolumn{1}{c}{}  &\multicolumn{1}{c} {Assump}  & \multicolumn{1}{c} {Eigen-decay} & \multicolumn{1}{c} {Rate} & \multicolumn{1}{c} {$m$} \\
		\cmidrule(r){1-1}	\cmidrule(r){2-2} \cmidrule(r){3-3} \cmidrule(r){4-4} \cmidrule(r){5-5} 
		\textit{Square:} Theorem \ref{thm:square loss 01} & \ref{ass: sub-gaussian},\ref{ass: low-noise} &  $\sigma_j\lesssim j^{-\frac{1}{p}}$    & $n^{-\min\{\frac{2r}{(2-\gamma)(r+1)},\frac{r}{(2-\gamma)(r+p)}\}}$ & $n^{\min\{\frac{2p}{r+1},\frac{p}{ r+p}\}}$ \\
		\textit{Logistic }  & \ref{ass: sub-gaussian},\ref{ass: low-noise} &  $\sigma_j\lesssim j^{-\frac{1}{p}}$    & $n^{-\min\{\frac{2r}{(2-\gamma)(r+1)},\frac{r}{(2-\gamma)(r+p)}\}}$ & $n^{\min\{\frac{2p}{r+1},\frac{p}{ r+p}\}}$ \\
		\textit{Hinge:} Theorem \ref{thm:hinge loss 0-1} & \ref{ass: sub-gaussian},\ref{ass: low-noise}    & $\sigma_j\lesssim j^{-\frac{1}{p}}$  & $n^{-\min\{\frac{2r}{r+1},\frac{r}{r(2-p-\gamma+\gamma p)+p}\}}$   &   $ n^{\min\{2p,\frac{p(r+1)}{ r(2-p-\gamma+\gamma
				p)+p}\}}$     \\ 	
		\bottomrule
	\end{tabular}
}
\end{table}
Next, we will discuss the results obtained in Table $4$.
\subsection{Discussion of the results}
\label{sec:disc}
We compare now the two upper bounds we obtained in Theorem \ref{thm:square loss 01} and Theorem \ref{thm:hinge loss 0-1}.
Since
\begin{equation*}
  \min \Big\{\frac{2r}{(2-\gamma)(r+1)},\frac{r}{(2-\gamma)(r+p)} \Big\}
  \leq \min \Big\{\frac{2r}{r+1},\frac{r}{r(2-p-\gamma+\gamma p)+p} \Big\}
\end{equation*}
for all the choices of $p$, $\gamma$ and $r$ the bound for the classification error derived using the hinge loss can always achieve a better rate than the one derived from the square loss. On the other hand, since $\min\{\frac{2p}{r+1},\frac{p}{ r+p}\} \leq \min\{2p,\frac{p(r+1)}{ r(2-p-\gamma+\gamma
	p)+p}\}$, the choice of the hinge loss results, according to our upper bounds, to be more expensive in term of $m$ (while achieving a better rate). Therefore, we can try to compare the two rates while fixing the number of number of Nystr\"om points selected, or, viceversa, we can fix the rate and compare the number of Nystr\"om points needed to achieve it. The results here are less obvious and we do not have a clear winner. What appears from the analysis is that the discriminant is the choice of the low noise condition parameter $\gamma$ and the $r$ parameter, which controls the approximation error decay. 

Let us fix an achievable convergence rate $O\left(n^{-R}\right)$ for the classification risk.
To achieve this rate we need at least $m_{\text{s}}=n^{R(2-\gamma)p/r}$ for square loss and $m_{\text{h}}=n^{R(1+r)p/r}$ for hinge loss. Since when $\gamma+r<1$ then $m_h\leq m_s$, we have that using hinge is, according to our upper bounds, computationally \textit{cheaper} than using the square loss (see Figure~\ref{fig: m}). This means that when the problem is hard, hinge loss seems to be also \textit{less expensive} than the square loss (and viceversa) in terms of number of required Nystr\"om points.

Similarly, imagine now to have some budget constraint on $m$ so that we are not allowed to choose its optimal value: which loss will show a faster rate? Again the condition on $\gamma+r$ is the key, with the upper bound for hinge loss achieving a faster rate than the one for square loss, when $\gamma+r<1$.

In summary, when studying the misclassification error using surrogates, the comparison between our two upper bounds obtained from hinge and square loss does not suggest an univocal better choice between the two losses for all regimes. When the problem is \textit{hard}, i.e. slow decay of the approximation error ($\la\ll 1$) and/or strong noise  ($\gamma \ll 1)$, the upper bound for hinge loss behaves better than the one square loss; the opposite when the problem is \textit{easy}.


\begin{figure}[h]
\centering
\includegraphics[scale=0.25]{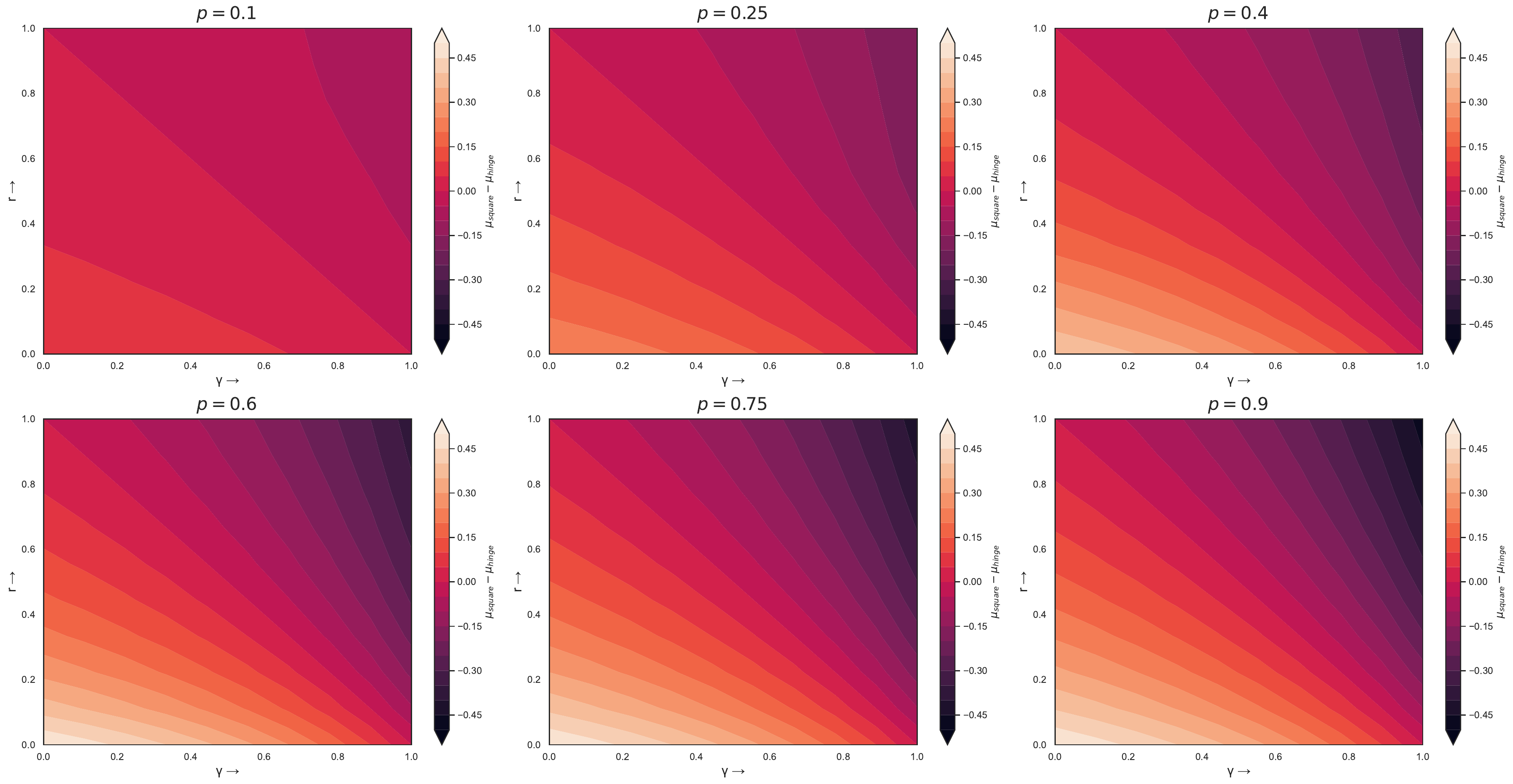}
\caption{Comparison between the number of Nystr\"om points needed by square and hinge loss to get a fixed common rate: the plots above show $\mu_{\text{square}}-\mu_{\text{hinge}}$, where $0\leq\mu\leq1$ is the exponent controlling $m$, i.e. $m\asymp n^{\mu}$. Light colours represent then the regimes where hinge loss is \textit{cheaper} than square loss.}
\label{fig: m}
\end{figure}


\section{Experiments}
\label{sec:experiments}

\begin{figure*}[h!]
	\centering
	\includegraphics[width=7.5cm]{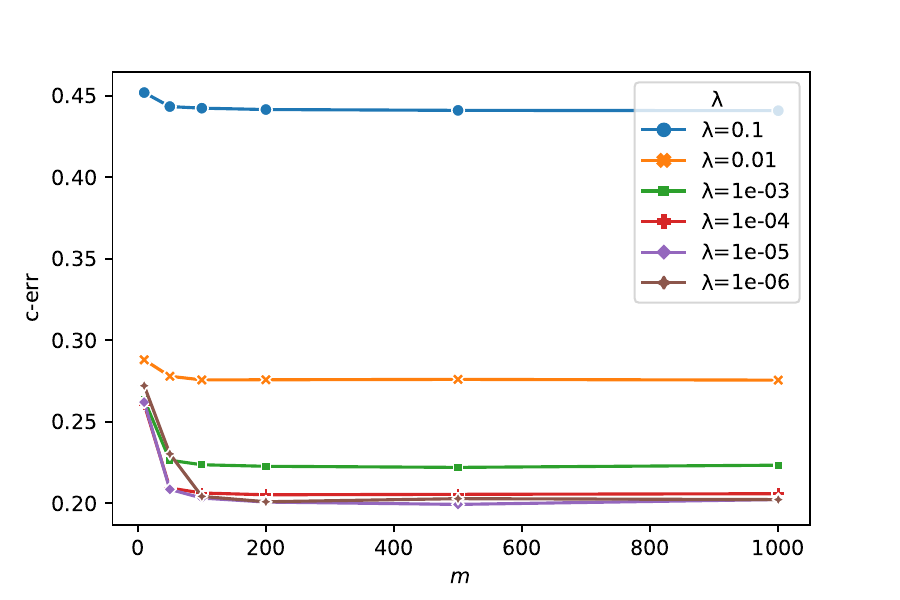}
	\includegraphics[width=7.5cm]{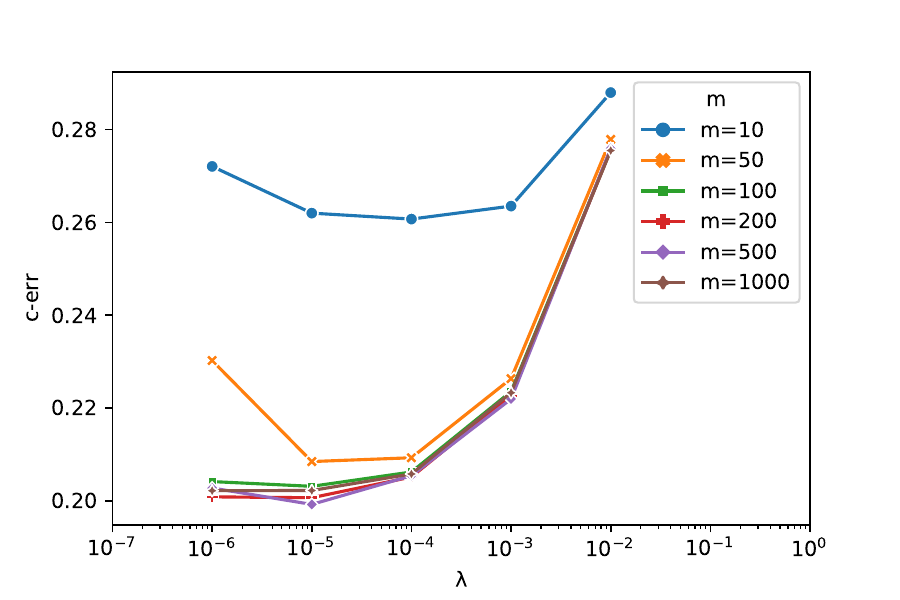}
	\includegraphics[width=7.5cm]{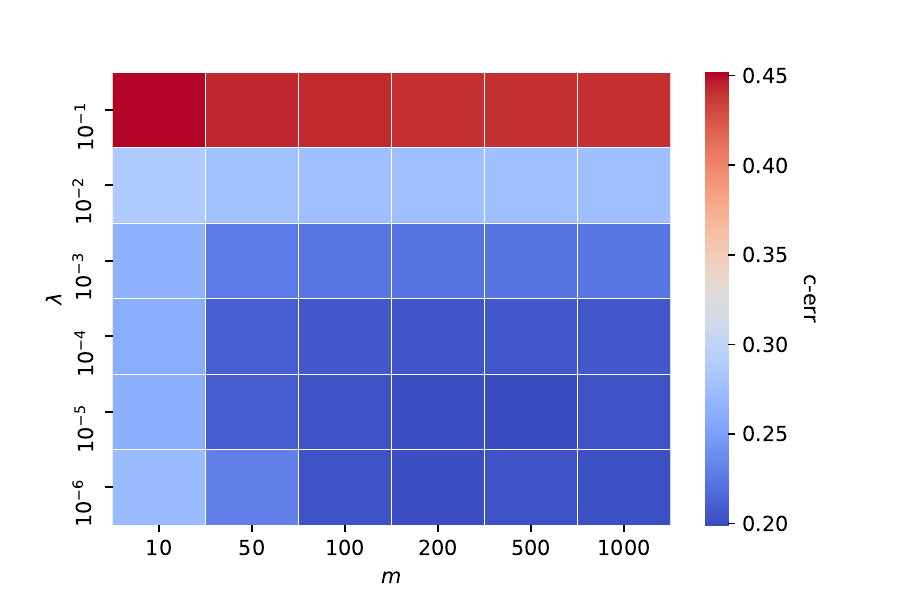}
	\caption{The graphs above are obtained from SUSY dataset: on the top left we show how c-err measure changes for different choices of $\lambda$ parameter; top right figure focuses on the stability of the algorithm varying $\lambda$; on the bottom the combined behavior is presented with a heatmap.}
	\label{fig}
\end{figure*}

As mentioned in the introduction, a main of motivation for our study
is showing that the computational savings can be achieved without
incurring in any loss of accuracy. In this section, we complement our
theoretical results investigating numerically the statistical and
computational trade-offs in a relevant  setting. More precisely, we
report simple experiments in the context of kernel methods,
considering Nystr\"om techniques. In particular, we choose the hinge
loss, hence SVM for classification. Keeping in mind
Theorem~\ref{thm: fast rate A(lambda)} we expect we can match the performances of
kernel-SVM  using a  Nystr\"om approximation with only $m\ll n$
centers.  The exact number depends on assumptions, such as the
eigen-decay of the covariance operator, that might be hard to know in
practice, so here we explore this empirically.

%
%
%

\paragraph{Nystr\"om-Pegasos.}
Classic SVM  implementations with hinge loss are based on considering a dual formulation and a quadratic programming problem \citep{joachims1998making}.
This is the case for example, for the LibSVM library \citep{chang2011libsvm} available on Scikit-learn \citep{pedregosa2011scikit}. We use this implementation for comparison, but find it convenient to combine the Nystr\"om method to a primal solver akin to~\eqref{subgsvm} (see \citep{li2016fast,hsieh2014fast} for the dual formulation).
More precisely, we use Pegasos \citep{shalev2011pegasos}  which is based on  a  simple and easy to use stochastic subgradient iteration\footnote{Python
  implementation from \url{https://github.com/ejlb/pegasos}}. We
consider a procedure in two steps. First, we compute the
embedding discussed in Section~\ref{sec: ERM on random
  subspace}. With kernels it takes the form  
$
\xx_i=(K_m^{\dagger})^{1/2} (K(x_i,\tilde{x}_1),\dots,K(x_i,\tilde{x}_m))^T,
$
where $K_m\in \R^{m\times m}$ with $(K_m)_{ij}=K(\tilde{x}_i,\tilde{x}_j)$. 
Second, we  use  Pegasos on the embedded data.
As discussed in Section~\ref{sec: ERM on random subspace}, the total cost is $O(nm^2C_K+nm\cdot \#iter)$ in time (here $\text{iter}=\text{epoch}$, i.e. one epoch equals  $n$ steps of 
stochastic subgradient) and $O(m^2)$ in memory (needed to compute the pseudo-inverse and  embedding the data in batches of size $m$).
        
        \begin{table*}[h!]
        	\caption[caption table]{Architecture: single machine with AMD EPYC 7301 16-Core Processor and 256GB of RAM. For Nystr\"om-Pegaos, ALS sampling has been used (see \citep{rudi2018fast}) and the results are presented as mean and standard deviation deriving from 5 independent runs of the algorithm. The columns of the table report classification error, training time and prediction time (in seconds).}
        	\label{tab:results}
        	\centerline{%
        		\begin{tabular}{lllllllll}
        			\toprule
        			\multicolumn{1}{c}{}  &  \multicolumn{1}{c} {LinSVM} &\multicolumn{3}{c} {KSVM} &\multicolumn{4}{c} {Nystr\"om-Pegasos (ALS)}             \\ 
        			\cmidrule(r){1-1}	\cmidrule(r){2-2}	\cmidrule(r){3-5}\cmidrule(r){6-9}
        			Datasets       & c-err & c-err & t train & t pred & c-err &t train & t pred& $m$\\
        			\cmidrule(r){1-1}	\cmidrule(r){2-2}	\cmidrule(r){3-5}\cmidrule(r){6-9}
        			SUSY & $28.1\% $ &    - &-&-&$20.0\% \pm 0.1 \%$&$ 608\pm 2$& $134\pm 4$ & $2500$\\
        			Mnist    & $12.4\% $& $2.2\% $& $1601$&$87$ & $2.2\% \pm 0.1 \%$ &$ 1342\pm 5 $& $ 491 \pm 32 $ & $15000$  \\
        			Usps    &$16.5\%$     &$3.1\%$   & $4.4$& $1.0$& $3.0\%\pm 0.1  \%$ & $ 19.8 \pm 0.1 $ &$7.3 \pm 0.3 $& $2500$ \\
        			Webspam    & $8.8\%$  &$1.1\%$ & $6044$& $473$& $1.3\% \pm 0.1\%$    & $2440 \pm 5$& $376 \pm 18$ & $11500$ \\
        			a9a & $16.5\%$ & $15.0\%$& $114$ & $31$&$15.1\%\pm 0.1\%$ & $29.3\pm 0.2$& $1.5\pm 0.1$ &$800$ \\
        			CIFAR  & $31.5\%$&$19.1\%$ &$6339$ &$213$&$19.2\%\pm 0.1\%$ & $2408\pm 14$& $820\pm 47$& $20500$ \\
        			\bottomrule
        		\end{tabular}
        	}
        \end{table*}

\paragraph{Datasets \& setup (see  Appendix~\ref{appexp}).}      
We consider five  datasets\footnote{Datasets available from
  LIBSVM website \url{http://www.csie.ntu.edu.tw/\textasciitilde cjlin/libsvmtools/datasets/} and from \citep{jose2013local} \url{http://manikvarma.org/code/LDKL/download.html\#Jose13}}
of size  $10^4-10^6$, challenging for standard SVMs.
	{We use  a Gaussian kernel,  tuning   width and  regularization parameter as explained in appendix. 
          We report classification error and for data sets with no fixed test set, we set apart $20\%$ of the data.}
        
\paragraph{Procedure.}
Given the accuracy achieved by K-SVM algorithm, which is our target, we increase the number of sampled  Nystr\"om points $m<n$ as long as also Nystr\"om-Pegasos matches that result.

\paragraph{Results.}        
We compare with linear (used only as baseline) and K-SVM see Table~\ref{tab:results}.  For all the datasets, the Nystr\"om-Pegasos approach achieves comparable performances of K-SVM with much better time requirements (except for the small-size Usps).  Moreover, note that K-SVM cannot be run  on millions of points (SUSY), whereas Nystr\"om-Pegasos is still fast and provides much better results than linear SVM. Further comparisons with state-of-art algorithms for SVM are left for a future work. Finally, in Figure~\ref{fig} we illustrate the interplay between $\la$ and $m$ for the Nystr\"om-Pegasos considering SUSY dataset. In Appendix~\ref{appexp} we compare also with results obtained using the simpler uniform sampling of the points.

\begin{table}[h]
	\caption[caption table]{Comparison between Nystr\"om-Pegasos (hinge loss) and Nystr\"om-KRR (square loss) when using ALS sampling. We report the respective classification errors fixing the number of Nystr\"om centers. }
	\centering
	\label{tab:comparisonALS}
	\begin{tabular}{lllllllll}
		\toprule
		\multicolumn{1}{c}{}   &\multicolumn{2}{c} {Nystr\"om-Pegasos (ALS)} &\multicolumn{2}{c} {Nystr\"om-KRR (ALS)}             \\ 
		\cmidrule(r){1-1}	\cmidrule(r){2-3}\cmidrule(r){4-5}
		Datasets        & c-err & $m$  & c-err &$m$ \\
		
		\cmidrule(r){1-1}	\cmidrule(r){2-3}\cmidrule(r){4-5}
		SUSY  &$20.0\% \pm 0.1 \%$&$2500$& $19.9\%\pm 0.1\%$& $2500$\\
		Mnist    & $2.2\% \pm 0.1 \%$ &$15000$ & $ 2.4\%\pm 0.1\%$ & $15000$\\
		Usps         & $3.0\%\pm 0.1  \%$ &  $2500$& $2.9\%\pm 0.1\% $& $2500$ \\
		Webspam      & $1.3\% \pm 0.1\%$    &$11500$ &  $1.3\%\pm 0.1\%$  & $11500$  \\
		a9a  &$15.1\%\pm 0.1\%$ & $800$ & $14.9\%\pm 0.1\%$& $800$\\
		CIFAR  &$19.2\%\pm 0.1\%$ &  $20000$& $19.0\%\pm 0.1\%$ & $20000$\\
		\bottomrule
	\end{tabular}
\end{table}

\paragraph{Comparison between Nystr\"om-Pegasos and Nystr\"om KRR} We finally want to test the theoretical results discussed in Section \ref{sec:disc} with real data. We compare the Nystr\"om-Pegasos algorithm (Nystr\"om SVM), i.e. ERM with Nystr\"om projection when using the hinge loss as surrogate, against Nystr\"om KRR, i.e. ERM with Nystr\"om projection when using the square loss as surrogate. Extensive experimental analysis for Nystr\"om KRR in regression problems can be found in \cite{rudi2015less,rudi2017falkon,meanti2020kernel}. Following the discussion in Section \ref{sec:disc}, here we are instead interested in comparing the two methods in classification problems. We follow the scenario described in the second part of Section~\ref{sec:disc}, where the budget of Nystr\"om centers $m$ is fixed, and we compare the respective classification errors. As theory suggests there is not a clear winner between the two methods for all data distributions, as shown in Table~\ref{tab:comparisonALS}. Results are always similar, with Nystr\"om-Pegasos slightly outperforming Nystr\"om-KRR on Mnist, while the outcome is reversed on a9a and CIFAR. Analogous results can be found in Table~\ref{tab:comparison} in Appendix~\ref{appexp}.

\section*{Acknowledgments}
L. R. acknowledges the financial support of the European Research Council (grant SLING 819789), the European Commission (Horizon Europe grant ELIAS 101120237), the US Air Force Office of Scientific Research (FA8655-22-1-7034), the Ministry of Education, University and Research (FARE grant ML4IP R205T7J2KP; grant BAC FAIR PE00000013 funded by the EU - NGEU) and the Center for Brains, Minds and Machines (CBMM), funded by NSF STC award CCF-1231216. This work represents only the view of the authors. The European Commission and the other organizations are not responsible for any use that may be made of the information it contains. The research by E.D.V. and L. R. has been supported by the MIUR grant PRIN 202244A7YL. The research by E.D.V. has been supported by the MIUR Excellence Department Project awarded to Dipartimento di Matematica, Universita di Genova, CUP D33C23001110001. E.D.V. is a member of the Gruppo Nazionale per l’Analisi Matematica, la Probabilità e le loro Applicazioni (GNAMPA) of the Istituto Nazionale di Alta Matematica (INdAM). This work represents only the view of the authors. The European Commission and the other organizations are not responsible for any use that may be made of the information it contains.
J.M. is supported by a grant of the French National Research Agency (ANR), “Investissements d’Avenir” (LabEx Ecodec/ANR-11-LABX-0047).

\bibliography{nystrom}

\onecolumn
\appendix

 \section{Proof of Section~\ref{sec:ERM}}  \label{sec:radem-compl}

This section is devoted to the proof of Theorems~\ref{thm:excessrisk-erm-standard} and~\ref{thm:regularized-full-space}.
With slight abuse of notation we set
  \[
\ell(w,z)= \ell(y,\scal{w}{x}), \qquad z=(x,y)\in \X\times\Y,\ w\in\X.
\]
With this notation  $L(w)=
  \int_{\X\times\Y}\ell(w,z)dP(z)$.

The following result is known,
\citep[Lemma 8.1]{alquier2019estimation}. We provide an
alternative proof tailored to the Hilbert setting.  
\begin{lemma}
	\label{lem:gen-gap-ball-gauss}
	Under Assumptions~\ref{ass:
		sub-gaussian} and~\ref{ass:loss},  fix 
	$R > 0$ and $\tau>0$, with probability at least $1 -\delta $, 
	\begin{equation}
		\label{eq:gen-gap-ball-bound-gauss}
		\sup_{ \| w \| \leq R} \big| \wh L (w)-L(w) \big|
		<  \frac{D}{\sqrt{n}}\Big( GR C  \|\Sigma\|^{\frac{1}{2}}\big(  \sqrt{r_\Sigma}+\sqrt{\log (4/\delta)}
		\big)+
		\ell_0 \sqrt{\log (4/\delta)}\Big) ,
	\end{equation}
	where $D>0$ is an absolute
	numerical constant and $r_\Sigma=\tr{\Sigma}/\|\Sigma\|$ is the effective rank of $\Sigma$.  Furthermore, 
	for each $w\in\X$, $\wh L (w)-L(w)$
	is a sub-gaussian centered real random
	variable  and
	\begin{equation}
		\label{eq:23}
		\| \wh L (w)-L(w)\|_{\psi_2} \leq
		\frac{2}{\sqrt{n}} ( \ell_0 +   C
		G \|\scal{X}{w}\|_2) .
	\end{equation}
\end{lemma}
\begin{proof}
	In the proof $D$ denotes an absolute numerical constant, whose value
	can change from line to line.   Fix $w\in \X$ and define the centered real random variable
	\[
	Z_w = \ell(Y,\scal{X}{w}) - \E[ \ell(Y,\scal{X}{w})].
	\]
	We claim that, for any pair $w,w'\in\X$
	\begin{equation}
		\label{eq:14}
		\| Z_w - Z_{w'} \|_{\psi_2} \leq 2C G\|\scal{X}{w-w'}\|_2,
	\end{equation}
	where $\| Z_w - Z_{w'} \|_{\psi_2}$ is defined
	by~\eqref{eq:2}. Indeed,  for all $p\geq 1$,  recalling that  $\|  \xi\|_p = \E[
	|\xi|^p]^{\frac{1}{p}}$, then triangular
	inequality and continuity of expectation give
	\begin{align*}
		\| Z_w - Z_{w'} \|_p & \leq  \| \ell(Y,\scal{X}{w}) - \ell(Y,\scal{X}{w'}) \|_p +
		\| \ell(Y,\scal{X}{w}) - \ell(Y,\scal{X}{w'}) \|_1\\
		& \leq 2  \| \ell(Y,\scal{X}{w}) -
		\ell(Y,\scal{X}{w'}) \|_p\\
		& \leq 2 G \| \scal{X}{w-w'}) \|_p \leq 2G  C\sqrt{p} \| \scal{X}{w-w'}) \|_2
	\end{align*}
	where the last two inequalities are consequence of~\eqref{eq:5} and
	~\eqref{def: subgauss},
	respectively. Hence
	\[ \sup_{p\geq 2} \frac{\| Z_w -
		Z_{w'} \|_p}{ \sqrt{p}}\leq 2G  C\|\langle X, w-w'\rangle\|_2,\]
	so that ~\eqref{eq:14} is clear.
	Furthermore, since
	\begin{align*}
		\big(\wh{L}(w)- L(w))-(\wh{L}(w')-
		L(w')\big) & =\frac{1}{n}\sum_{i=1}^n\left(
		(\ell(Y_i,\scal{X_i}{w}) - \E[
		\ell(Y_i,\scal{X_i}{w})])
		\right.\\
		& \qquad -\left. (\ell(Y_i,\scal{X_i}{w'}) - \E[
		\ell(Y_i,\scal{X_i}{w'})])\right)
	\end{align*}
	is a sum of independent sub-gaussian
	random variables distributed as $(Z_w - Z_w')/n$,
	then by rotational invariance theorem \citep[Proposition
	2.6.1]{vershynin2018high}
	\begin{equation}
		\|(\wh{L}(w)- L(w))-(\wh{L}(w')- L(w'))\|_{\psi_2}\leq  \frac{D}{\sqrt{n}} \| Z_w - Z_{w'}
		\|_{\psi_2} \leq \frac{D}{\sqrt{n}} C G \| \scal{X}{w-w'})
		\|_2,\label{eq:18}   
	\end{equation}
	where the last inequality is a consequence of~\eqref{eq:14} and $D$ is
	an absolute constant. Consider $\X$ as a metric space  with respect to
	the metric
	\[
	d(w,w')= \|\scal{X}{w-w'}\|_2
	\]
	where without loss  of generality we assume that $\Sigma$ is
	injective, then~\eqref{eq:18} states that the centered random process
	$\big( \wh{L}(w)-
	L(w)\big)_{w\in\X}$  has
	sub-gaussian increments and the generic chaining tail bound
        \citep[Theorem 8.5.5]{vershynin2018high}
	implies that, with probability at least $1-2 e^{-\tau}$,
	\begin{equation}
		\sup_{w,w'\in B_R}\big|(\wh{L}(w)- L(w))-(\wh{L}(w')- L(w'))\big|\leq 
		\frac{D}{\sqrt{n}} C G \big( \sqrt{\tau} \diam(B_R) +
		\gamma_2(B_R)\big),\label{eq:20}
	\end{equation}
	where $B_R=\{w\in\X: \|w\|\leq R\}$, $\diam(B_R)$ and
	$\gamma_2(B_R)$ are the diamater with respect to the metric $d$ and
	the Talagrand's $\gamma_2$
	functional of $B_R$,  \citep[Definition 8.5.1]{vershynin2018high}.
	
	Let $G$ be the Gaussian random vector in $\X$ with covariance
	$\Sigma$, which always exists since $\Sigma$ is a trace class
	operator. Talagrand's majorizing measure theorem~\citep[Theorem
	8.6.1]{vershynin2018high} implies that
	\begin{align*}
		\gamma_2(B_R) \leq D \E[\sup_{w\in B_R} \scal{G}{w} ]= \E[ \sup_{w\in B_R}
		|\scal{G}{w}| ]= R\,\E[ \|G\| ]\leq R\, \E[ \|G\|^2]^{\frac{1}{2}} =R
		\, \tr(\Sigma) ^{\frac{1}{2}} ,
	\end{align*}
	where  the first equality is due to the fact that $B_R$ is symmetric,
	the second inequality is a consequence of Jansen inequality and the
	last equality by definition of $G$. Furthermore, the definition of $d$
	gives that
	\[
	\diam(B_R) \leq 2 R \|\Sigma\|^{\frac{1}{2}}.
	\]
	Plugin these last  two bounds in~\eqref{eq:20}, it holds that
	\begin{equation}
		\sup_{w,w'\in B_R}\big|(\wh{L}(w)- L(w))-(\wh{L}(w')- L(w'))\big|\leq 
		\frac{D}{\sqrt{n}} C G R \big( \sqrt{\tau}  \|\Sigma\|^{\frac{1}{2}}+ \tr(\Sigma) ^{\frac{1}{2}}
		\big).\label{eq:21}
	\end{equation}
	with high probability.  Finally, observe that 
	\[
	|\ell(Y,0) - \E[\ell(Y,0)])| \leq 2 \sup_{y\in Y} \ell(y,0) =2\ell_0,
	\]
	by~\eqref{eq:5}, and
	\[
	\wh{L}(0)- L(0) = \frac{1}{n} \sum_{i=1} ( \ell(Y_i,0) - \E[\ell(Y_i,0)]) 
	\]
	so that Hoeffding's
	inequality~\citep{boucheron2013concentration}
	implies that, with   probability $1- 2 e^{-\tau}$, 
	\begin{equation}\label{eq:22}
		|\wh{L}(0)- L(0) |\leq  2\ell_0\sqrt{\frac{2 \tau}{n}}.
	\end{equation}
	Finally, since
	\[
	\sup_{w\in B_R} |\wh{L}(w)- L(w)| \leq \sup_{w\in B_R}|\wh{L}(w)- L(w)
	-(\wh{L}(0)- L(0)) |  + |\wh{L}(0)- L(0) |
	\]
	bounds~\eqref{eq:21} and~\eqref{eq:22}
	give~\eqref{eq:gen-gap-ball-bound-gauss} with $4\exp(-\tau)=\delta$.
	Bound~\eqref{eq:18} with $w'=0$ implies~\eqref{eq:23}. 
\end{proof}

This result cannot be readily applied to $\wh
	w_\lambda$, since its norm $\| \wh w_\lambda \|$ is itself random.
	Observe that, by definition and by Assumption~\ref{ass:loss},
	\[
	\lambda \| \wh w_\lambda \|^2 \leq \wh L_\lambda (\wh w_\lambda) \leq
	\wh L_\lambda (0) = \wh L (0) \leq \sup_{y\in\Y} \ell(y,0) =\ell_0 ,
	\]
	so that $\| \wh w_\lambda \| \leq
	\sqrt{\ell_0/\lambda}$.
	One could in principle apply this bound on $\wh{w}_\la$, but this would yield a suboptimal dependence on $\lambda$ and thus a suboptimal rate.	
	
	The next step in the proof is to make the bound of Lemma~\ref{lem:gen-gap-ball-gauss} valid for all norms $R$, so that it can be applied to the random quantity $R = \| \wh w_\lambda \|$.
	This is done in Lemma~\ref{lem:gen-gap-union} below though a union bound.
	
	\begin{lemma}
		\label{lem:gen-gap-union}
		Under Assumptions~\ref{ass: sub-gaussian} and~\ref{ass:loss}, $\forall w \in \H$, with probability $1 - \delta$:
		\begin{align*}
			L (w) - \wh L (w)
			\leq& \frac{DGC  \|\Sigma\|^{\frac{1}{2}}(1+\nor{w}) \sqrt{r_\Sigma}}{\sqrt{n}} +\\
			& \frac{D}{\sqrt{n}}\Big( GC  \|\Sigma\|^{\frac{1}{2}}(1+\nor{w}) +\ell_0\Big)\sqrt{\log(2+\log_2(1+\nor{w}))+\log(1/\delta)}.
		\end{align*}
	\end{lemma}
	
	\begin{proof}
		Fix $\delta \in (0, 1)$.
		For $p \geq 1$, let $R_p :=  2^p$ and $\delta_p = \delta / ( p (p+1) )$.
		By Lemma~\ref{lem:gen-gap-ball-gauss}, one has for every $p \geq 1$,
		\begin{equation*}
			\P \Bigg(\sup_{\Vert w \Vert\leq R_p}  \Big[ L (w) - \wh L (w) \Big]
			\geq \frac{D}{\sqrt{n}}\Big( GR_p C  \|\Sigma\|^{\frac{1}{2}}\big(  \sqrt{r_\Sigma}+\sqrt{\log (1/\delta_p)}
			\big)+
			\ell_0 \sqrt{\log (1/\delta_p)}\Big)\Bigg)\leq\delta_p.
		\end{equation*}
		Collecting the terms containing $\delta_p$ and taking a union bound over $p\geq 1$ while using that $\sum_{p\geq 1}\delta_p=\delta$ and $\delta_p \geq \delta^2 / (p+1)^2$, we get:
		\begin{equation*}
			\P \Bigg(\exists p\geq 1, \; \sup_{\Vert w\Vert\leq R_p} \Big[ L (w) - \wh L (w) \Big]
			\geq\frac{D}{\sqrt{n}}\Big( GR_p C  \|\Sigma\|^{\frac{1}{2}}\big(  \sqrt{r_\Sigma}+\sqrt{\log \frac{p+1}{\delta}}
			\big)+
			\ell_0 \sqrt{\log \frac{p+1}{\delta}}\Big)\Bigg)\leq\delta.          
		\end{equation*}
		Now, for $w \in \H$, let $p=\lceil \log_2(1+\Vert w\Vert) \rceil$; then, $1+\Vert w \Vert \leq R_p = 2^p \leq 2 (1+\Vert w \Vert)$, so $\| w \| \leq R_p$.
		Hence, $\forall w \in \H$, with probability $1 - \delta$:
		\begin{align*}
			L (w) - \wh L (w) &
			\leq \frac{DGC  \|\Sigma\|^{\frac{1}{2}}(1+\nor{w}) \sqrt{r_\Sigma}}{\sqrt{n}} + \frac{D}{\sqrt{n}}\sqrt{\log \frac{p+1}{\delta}}\Big( GC  \|\Sigma\|^{\frac{1}{2}}(1+\nor{w}) +\ell_0\Big)\\
			&\leq \frac{DGC  \|\Sigma\|^{\frac{1}{2}}(1+\nor{w}) \sqrt{r_\Sigma}}{\sqrt{n}} +\\
			&\quad \;+ \frac{D}{\sqrt{n}}\Big( GC  \|\Sigma\|^{\frac{1}{2}}(1+\nor{w}) +\ell_0\Big)\sqrt{\log(2+\log_2(1+\nor{w}))+\log(1/\delta)}\nonumber\\
			&\leq\delta.
		\end{align*}
		This is precisely the desired bound.    
	\end{proof}
	
	We are now able to prove the two theorems.
	\begin{proof}[Proof of Theorem~\ref{thm:excessrisk-erm-standard}]
		Since the bound of Lemma~\ref{lem:gen-gap-union} holds
		simultaneously for all $w \in \H$, one can apply it to
		$\wh w_\lambda$; using the inequality
		$ \| \wh w_\lambda \| \leq  \sqrt{\ell_0 / \lambda} \leq
		(1 + \ell_0/\lambda)/2$ to bound the $\log \log $ term,
		this gives with probability $1 - \delta$,
		\begin{align}
			\label{eq:gen-gap-wlambda}
			L (\wh w_\lambda) - \wh L (\wh w_\lambda)
			\leq & \frac{DGC  \|\Sigma\|^{\frac{1}{2}}(1+\nor{\wh w_\lambda}) \sqrt{r_\Sigma}}{\sqrt{n}} + \nonumber \\
			& \frac{D}{\sqrt{n}}\Big( GC  \|\Sigma\|^{\frac{1}{2}}(1+\nor{\wh w_\lambda}) +\ell_0\Big)\sqrt{\log(1+\log_2(3+\ell_0/\la))+\log(1/\delta)}
			.
		\end{align}
		Now, let
		$K = K_{\lambda, \delta} =  \sqrt{\log (1 + \log_2 ( 3 +
			{\ell_0/\lambda} ) ) + \log (1/\delta)}
		$. 
Eq~\eqref{eq:gen-gap-wlambda} writes
		\begin{equation}
		\label{eq:lemma3 simple}
		L (\wh w_\lambda) - \wh L (\wh w_\lambda) \leq  \frac{DGC\|\Sigma\|^{\frac{1}{2}}(1+\nor{\wh w_\lambda})\sqrt{r_\Sigma}}{\sqrt n}+\frac{DK}{\sqrt{n}}\Big( GC  \|\Sigma\|^{\frac{1}{2}}(1+\nor{\wh w_\lambda}) +\ell_0\Big)
		\end{equation}
		Using that   $a b \leq \lambda a^2 + b^2/(4 \lambda)$ for $a, b \geq 0$, one can
		then write
		\begin{align}
			L (\wh w_\lambda)
			&\leq \wh L (\wh w_\lambda) + \sqrt{r_\Sigma}\frac{DGC \|\Sigma\|^{\frac{1}{2}}(1+\nor{\wh w_\lambda})}{\sqrt{n}} +K\frac{DGC \|\Sigma\|^{\frac{1}{2}}(1+\nor{\wh w_\lambda})}{\sqrt n} +\frac{DK\ell_0 }{\sqrt{n}}\nonumber \\
			&\leq \wh L (\wh w_\lambda) + (\sqrt{r_\Sigma}+K)\frac{DGC \|\Sigma\|^{\frac{1}{2}}\nor{\wh w_\lambda}}{\sqrt{n}} +(\sqrt{r_\Sigma}+K)\frac{DGC \|\Sigma\|^{\frac{1}{2}}}{\sqrt n} +\frac{DK\ell_0 }{\sqrt{n}}\nonumber \\		
			&\leq \wh L (\wh w_\lambda) + \lambda \| \wh w_\lambda \|^2 + \frac{D^2G^2C^2(\sqrt{r_\Sigma}+K)^2 \|\Sigma\|}{4\la n} +(\sqrt{r_\Sigma}+K)\frac{DGC \|\Sigma\|^{\frac{1}{2}}}{\sqrt n} +\frac{DK\ell_0 }{\sqrt{n}} \nonumber\\
			&\leq \wh L (w_\lambda) + \lambda \| w_\lambda \|^2 + \frac{D^2G^2C^2(\sqrt{r_\Sigma}+K)^2 \|\Sigma\|}{4\la n} +(\sqrt{r_\Sigma}+K)\frac{DGC \|\Sigma\|^{\frac{1}{2}}}{\sqrt n} +\frac{DK\ell_0 }{\sqrt{n}}
			\label{eq:proof-wlambda-def-wlambda}
		\end{align}
		where~\eqref{eq:proof-wlambda-def-wlambda} holds by definition of
		$\wh w_\lambda$.  
Now, using again Lemma~\ref{lem:gen-gap-ball-gauss} for $\nor{w_\la}$ we have that, with
		probability $1 - \delta$:
		\begin{equation*}
	\wh L (w_\la)-L(w_\la) 
	<  \frac{D}{\sqrt{n}}\Big( G C  \|\Sigma\|^{\frac{1}{2}} \nor{w_\la}\big(  \sqrt{r_\Sigma}+\sqrt{\log (4/\delta)}
	\big)+
	\ell_0 \sqrt{\log (4/\delta)}\Big).
		\end{equation*}
		Combining this inequality with~\eqref{eq:proof-wlambda-def-wlambda}
		with a union bound, with probability $1 - 2 \delta$:
		\begin{align}
			\label{eq:proof-wlambda-general}
			L (\wh w_\lambda)
			< &L (w_\lambda) + \lambda \| w_\lambda \|^2 + \frac{D^2G^2C^2(\sqrt{r_\Sigma}+K)^2 \|\Sigma\|}{4\la n} +(\sqrt{r_\Sigma}+K)\frac{DGC \|\Sigma\|^{\frac{1}{2}}}{\sqrt n} +\frac{DK\ell_0 }{\sqrt{n}} +\nonumber\\
			&+ \frac{DG C  \|\Sigma\|^{\frac{1}{2}} \nor{w_\la}\big(  \sqrt{r_\Sigma}+\sqrt{\log (4/\delta)}
			\big)}{\sqrt{n}}+\frac{D\ell_0 \sqrt{\log (4/\delta)}}{\sqrt{n}}.
		\end{align}
		Since again $a b \leq \lambda a^2 + b^2 / (4 \lambda)$, then
		\begin{align*}
			\frac{DG C  \|\Sigma\|^{\frac{1}{2}} \nor{w_\la}\big(  \sqrt{r_\Sigma}+\sqrt{\log (1/\delta)}
				\big)}{\sqrt{n}}
			&\leq \lambda \| w_\lambda \|^2 + \frac{D^2G^2 C^2  \|\Sigma\| \big(  \sqrt{r_\Sigma}+\sqrt{\log (4/\delta)}\big)^2}{4\lambda n}\nonumber\\
			&\leq \mathcal{A} (\lambda) +\frac{D^2G^2 C^2  \|\Sigma\| \big(  \sqrt{r_\Sigma}+\sqrt{\log (4/\delta)}\big)^2}{4\lambda n}
		\end{align*}
		so that~\eqref{eq:proof-wlambda-general} implies, with probability
		$1 - 2\delta$:
		\begin{align*}
			L (\wh w_\lambda) - \inf_{w\in\H} L(w)
			< &2 \mathcal{A} (\lambda) + \frac{D^2G^2C^2 \|\Sigma\|((\sqrt{r_\Sigma}+K)^2+(  \sqrt{r_\Sigma}+\sqrt{\log (4/\delta)})^2)}{4\lambda n}
			+  \nonumber \\ 
			&+ \frac{DGC(\sqrt{r_\Sigma}+K) \|\Sigma\|^{\frac{1}{2}} +
			}{\sqrt{n}} + D\ell_0 (K+\sqrt{\log (4/\delta)})
		\end{align*}
		After replacing $\delta$ by $\delta/2$, we get
		bound~\eqref{eq:excessrisk-standard-general}.  
	\end{proof}
	
	\begin{proof}[Proof of Theorem~\ref{thm:regularized-full-space}]
		Assume that $w_* = \argmin_{w \in \H} L (w)$ exists.
		Then, by definition of $w_\lambda$,
		\[L (w_\lambda) + \lambda \|w_\lambda\|^2 \leq L (w_*) + \lambda \| w_* \|^2.\]
		In addition, $\| w_\lambda \| \leq \| w_* \|$, since otherwise having $\|w_*\| < \| w_\lambda\|$ and $L (w_*) \leq L (w_\lambda)$ would imply $L (w_*) + \lambda \|w_*\|^2 < L (w_\lambda) + \lambda \|w_\lambda\|^2$, contradicting the above inequality.
		Since $L (w_*) = \inf_\H L$, it follows from \eqref{eq:proof-wlambda-general} that, with probability $1 - 2\delta$,
		\begin{align}
			\label{eq:proof-mainbound-best}
			L (\wh w_\lambda)
			<& L(w_*) + \lambda \| w_* \|^2 + 		\frac{D^2G^2C^2(\sqrt{r_\Sigma}+K)^2 \|\Sigma\|}{4\la n} +\frac{DGC(\sqrt{r_\Sigma}+K) \|\Sigma\|^{\frac{1}{2}} +
				DK\ell_0 }{\sqrt{n}} +\nonumber\\
			&+\frac{DG C  \|\Sigma\|^{\frac{1}{2}} \nor{w_*}\big(  \sqrt{r_\Sigma}+\sqrt{\log (4/\delta)}
				\big)}{\sqrt{n}}+ \frac{D\ell_0 \sqrt{\log (4/\delta)}}{\sqrt{n}}
		\end{align}
		The bound~\eqref{eq:proof-mainbound-best} precisely corresponds to the desired bound~\eqref{eq:excessrisk-standard-best} after replacing $\delta$ by $\delta/2$.
		In particular, tuning $\lambda \asymp (DGCK\nor{\Sigma}^{1/2} / \| w_* \|) \sqrt{\log (1/\delta) / n}$ yields
		\begin{equation*}
			L (\wh w_\lambda) - L(w_*)
			\lesssim 
			\frac{\{DGC\nor{\Sigma}^{1/2} \| w_* \| \} \{ \log \log n + \sqrt{\log (1/\delta)} \}}{\sqrt{n}}
			.
		\end{equation*}
		Omitting the $\log \log n$ term, this bound essentially scales as $\wt O ( DGC\nor{\Sigma}^{1/2} \| w_* \| \sqrt{\log (1/\delta)/n} )$.
\end{proof}

\section{Proof of Section~\ref{sec: ERM on random subspace}}\label{proofthmbasic}
In order to prove Theorem~\ref{thm:1}, we need to previously extend 
 Lemma 7 in \citep{rudi2015less} to sub-gaussian random variables. 
\begin{lemma} \label{lem:id_min_proj_lev_subgauss}
	Fix $\delta>0$ and a $(T,\alpha_0)$-approximate leverage scores
	$(\hat{l}_i(\alpha))_{i=1}^n$ with confidence $\delta>0$. 
	Given $\alpha>\alpha_0$, let $\{\wt  x_1, \dots, \wt
	x_m\}$ be the Nystr\"om points selected according to
	Definition~\ref{def:approx_lev_scores} and set $\mathcal
	B_m=\mathrm{span}\{ \wt  x_1, \dots, \wt x_m\}$. Under
	Assumption~\ref{ass: sub-gaussian}, with probability at least $1-\delta$:
	\begin{equation}
		\Big\Vert (I-{\mathcal P}_{\BB_m})\Sigma^{1/2}\Big\Vert^2 \leq 	\Big\Vert (I-{\mathcal P}_{\BB_m}) (\Sigma+\alpha\operatorname{I})^{1/2}\Big\Vert^2\leq 3\alpha,
	\end{equation}
	provided that
	\begin{align}
		& n\gtrsim d_\alpha\vee \log(5/\delta)   \label{eq: subgauss n cond} \\
		& m\gtrsim d_\alpha\log(\frac{10n}{\delta})   \label{eq: subgauss m cond}.
	\end{align}
	Furthermore, if the spectrum of $\Sigma$ satisfies~the 
	decay conditions~\eqref{eq:2} (polynomial decay) or \eqref{eq:19} (exponential decay),  it is enough to assume  that
	\begin{align} 
		\label{eq:31a}
		&  n\gtrsim
		\log(5/\delta) && 	\alpha \gtrsim  n^{-1/p} && m\gtrsim \alpha^{-p}\log(\frac{10n}{\delta})  && \text{ polynomial decay}
		\\
		& n\gtrsim
		\log(5/\delta) && \alpha \gtrsim e^{-n}   && m\gtrsim \log(1/\alpha)\log(\frac{10n}{\delta})
		&&\text{ exponential decay} \label{eq:32a} 
	\end{align}
	
\end{lemma}

\begin{proof}
  Exploiting sub-gaussianity, the various terms are bounded differently.
  In particular, to bound $\beta_1$ we refer to Theorem 9 in \citep{koltchinskii2014concentration}, obtaining with probability at least $1-\delta$ 
  \begin{equation}
    \beta_1(\alpha)\lesssim \max \left\{ \sqrt{\frac{d_\alpha}{ n}}, \sqrt{\frac{\log(1/\delta)}{n}}\right\}.
  \end{equation}
  As regards $\beta_3$ term we apply Proposition \ref{prop: bound_emp_deff} below to get 	with probability greater than $1-3\delta$
  $$
  \beta_{3}(\alpha) \leq \frac{2 \log \frac{2 n}{\delta}}{3 m}+\sqrt{\frac{32 T^{2} d_\alpha \log \frac{2 n}{\delta}}{m}}
  $$
  for $n\geq 2 C^2 \log(1/\delta) $.
  
  Finally, taking a union bound we have with probability at least $1-5\delta$
  \begin{align*}
    \beta(\alpha) \lesssim &\max \left\{ \sqrt{\frac{d_\alpha}{ n}}, \sqrt{\frac{\log(\frac{1}{\delta})}{n}}\right\}+
    \\
                           &\;\;\;\;\;\qquad+\left(1+\max \left\{ \sqrt{\frac{d_\alpha}{ n}}, \sqrt{\frac{\log(\frac{1}{\delta})}{n}}\right\}\right)\left(\frac{2 \log \frac{2 n}{\delta}}{3 m}+\sqrt{\frac{32 T^{2} d_\alpha \log \frac{2 n}{\delta}}{m}}\right)\lesssim 1
  \end{align*}
  when $n\gtrsim d_\alpha\vee \log(1/\delta)$ and $m\gtrsim d_\alpha\log \frac{2 n}{\delta}$. See \citep{rudi2015less} to conclude the proof  of  the first claim.
  Assume now~\eqref{eq:2}  or \eqref{eq:19} . The second claim is consequence of Proposition~\ref{prop: eig polynom decay} or Proposition~\ref{prop:Exponential eigenvalues decay}. 
\end{proof}

We can proceed now with the proof of Theorem~\ref{thm:1}:
\begin{proof}[Proof of Theorem~\ref{thm:1}]
  We recall the notation.
  \begin{align*}
    \BB_m=\mathrm{span} \{\tilde{x}_1,\dots, \tilde{x}_m\}, \qquad \wh{\beta}_{\lambda} =\argmin_{w\in\BB_m} \wh{L}(w),\qquad w_* =   \argmin_{w\in\X} L(w)
  \end{align*}
  and ${\mathcal P}_m=\mathcal{P}_{\BB_m}$ the orthogonal projector
  operator onto $\BB_m$.
  
  In order to bound the excess risk of $\wh{\beta}_{\lambda}$, we
  decompose the error as follows:
  \begin{align}\label{1st_split_decomp}
    L(\wh{\beta}_{\lambda})-L(w_*) &\leq 
    	\left|L(\wh{\beta}_{\lambda})- \wh{L}(\wh{\beta}_{\lambda})-\lambda\Vert \wh{\beta}_{\lambda} \Vert^2_\H\right|
    +
    \left|\wh{L}(\wh{\beta}_{\lambda})+\lambda\Vert \wh{\beta}_{\lambda} \Vert^2_\H-\wh{L}(\mathcal{P}_m w_*)-\lambda\Vert \mathcal{P}_m w_* \Vert^2_\H\right|
+\nonumber\\
                                   &\hspace{0.5cm}+
                                   \left|\wh{L}(\mathcal{P}_m w_*)-L(\mathcal{P}_m w_*)\right|
                                     +
                                     \left|L(\mathcal{P}_m w_*)-L(w_*)\right|
                                 +\lambda\Vert \mathcal{P}_m w_*\Vert^2_\H
  \end{align}
\vspace{0.5cm} 
To bound the first term $\left|L(\wh{\beta}_{\lambda})-
  \wh{L}(\wh{\beta}_{\lambda})-\lambda\Vert \wh{\beta}_{\lambda}
  \Vert^2_\H\right|$ we apply Lemma~\ref{lem:gen-gap-union} for 
  $\wh{\beta}_{\lambda}$ and we get
\begin{equation*}
L (\wh{\beta}_{\lambda}) - \wh L (\wh{\beta}_{\lambda}) \leq  \frac{DGC(\sqrt{r_\Sigma}+K) \|\Sigma\|^{\frac{1}{2}}(1+\Vert \wh{\beta}_{\lambda}\Vert)}{\sqrt n}+\frac{DK\ell_0}{\sqrt{n}}
\end{equation*}
with $K = K_{\lambda, \delta} =  \sqrt{\log (1 + \log_2 ( 3 +{\ell_0/\lambda} ) ) + \log (1/\delta)}$ as in \eqref{eq:lemma3 simple}.

 Now since $xy\leq \lambda x^2+y^2/(4\lambda)$, we can write
  \begin{equation}
	\frac{DGC(\sqrt{r_\Sigma}+K) \Vert \wh{\beta}_{\lambda}\Vert  \|\Sigma\|^{\frac{1}{2}}}{\sqrt{n}}\leq \lambda\Vert \wh{\beta}_{\lambda}\Vert^2 + \frac{D^2G^2C^2(\sqrt{r_\Sigma}+K)^2 \|\Sigma\|}{\lambda n}
\end{equation}
  hence, 
  \begin{align}
	\label{bound_A_union_bound}
	\left|L(\wh{\beta}_{\lambda})-\wh{L}(\wh{\beta}_{\lambda})-\lambda\Vert \wh{\beta}_{\lambda}\Vert^2\right| \leq \frac{D^2G^2C^2(\sqrt{r_\Sigma}+K)^2 \|\Sigma\|}{\lambda n}+ \frac{DGC(\sqrt{r_\Sigma}+K) \|\Sigma\|^{\frac{1}{2}}}{\sqrt n}+\frac{DK\ell_0 }{\sqrt{n}},
\end{align}
Term $ \left|\wh{L}(\wh{\beta}_{\lambda})+\lambda\Vert \wh{\beta}_{\lambda} \Vert^2_\H-\wh{L}(\mathcal{P}_m w_*)-\lambda\Vert \mathcal{P}_m w_* \Vert^2_\H\right|$ is less or equal than 0.

As regards term $ \left|\wh{L}(\mathcal{P}_m w_*)-L(\mathcal{P}_m w_*)\right|$, since $\mathcal{P}_m$ is a
projection  $\|\mathcal{P}_m w_*\| \leq \|w_*\|$, so that with
probability at least $1 - \delta$:  

  \begin{align}
	\left|\wh{L}(\mathcal{P}_m w_*)-L(\mathcal{P}_m w_*)\right|&\leq
	\sup_{\|w\|\leq \|w_*\| } \left(\left|\wh{L}(w)-L(w)\right| \right)\nonumber \\
	&<
	\frac{D}{\sqrt{n}}\Big( GC \|w_*\|  \|\Sigma\|^{\frac{1}{2}}\big(  \sqrt{r_\Sigma}+\sqrt{\log (4/\delta)}
	\big)+\ell_0 \sqrt{\log (4/\delta)}\Big)
	.\label{eq:6}
\end{align}
where in the $\sup$ in  the left hand side is taken over all possible
Nystr\"om points and the second inequality is the content of
Lemma~\ref{lem:gen-gap-ball-gauss} where the role of $L$ and $\wh{L}$ is
interchanged. 

Finally, term $\left|L(\mathcal{P}_m w_*)-L(w_*)\right|$ can be rewritten as
 \begin{align}
    \label{reg C term}
    \left|L(\mathcal{P}_m w_*)-L(w_*)\right|
                   & \leq G  \int  | \scal{w}{\mathcal{P}_m w_*} -
                     \scal{w}{w_*}| dP_X(w) \nonumber\\
                   & \leq G  \left(\int  | \scal{w}{(I-\mathcal{P}_m
                     )w_*}|^2 dP_X(w)\right)^{\frac{1}{2}} \nonumber\\
               & = G\scal{\Sigma
                 (I-\mathcal{P}_m)w_*}{(I-\mathcal{P}_m)w_*}^{\frac{1}{2}} \\
              &= G\Vert \Sigma^{1/2}(I-\mathcal{P}_m)w_*\Vert_{\H}\nonumber\\
              &\leq G\Vert \Sigma^{1/2}(I-\mathcal{P}_m)\Vert \Vert
                w_*\Vert_{\H}\nonumber\\
         & = G\Vert (I-\mathcal{P}_m) \Sigma^{1/2}\Vert \Vert
                w_*\Vert_{\H} \leq G\sqrt{3\alpha} \|w_*\|,
  \end{align}
where the last bound is a consequence of Lemma~\ref{lem:id_min_proj_lev_subgauss}
and it holds true with probability at least $1-\delta$.

Putting the pieces together 
  we finally get the result in Theorem \ref{thm:1} by replacing $\delta$ with $\delta/3$. 
\end{proof}

\begin{proof}[Proof of Theorem.~\ref{prop: constr pb}]
 Under polynomial decay assumption~\eqref{eq:2}, the claim is a
 consequence of Theorem~\ref{thm:1} with Proposition~\ref{prop: eig
   polynom decay}  with $\beta=1/p$ so that
 \begin{equation}
	m\gtrsim  d_\alpha \log n,\hspace{1cm} d_\alpha
	\lesssim\alpha^{-p},\hspace{1cm} m\asymp n^p (\log n)^{1-p}
\end{equation}
Under exponential  decay assumption~\eqref{eq:19}, the claim is a consequence of Theorem~\ref{thm:1} with Proposition~\ref{prop: eig exp decay}  so that
\begin{equation}
	m\gtrsim  d_\alpha \log n,\hspace{1cm} d_\alpha\lesssim\log(1/\alpha),\hspace{1cm} m\asymp \log^2n
\end{equation}
\end{proof}

\begin{proof}[Proof of Theorem~\ref{firstmain}]
  The proof  is given by decomposing the
excess risk as in~\eqref{1st_split_decomp} where $\mathcal{P}_m$ is
replaced by $\mathcal P_{\BB}$,~\eqref{bound_A_union_bound} bounds
term~A,~\eqref{eq:6} bounds term~B and~\eqref{reg C term} and~\ref{proj} bound term~C.
\end{proof}

\section{Proofs of Section~\ref{sec:statistical}}\label{app:theorem 4} 

The following proposition provides a bound on the empirical effective
dimension $d_\alpha(\wh\Sigma)=\tr (
                \wh{\Sigma}_\alpha^{-1} \wh \Sigma )$
in terms of  the correspondent population quantity
$d_\alpha= \tr ( (\Sigma_\alpha+\alpha \operatorname{I})^{-1} \Sigma )$. 
	\begin{proposition}
		\label{prop: bound_emp_deff}
Let $X,X_1,\dots,X_n$ be iid $C$-sub-gaussian random
                variables in $\H$. For any $\delta>0$ and $n\geq 2 C^2 \ln(1/\delta) $, then the following hold with probability $1-\delta$
	\begin{equation}
	d_\alpha(\wh\Sigma)\leq 16 d_\alpha
\end{equation}
	\end{proposition}
	\begin{proof}
          Let $V_\alpha$ be the space spanned by eigenvectors of $\Sigma$ with corresponding eigenvalues $\alpha_j \geq \alpha$, and call $D_\alpha$ its dimension. Notice that  ${D_\alpha\leq 2d_\alpha}$ since $d_\alpha=\tr ( (\Sigma_\alpha+\alpha \operatorname{I})^{-1} \Sigma )=\sum \frac{\alpha_i}{\alpha_i+\alpha}$, where in the sum we have $D_\alpha$ terms greater or equal than $1/2$.
          
		Let $X = X_1 + X_2$, where $X_1$ is the orthogonal projection of $X$ on the space $V_\alpha$, we have
		\begin{equation}
			\wh \Sigma = \wh \Sigma_1 + \wh \Sigma_2 + \frac 1 n \sum_{i=1}^n (X_{1,i} X_{2,i}^\top + X_{2,i} X_{1,i}^\top ) \mleq 2 (\wh \Sigma_1 + \wh \Sigma_2)
		\end{equation}
		Now, since the function $g: t \mapsto \frac{t}{t + \alpha}$ is sub-additive (meaning that $g (t + t') \leq g(t) + g(t')$), denoting $d_\alpha (\Sigma) = \tr \, g (\Sigma) = \tr ((\Sigma_\alpha+\alpha \operatorname{I})^{-1} \Sigma)$,
		\begin{equation}
			d_\alpha (\wh \Sigma)
			\leq 2 ( d_\alpha (\wh \Sigma_1) + d_\alpha (\wh \Sigma_2) )
		\end{equation}
		and, since $(\wh \Sigma_1 + \alpha)^{-1} \wh \Sigma_1 \mleq I_{V_\alpha}$,
		\begin{align}
			\label{upper_bound_emp_d}
			\tr ( (\wh \Sigma_\alpha+\alpha \operatorname{I})^{-1} \wh \Sigma )
			&\leq 2 
			D_\alpha + \frac{2\tr (\wh \Sigma_2)}{\alpha}= 4 d_\alpha + \frac{2\tr (\wh \Sigma_2)}{\alpha}
		\end{align}
		Now,
		\begin{equation*}
			\tr (\wh \Sigma_2)
			= \frac{1}{n} \sum_{i=1}^n \| X_{2,i} \|^2
		\end{equation*}
		It thus suffices establish concentration for averages of the random variable $\| X_{2} \|^2$.
                
		Since $X$ is sub-gaussian then $\|X_2\|^2$ is sub-exponential. In fact, since $X$ is $C$-sub-gaussian then
		\begin{equation}
			\| \langle v, X\rangle\|_{\psi_2}\leq C \| \langle v, X\rangle\|_{L_2} \quad\quad \forall v\in\H
		\end{equation}
		and given that  $\langle v, \mathcal{P} X\rangle=\langle \mathcal{P} v, X\rangle$ with $\mathcal{P}$ an orthogonal projection, then also $X_2$ is $C$-sub-gaussian. Now take $e_i$ the orthonormal basis of $V$ composed by the eigenvectors of $\Sigma_2=\E[X_2X_2^T]$, then
		\begin{align}
			\left\| \|X_2\|^2\right\|_{\psi_1}&=\Big\| \sum_i \langle X_2, e_i\rangle^2\Big\|_{\psi_1} \leq \sum_i \left\| \langle X_2, e_i\rangle^2\right\|_{\psi_1}\\
			&= \sum_i \left\| \langle X_2, e_i\rangle\right\|_{\psi_2}^2\leq C^2 \left\| \langle X_2, e_i\rangle\right\|_{L_2}^2\\
			&=C^2\sum_i \alpha_i=C^2\tr\left[\Sigma_2\right]=C^2 \E\left[\|X_2\|^2\right]
		\end{align}
		so $\|X_2\|^2$ is $C^2 \E\left[\|X_2\|^2\right]$-sub-exponential. Note that $\E \| X_2 \|^2 = \E [ \tr (X_2 X_2^\top) ]
		= \tr ( \Sigma_2 ) \leq 2\alpha d_\alpha (\Sigma)$, in fact 
		\begin{align}
			d_\alpha&=\sum_{i=1}^\infty \frac{\alpha_i}{\alpha_i+\alpha}
			\geq \sum_{i: \alpha_i< \alpha} \frac{\alpha_i}{\alpha_i+\alpha}
			\geq \sum_{i: \alpha_i< \alpha}\frac{\alpha_i}{2\alpha}
			=\frac{\tr(\Sigma_2)}{2\alpha}
		\end{align}                
		Hence, we can apply then Bernstein inequality for sub-exponential scalar variables (see Theorem 2.10 in \citep{boucheron2013concentration}), with parameters $\nu$ and $c$ given by
		\begin{align}
			&n\E\left[\|X_2\|^4\right]\leq \underbrace{4nC^2\alpha^2 d_\alpha^2(\Sigma)}_\nu\\
			&c=C\alpha d_\alpha
		\end{align}
		where we used the bound on the moments of a sub-exponential variable (see \citep{vershynin2018high}).
		With high probability (\ref{upper_bound_emp_d}) becomes
		\begin{equation}
			d_\alpha(\wh\Sigma)
			\leq 
			8d_\alpha+ \frac{4C d_\alpha\sqrt{2\ln(1/\delta)}}{\sqrt{n}}+\frac{2C d_\alpha \ln(1/\delta)}{n }\leq 16 d_\alpha
		\end{equation}
		for $n\geq 2 C^2 \log(1/\delta) $.
              \end{proof}
          
          From \citep{adamczak2008tail} Theorem 4 we write a concentration inequality we will use in the following, corresponding to the simplified Talagrand's inequality in Theorem 7.5 of \citep{steinwart2008support} but for sub-exponential random variables:
          \begin{theorem}[Theorem 4 in \citep{adamczak2008tail}]
          	\label{thm: Talagrand subexp}
          	Let $X, X_{1}, \ldots, X_{n}$ be i.i.d. random variables with values in a measurable space $(\mathcal{S}, \mathcal{B})$ and let $\mathcal{F}$ be a countable class of measurable functions $f: \mathcal{S} \rightarrow \mathbb{R}$. Assume that $\mathbb{E} f\left(X\right)=0$ and $\left\|\sup _{f}\left|f\left(X\right)\right|\right\|_{\psi_1}<\infty$ for every $f \in \mathcal{F}$. Let
          	$$
          	Z=\sup _{f \in \mathcal{F}}\left|\frac 1 n \sum_{i=1}^{n} f\left(X_{i}\right)\right|
          	$$
          	and define
          	$$
          	\sigma^{2}=\sup _{f \in \mathcal{F}} \mathbb{E} f\left(X\right)^{2}.
          	$$
          	Then, for all $\tau>0$ and $\eta>0$, we have
          	\begin{align}
          		\mathbb{P}\left(Z\geq(1+\eta) \mathbb{E} Z+\frac{K_1 \left\| \sup _{f \in \mathcal{F}}\left|f\left(X\right)\right|\right\|_{\psi_1} (2+\tau)}{n}+ \sqrt {\frac{3(1+\tau)\sigma^2}{n}}\right)\leq e^{-\tau}
          	\end{align}
          where $K_1=K_1(\delta, \eta)$.
          \end{theorem}
          
          Similarly to \citep{steinwart2008support}, we define the quantity 
          \begin{equation}
          	\label{def: g_w,r}
          	g_{w, r}:=\frac{h_{w}-\E  h_{w}}{\la\nor{w}^2+\E  h_{w}+r}, \quad w \in \X, \quad r>0
          \end{equation}
          (Note that in \citep{steinwart2008support} they define $-g_{w,r}$).
          Our plan is to apply Theorem \ref{thm: Talagrand subexp} to $g_{\wh w_0,r}$, with $\wh w_{0}\in\BB_m\subseteq \X$ and $\nor{\wh w_0}\leq \nor{w_*}$.
          \begin{corollary}
          	\label{cor: Talagrand}
          	Under the hypothesis of Theorem \ref{thm: Talagrand subexp}, for all $\tau>0$ we have
          	\begin{align}
          		\sup _{w \in \X, \nor{w}\leq\nor{w_*}} \frac{\wh \E   h_{w}-\E  h_{w}}{
          			\la\nor{w}^2+\E  h_{w}+r}<&2 \mathbb{E}_{D \sim \mathrm{P}^{n}} \sup _{w \in \X, \nor{w}\leq\nor{w_*}} \frac{\wh \E   h_{w}-\E  h_{w}}{\la\nor{w}^2+\E  h_{w}+r}\nonumber\\
          		&+\sqrt{\frac{3V (1+\tau)}{n r^{2-\theta}}}+2GK_1 \nor{w_*}\frac{(C\sqrt{2\tr \Sigma}+\E \nor{X}) (2+\tau)}{nr}
          	\end{align}
          	\begin{proof}
          		In Theorem \ref{thm: Talagrand subexp}, we take
          		\begin{equation}
          			Z=\sup_{w\in\H,\nor{w}\leq R}\left|\frac 1 n \sum_{i=1}^{n} g_{w,r}\left(X_{i}\right)\right|.
          		\end{equation}
          		We have also that, using the second inequality of Lemma $7.1$ in \citep{steinwart2008support} and taking $\theta>0$, $q:=\frac{2}{2-\theta}$, $q^{\prime}:=\frac{2}{\theta}$, $a:=r$, and $b:=\E  h_{w} \neq 0$:
          		
          		$$
          		\E  g_{w, r}^{2} \leq \frac{\E  h_{w}^{2}}{\left(\la\nor{w}^2+\E  h_{w}+r\right)^{2}} \leq \frac{(2-\theta)^{2-\theta} \theta^{\theta} \E  h_{w}^{2}}{4 r^{2-\theta}\left(\E  h_{w}\right)^{\theta}} \leq V r^{\theta-2}=\sigma^2
          		$$
          		Moreover, 
          		\begin{align*}
          			&\left\|\sup_{w \in \X, \nor{w}\leq \nor{w_*}}\left|g_{w,r}\left(X\right)\right|\right\|_{\psi_1}=\left\| \sup _{w \in \X, \nor{w}\leq \nor{w_*}}\left|\frac{h_{w}\left(X\right)-\E  h_{w}}{\la\nor{w}^2+\E  h_{w}+r}\right|\right\|_{\psi_1} \\
          			&\leq \frac{1}{r}\left\|\sup _{w \in \X, \nor{w}\leq \nor{w_*}}\left|h_{w}-\E  h_{w}\left(X\right)\right|\right\|_{\psi_1}\\
          			&=\frac{1}{r}\left\|\sup _{w \in \X, \nor{w}\leq \nor{w_*}}\left|\ell(\scal{w}{X},Y)-\ell(\scal{w_*}{X},Y)-\E [ \ell(\scal{w}{X},Y)-\ell(\scal{w_*}{X},Y)]\right|\right\|_{\psi_1}\\
          			&\leq\frac{1}{r}\left\|\sup _{w \in \X, \nor{w}\leq \nor{w_*}}\left|\ell(\scal{w}{X},Y)-\ell(\scal{w_*}{X},Y)\right|+\sup _{w \in \X, \nor{w}\leq \nor{w_*}}\left|\E [ \ell(\scal{w}{X},Y)-\ell(\scal{w_*}{X},Y)]\right|\right\|_{\psi_1}\\
          			&\leq\frac{1}{r}\left\| G\sup _{w \in \X, \nor{w}\leq \nor{w_*}}\left|\scal{w-w_*}{X}\right|+G\sup _{w \in \X, \nor{w}\leq \nor{w_*}}\E \left| \scal{w-w_*}{X}\right|\right\|_{\psi_1}\\
          			&\leq\frac{1}{r}\Big\| 2G\nor{w_*}\nor{X}+2G\nor{w_*}\E \nor{X}\Big\|_{\psi_1}=\frac{2G\nor{w_*}}{r}\Big\|\nor{X}+\E \nor{X}\Big\|_{\psi_1}\\
          			&\leq\frac{2G\nor{w_*}}{r}\Big\|\nor{X}+\E \nor{X}\Big\|_{\psi_2}\leq \frac{2G\nor{w_*}(C\sqrt{2\tr \Sigma}+\E \nor{X})}{r}
          		\end{align*}
          	where last inequality derives from the fact that $\|X\|$ is sub-gaussian since, given an orthonormal basis $e_i$, 
          	$$
          	\begin{aligned}
          		\big\| \left\|X\right\| \big\|^2_{\psi_2}&\leq
          		\big\|\left\| X\right\|^{2}\big\|_{\psi_{1}} =\Big\|\sum_{i}\left\langle X, e_{i}\right\rangle^{2}\Big\|_{\psi_{1}} \leqslant \sum_{i}\left\|\left\langle X, e_{i}\right\rangle^{2}\right\|_{\psi_{1}} \\
          		&\leq 2\sum_{i}\left\|\left\langle X, e_{i}\right\rangle\right\|_{\psi_{2}}^{2} \leqslant 2C^{2}\left\|\left\langle X, e_{i}\right\rangle\right\|_{L_{2}}^{2}=2C^{2} \operatorname{Tr}\left[\Sigma\right]
          	\end{aligned}
          	$$
          		Applying Theorem \ref{thm: Talagrand subexp} with $\eta=1$ we get the result.
          	\end{proof}
          \end{corollary}

We now adapt Theorem 7.23 in \citep{steinwart2008support} to our setting: 
\begin{theorem}\label{thm: adaptation thm 7.23}
	Under assumptions~\ref{ass: sub-gaussian},~\ref{ass:loss},~\ref{ass:clipped} and~\ref{ass:best}, the covariance matrix  satisfies the polynomial decay condition \eqref{eq:2},  and the Bernstein conditions~\eqref{supremum bound}--\eqref{variance bound} hold true.  Fix a closed subspace $\widehat{\mathcal{F}}$ of $\X$ and set
	\begin{equation}
		\label{eq:34}
		w_{\widehat{\mathcal{F}},\la} =\operatornamewithlimits{argmin}_{w\in \widehat{\mathcal{F}}}\left( \wh L(w) +\la \|w\|^2 \right) \qquad \la>0.
	\end{equation}
	Choose $\wh w_0\in\widehat{\mathcal{F}}$, fix $\delta >0$, then with probability at least $1-\delta$
	\begin{align}\label{fast rate thm 7.23 application}
		\lambda\|\wh w_{\mathcal F,\la}\|^{2}+&L (\wh w_{\mathcal F,\la}^{cl})-L(f_*) \leq  7\left(\lambda\left\|\wh w_{0}\right\|^{2}+L(\wh w_{0})-L(f_*)\right)+K_3\left(\frac{a^{2 p}}{\lambda^{p} n}\right)^{\frac{1}{2-p-\vartheta+\vartheta p}}+\nonumber\\ 
		&\quad+2\left(\frac{72 V \log(3/\delta)}{n}\right)^{\frac{1}{2-\vartheta}}+16GK_1\nor{w_*}\frac{(C\sqrt{2\tr \Sigma}+\E \nor{X}) (2+\log(3/\delta))}{n} 
	\end{align}
	where the constant $a$ only depends on~\eqref{eq:2} and $K_3\geq 1$  only depends on $p, M, B, \vartheta,$ and $V$.
\end{theorem} 
\begin{proof}
  The proof mimics the one of Theorem 7.23 \citep{steinwart2008support}, with some major differences.
  
  We start recalling that Theorem 15 in \citep{steinwart2009optimal} shows that that the decay condition~\eqref{eq:2}  is equivalent to condition~(7.48) of  Theorem 7.23, which is given in terms of entropy numbers $e_j$, see Lemma~\ref{entropy}. Note that the constant $a$ is defined by the bound~(7.48). Using this remark, the above assumptions let us upper bound the empirical Rademacher complexity of $\X_r$ in term of a function $\varphi_n(r)$ defined as in \citep{steinwart2008support} (see pag. 267). Thus, the result comes from the application of Steinwart's Theorem 7.20,  with the key difference that our $X$ is not bounded but sub-gaussian and that $\wh w_0$ here is not deterministic but depends on the data.

  As a consequence, in order to control the quantity $\wh\E h_{\wh w_0}-\E h_{\wh w_0}$ we cannot simply apply a Bernstein's inequality for sub-gaussian but we need to use the more refined Corollary \ref{cor: Talagrand}. In particular, we mimic the reasoning to derive \citep[eq. (7.44)]{steinwart2008support}, but where Talagrand's inequality for bounded random variables is replaced by our Theorem \ref{thm: Talagrand subexp} for sub-exponential ones and in the specific case of Corollary \ref{cor: Talagrand}.
  
  We split the error as in \citep[eq. (7.39)]{steinwart2008support},
	\begin{equation}
		\label{eq: split 7.20}
		\la\nor{\wh w_\la}^2+\E h_{\wh w_\la^{cl}}\leq (\la\nor{\wh w_0}^2+\E h_{\wh w_0})+(\wh \E   h_{\wh w_0}-\E  h_{\wh w_0})+(\E h_{\wh w_\la^{cl}}-\wh \E  h_{\wh w_\la^{cl}})
	\end{equation}
	and we start with controlling the term $\wh \E   h_{\wh w_0}-\E  h_{\wh w_0}$.
        
	Exploiting the definition of $g_{w,r}$ in \eqref{def: g_w,r}, we know that for all the $w\in \X$ with $\nor{w}\leq \nor{w_*}$ and $r>0$ we can apply Corollary \ref{cor: Talagrand}. In particular, since $\wh w_0\in \BB_m\subseteq\X$, the bound in the Corollary is valid also for $\wh w_0$, i.e 
	\begin{align}
		\frac{\wh \E   h_{\wh w_0}-\E  h_{\wh w_0}}{
			\la\nor{\wh w_0}^2+\E  h_{\wh w_0}+r}<&2 \mathbb{E}_{D \sim \mathrm{P}^{n}} \frac{\wh \E   h_{\wh w_0}-\E  h_{\wh w_0}}{\la\nor{\wh w_0}^2+\E  h_{\wh w_0}+r}\nonumber\\
		&+\sqrt{\frac{3V (1+\tau)}{n r^{2-\theta}}}+2GK_1 \nor{w_*}\frac{(C\sqrt{2\tr \Sigma}+\E \nor{X}) (2+\tau)}{nr}.
	\end{align}
	Using {symmetrization} (see Prop. 7.10 in \citep{steinwart2008support}) we have 
	\begin{align}
		\mathbb{E}_{D \sim \mathrm{P}^{n}} \sup_{w\in \BB_{m,r}, \nor{w}\leq\nor{w_*}} \left|\wh \E   h_{w}-\E  h_{w}\right|&\leq \mathbb{E}_{D \sim \mathrm{P}^{n}} \sup_{w\in \X_r, \nor{w}\leq\nor{w_*}} \left|\wh \E   h_{w}-\E  h_{w}\right|\nonumber\\
		&\leq 2 \mathbb{E}_{D \sim \mathrm{P}^{n}} \wh \rad(\X_r,n)\leq 2\varphi_{n}(r).
	\end{align}
	Peeling by Steinwart's Theorem 7.7 together with $\X_r=\{w\in \X: \la\nor{w}^2+\E h_w\leq r\}$ hence gives 
	\begin{equation}
		\mathbb{E}_{D \sim \mathrm{P}^{n}} \sup_{w\in \BB_{m}, \nor{w}\leq\nor{w_*}} \left|\wh \E   g_{w,r}\right|
		\leq\mathbb{E}_{D \sim \mathrm{P}^{n}} \sup_{w\in \X, \nor{w}\leq\nor{w_*}} \left|\wh \E   g_{w,r}\right|\leq \frac{8\varphi_n(r)}{r}
	\end{equation}
	Putting all together we get w.h.p. 
	\begin{align}
		\wh \E   h_{\wh w_0}-\E  h_{\wh w_0}&<(\la\nor{\wh w_0}^2+\E  h_{\wh w_0})\left(\frac{10 \varphi_{n}(r)}{r}+\sqrt{\frac{3V (1+\tau)}{n r^{2-\theta}}}+2GK_1\nor{w_*}\frac{(C\sqrt{2\tr \Sigma}+\E \nor{X}) (2+\tau)}{nr}\right)\nonumber \\
		& \qquad +10 \varphi_{n}(r)+\sqrt{\frac{3V (1+\tau)r^\theta}{n}}+2GK_1\nor{w_*}\frac{(C\sqrt{2\tr \Sigma}+\E \nor{X}) (2+\tau)}{n}
	\end{align}
	As regards the term $\E h_{w_\la^{cl}}-\wh \E  h_{w_\la^{cl}}$ we proceed as in \citep{steinwart2008support}.
	We finally obtain, for $\wh w_0\in \BB_m$ with $\nor{\wh w_0}\leq \nor{w_*}$ and with $r\geq r^*_{\BB_m}\geq r^*_\X$, w.h.p.
	\begin{align}
		\la &\nor{\wh w_\la}^2+\E  h_{\wh w_\la^{cl}}< \left(\la\nor{\wh w_0}^2+\E  h_{\wh w_0}\right)+\nonumber\\
		&+(\la\nor{\wh w_0}^2+\E  h_{\wh w_0})\left(\frac{10 \varphi_{n}(r)}{r}+\sqrt{\frac{3V (1+\tau)}{n r^{2-\theta}}}+2GK_1\nor{w_*}\frac{(C\sqrt{2\tr \Sigma}+\E \nor{X}) (2+\tau)}{nr}\right)+\nonumber\\
		& +10 \varphi_{n}(r)+\sqrt{\frac{3V (1+\tau)r^\theta}{n}}+2GK_1\nor{w_*}\frac{(C\sqrt{2\tr \Sigma}+\E \nor{X}) (2+\tau)}{n}+\nonumber\\
		&+\left(\la\nor{\wh w_\la}^2+\E  h_{\wh w_\la^{cl}}\right)\left(\frac{10 \varphi_{n}(r)}{r}+\sqrt{\frac{2 V \tau}{n r^{2-\theta}}}+\frac{28 B \tau}{3 n r}\right)\nonumber \\
		&+10 \varphi_{n}(r)+\sqrt{\frac{2 V \tau r^{\theta}}{n}}+\frac{28 B \tau}{3 n}
	\end{align}
	which replaces (7.44) in \citep{steinwart2008support}.
        
	Observe now that $r\geq 30\varphi_{n}(r)$ implies $10\varphi_{n}(r)r^{-1}\leq 1/3$ and $10\varphi_{n}(r)\leq r/3$. Moreover, $r \geq\left(\frac{72 V (1+\tau)}{n}\right)^{1 /(2-\theta)}$ yields
	$$
	\left(\frac{2 V \tau}{n r^{2-\theta}}\right)^{1 / 2} \leq \frac{1}{6} \quad \text { and } \quad\left(\frac{2 V \tau r^{\theta}}{n}\right)^{1 / 2} \leq \frac{r}{6}
	$$
	and
	$$
	\left(\frac{3 V (1+\tau)}{n r^{2-\theta}}\right)^{1 / 2} \leq \frac{1}{4} \quad \text { and } \quad\left(\frac{2 V (1+\tau) r^{\theta}}{n}\right)^{1 / 2} \leq \frac{r}{4}
	$$
	In addition $n\geq 72(1+\tau)$, $V\geq B^{2-\theta}$, and $r \geq\left(\frac{72 V (1+\tau)}{n}\right)^{1 /(2-\theta)}$ imply
	$$
	\frac{28 B \tau}{3 n r}=\frac{7}{54} \cdot \frac{72 \tau}{n} \cdot \frac{B}{r} \leq \frac{7}{54} \cdot\left(\frac{72 \tau}{n}\right)^{\frac{1}{2-\theta}} \cdot \frac{V^{\frac{1}{2-\theta}}}{r} \leq \frac{7}{54}
	$$
	and $\frac{28 B \tau}{3 n} \leq \frac{7 r}{54}$. Finally $r\geq 8GK_1\nor{w_*}\frac{(C\sqrt{2\tr \Sigma}+\E \nor{X}) (2+\tau)}{n}$ gives
	$$
	2GK_1\nor{w_*}\frac{(C\sqrt{2\tr \Sigma}+\E \nor{X}) (2+\tau)}{nr}\leq \frac{1}{4} $$and $$2GK_1\nor{w_*}\frac{(C\sqrt{2\tr \Sigma}+\E \nor{X}) (2+\tau)}{n}\leq \frac{r}{4}
	$$
	We finally obtain
	\begin{align}
		\la\nor{\wh w_\la}^2+\E  h_{\wh w_\la^{cl}}&< \frac{11}{6}\left(\la\nor{\wh w_0}^2+\E  h_{\wh w_0}\right)+\frac{79}{54}r+\epsilon+\frac{17}{27}\left(\la\nor{\wh w_\la}^2+\E  h_{\wh w_\la^{cl}}\right)\nonumber\\
		&\leq 5\left(\la\nor{\wh w_0}^2+\E  h_{\wh w_0}\right)+2r 
	\end{align}
	with 
	$$
	r>\max \left\{30 \varphi_{n}(r),\left(\frac{72 V \tau}{n}\right)^{\frac{1}{2-\vartheta}},8GK_1\nor{w_*}\frac{(C\sqrt{2\tr \Sigma}+\E \nor{X}) (2+\tau)}{n} , r_\X^{*}\right\}
	$$ 
\end{proof}

%

\begin{remark}
	\label{remark: w_la}
	Notice that the same reasoning can be applied in Section~\ref{sec:statistical} in the more general framework where $w_*$ does not exist. In that case $w_*$ will be replaced by $w_\la:=\argmin_{w\in\X} L(w)+\lambda\Vert
	w\Vert^2$, with $\nor{w_\la}\leq \sqrt{\frac{\A(\la)}{\la}}$.
\end{remark}\hspace{0.2cm}
\newline

We are now ready to prove our main result:
\begin{proof}[Proof of Theorem~\ref{thm: fast rate A(lambda)}, polynomial decay]
	Applying Theorem~\ref{thm: adaptation thm 7.23} in the general case of Remark~\ref{remark: w_la}, with the choice $\wh \F=\mathcal B_m$ and $\wh w_0=\mathcal{P}_{\BB_m}w_\lambda$,
	we rewrite (\ref{fast rate thm 7.23 application}) as: 
	\begin{align}
		\lambda\Vert \wh{\beta}_{\la,m} \Vert^2+&L(\wh{\beta}_{\la,m}^{cl})-L(f_*) \leq 7(\lambda\Vert \mathcal{P}_{\BB_m} w_\la\Vert^2 +L(\mathcal{P}_{\BB_m} w_\la)-L(f_*) )+K_3\Big(\frac{a^{2p}}{\lambda^p n}\Big)^{\frac{1}{2-p-\theta+\theta p}}+\nonumber\\
		&\;\;\;\;+2\Big(\frac{72V\log(3/\delta)}{n}\Big)^{\frac{1}{2-\theta}}
		+16GK_1\nor{w_\la}\frac{(C\sqrt{2\tr \Sigma}+\E \nor{X}) (2+\log(3/\delta))}{n} 
		\nonumber\\
		&\leq  7(\lambda\Vert \mathcal{P}_{\BB_m} w_\la\Vert^2 +L(\mathcal{P}_{\BB_m} w_\la)-L(w_\la)+L(w_\la)-L(f_*) )+K_3\Big(\frac{a^{2p}}{\lambda^pn}\Big)^{\frac{1}{2-p-\theta+\theta p}}+ \nonumber\\
		&\;\;\;\;+2\Big(\frac{72V\log(3/\delta)}{n}\Big)^{\frac{1}{2-\theta}}+16GK_1\frac{(C\sqrt{2\tr \Sigma}+\E \nor{X}) (2+\log(3/\delta))}{n} 
		\sqrt{\frac{\mathcal A(\la)}{\lambda}}\nonumber\\
		&\leq 7(L(\mathcal{P}_{\BB_m} w_\la)-L(w_\la)+\lambda\Vert w_\la\Vert^2+L(w_\la)-L(f_*) )+K_3\Big(\frac{a^{2p}}{\lambda^p n}\Big)^{\frac{1}{2-p-\theta+\theta p}}+ \nonumber\\
		&\;\;\;\;+2\Big(\frac{72V\log(3/\delta)}{n}\Big)^{\frac{1}{2-\theta}}+16GK_1\frac{(C\sqrt{2\tr \Sigma}+\E \nor{X}) (2+\log(3/\delta))}{n} 
		\sqrt{\frac{\mathcal A(\la)}{\lambda}}\nonumber\\
		&=7\mathcal{A}(\lambda) +7(L(\mathcal{P}_{\BB_m} w_\la)-L(w_\la))+K_3\Big(\frac{a^{2p}}{\lambda^p n}\Big)^{\frac{1}{2-p-\theta+\theta p}}+2\Big(\frac{72V\log(3/\delta)}{n}\Big)^{\frac{1}{2-\theta}}+\nonumber\\
		&\;\;\;\;+16GK_1\frac{(C\sqrt{2\tr \Sigma}+\E \nor{X}) (2+\log(3/\delta))}{n} \sqrt{\frac{\mathcal A(\la)}{\lambda}}
		\label{eq:39}
	\end{align}
	where we used the fact that $\left\|w_{\lambda}\right\|\leq \sqrt{\mathcal A(\la)/\lambda}$.\\
	We can deal with the term $L(\mathcal{P}_{\BB_m} w_\la)-L(w_\la)$ as
	in~\eqref{reg C term} (but where we use
	Lemma~\ref{lem:id_min_proj_lev_subgauss} instead of Lemma~7 in
	\citep{rudi2015less} to exploit sub-gaussianity), so that for $\alpha\gtrsim n^{-1/p}$ with probability greater than $1-\delta$
	\begin{equation}
		L(\mathcal{P}_{\BB_m} w_\la)-L(w_\la)\leq K_2 G\sqrt{\alpha} \nor{w_\la}
		\leq K_2 G\sqrt{\alpha}\sqrt{\frac{\mathcal{A}(\la)}{\la}}\label{eq:40}
	\end{equation}
	for some universal constant $K_2>0$.
	We finally obtain with probability greater than $1-2\delta$:
	%
	\begin{align}
		\la \Vert\wh{\beta}_{\lambda,m}
		\Vert^2_\H&+L(\wh{\beta}_{\lambda,m}^{cl})-L(f_*)
		\leq 7\mathcal A(\la) + 7K_2G \sqrt{\frac{\alpha
				\mathcal A(\la)}{\lambda}}+K_3\Big(\frac{a^{2p}}{\lambda^pn}\Big)^{\frac{1}{2-p-\theta+\theta p}}+\nonumber\\
		&\quad +2\Big(\frac{72V\log(3/\delta)}{n}\Big)^{\frac{1}{2-\theta}}+16GK_1\frac{(C\sqrt{2\tr \Sigma}+\E \nor{X}) (2+\log(3/\delta))}{n} 
		\sqrt{\frac{\mathcal A(\la)}{\lambda}} 
	\end{align}
	which proves the first claim. 
\end{proof}

\noindent The following corollary provides the optimal rates. 
\begin{corollary}\label{optimal_rates}
	Fix $\delta>0$. Under the Theorem~\ref{thm: fast rate A(lambda)}  and
	the source condition
	\[
	\mathcal A(\la) \leq A_0 \la^r 
	\]
	for some $r\in (0,1]$, set  
	\begin{align}
		\lambda & \asymp
		n^{-\min\{\frac{2}{r+1},\frac{1}{r(2-p-\theta+\theta
				p)+p}\}} \label{eq:3a}\\
		\alpha  & \asymp n^{-\min\{2,\frac{r+1}{r(2-p-\theta+\theta p)+p}\}}
		\label{eq:3b}\\
		m & \gtrsim  n^{\min\{2p,\frac{p(r+1)}{ r(2-p-\theta+\theta
				p)+p}\}} \label{eq:3c}
	\end{align}
	with probability at least $1-2\delta$:
	\begin{align}
		\lambda\Vert\wh{\beta}_{\lambda,m} \Vert^2+L(\wh{\beta}_{\lambda,m}^{cl})-L(f_*) &\lesssim n^{-\min\{\frac{2r}{r+1},\frac{r}{r(2-p-\theta+\theta p)+p}\}}
	\end{align}
\end{corollary}

\begin{proof}
	Lemma~\ref{lem:id_min_proj_lev_subgauss} with  Proposition~\ref{prop: eig polynom decay} gives
	\begin{equation}
		m\gtrsim d_\alpha \log (n/\delta),\hspace{1cm} d_\alpha
		\lesssim\alpha^{-p} \hspace{1cm} \alpha\asymp\frac{\log^{1/p} (n/\delta) }{m^{1/p}} 
	\end{equation}
	Lemma A.1.7 in \citep{steinwart2008support}  with   $r=2$,
	$1/\gamma=(2-p-\theta+\theta p)$, $\alpha=p$, $\beta=r$ shows that the
	choice of $\la$, $\alpha$ and $m$ given
	by~\eqref{eq:3a}--\eqref{eq:3c} provides the optimal rate. 
\end{proof}
\noindent Notice that $\alpha  \asymp n^{-\min\{2,\frac{r+1}{r(2-p-\theta+\theta p)+p}\}}$ is compatible with condition $\alpha\gtrsim d_\alpha\asymp n^{-1/p}$ in Lemma~\ref{lem:id_min_proj_lev_subgauss}.


When we are in the well-specified case, i.e. $w_*$ exists, we have the following results (see Section~\ref{subsec:poly}).
\begin{corollary}
  \label{thm: fast rate f_H}
Fix $\la>0$, $\alpha\gtrsim n^{-1/p}$ and $0<\delta<1$. 
 Under Assumptions~\ref{ass: sub-gaussian},~\ref{ass:loss},~\ref{target},~\ref{ass:berstein} (with $\theta=1$)  and polynomial decay condition \eqref{eq:2}, then, with probability at least $1-2\delta$:
  \begin{equation}
L(\widehat{\beta}_{\lambda,m}^{cl}) -L(w_*)\lesssim \frac{1}{\la^p n }
+ \la\nor{w_*}^2 +\sqrt{\alpha} \nor{w_*} 
\end{equation}
provided that $n$ and $m$ are large enough.

\begin{proof}
	The proof mimics the proof of Theorem~\ref{thm: fast rate A(lambda)} \textit{(a)} where in~\eqref{fast rate thm 7.23 application} we
	choose
	$\wh w_0=\mathcal{P}_{\BB_m} w_*$
	Hence~\eqref{fast rate thm 7.23 application}  with $\theta=1$ reads 
	\begin{align}
		\lambda\Vert
		\widehat{\beta}_{\lambda,m}\Vert^2+L(\widehat{\beta}_{\lambda,m}^{cl})-&L(w_*)
		\leq 7(\lambda\Vert \mathcal{P}_{\BB_m} w_*\Vert^2 +L(\mathcal{P}_{\BB_m}
		w_*)-L(w_*))+ K_3\frac{a^{2p}}{\lambda^p n}  +\nonumber\\
		&\;\;\;\;
		+144V \frac{\log(3/\delta)}{n}+16GK_1\nor{w_*}\frac{(C\sqrt{2\tr \Sigma}+\E \nor{X}) (2+\log(3/\delta))}{n} \nonumber\\
		&\leq 7\lambda\Vert w_*\Vert^2 +7(L(\mathcal{P}_{\BB_m} w_* )-L(w_*))+
		K_3\frac{a^{2p}}{\lambda^p n}+144 V \frac{\log(3/\delta)}{n}+ \nonumber\\
		&\;\;\;\;
		+16GK_1\nor{w_*}\frac{(C\sqrt{2\tr \Sigma}+\E \nor{X}) (2+\log(3/\delta))}{n} 
	\end{align}
	We can deal wit
	h the term $L(\mathcal{P}_{\BB_m} w_*)-L(w_*)$ as
	in~\eqref{reg C term}, so that for $\alpha\gtrsim n^{-1/p}$ with probability greater than $1-\delta$
	\[
	L(\mathcal{P}_{\BB_m} w_*)-L(w_*)\leq K_2G\sqrt{\alpha} \nor{w_*}
	\]
	for some $K_2>0$.
	Hence, with probability at least $1-2\delta$:
	\begin{align}
		\lambda\Vert \widehat{\beta}_{\lambda,m}\Vert^2+L(\widehat{\beta}_{\lambda,m}^{cl})-L(w_*)&\leq 7\lambda\Vert w_*\Vert^2 +7K_2G\sqrt{\alpha}\Vert w_*\Vert+K_3\frac{a^{2p}}{\lambda^p n}+144 V \frac{\log(3/\delta)}{n}+ \nonumber\\
		&\;\;\;\;
		16GK_1\nor{w_*}\frac{(C\sqrt{2\tr \Sigma}+\E \nor{X}) (2+\log(3/\delta))}{n} 
	\end{align}
	which proves the claim.  
\end{proof}
\end{corollary}

\noindent And, similarly to Corollary~\ref{optimal_rates}, we obtain the optimal rate presented in Eq.~\ref{rate2}.
\begin{corollary}
	\label{cor:optimal}
	Fix $\delta>0$. Under the assumptions of Theorem~\ref{thm: fast rate A(lambda)} \textit{(a),} when the variance
	bound~\eqref{variance bound} holds true  with the optimal paratemer
	$\theta=1$ and  the model is well-specified, i.e. $r=1$, set  
	\begin{align}
		\lambda & \asymp n^{-\frac{1}{1+ p}}\label{eq:4a}\\
		\alpha  & \asymp n^{-\frac{2}{1+p}}
		\label{eq:4b}\\
		m & \gtrsim  n^{\frac{2p}{1+p}}\log n \label{eq:4c}
	\end{align}
	then, for ALS sampling, with probability at least $1-2\delta$:
	\begin{align}      
		\lambda \Vert\widehat{\beta}_{\lambda,m}
		\Vert^2_\H+L(\widehat{\beta}_{\lambda,m}^{cl})-L(w_*)
		&\lesssim \Vert w_* \Vert\Big(\frac{1}{n}\Big)^{\frac{1}{1+p}}.
	\end{align}
\end{corollary}
\noindent Notice that $\alpha  \asymp n^{-\frac{2}{1+p}}$ is compatible with condition $\alpha\gtrsim d_\alpha\asymp n^{-1/p}$ in Lemma~\ref{lem:id_min_proj_lev_subgauss}.

\subsection{Excess risk under exponential decay}
As regards exponential decay, given the discussion in Appendix \ref{app: entropy}, we have a different bound on the empirical Rademacher complexity of $\X_r$. In particular, we obtain $\varphi_{n}(r):=C_{1} \sqrt{\frac{V}{n}}\log _{2}\left(\frac{1}{ \lambda}\right) \sqrt{r}+C_{2} \frac{\log^2_{2}(1 / \lambda)}{n}$ and we modify Theorem \ref{thm: adaptation thm 7.23} in the case of exponential decay using the following Lemma:
\begin{lemma} \label{lem: exp decay} When
	$$r=C_3  \frac{\ln_{2}^2(1 / \lambda)}{n}+\left(\frac{72 V \tau}{n}\right)^{\frac{1}{2-\vartheta}}+8GK_1\nor{w_*}\frac{(C\sqrt{2\tr \Sigma}+\E \nor{X}) (2+\tau)}{n}$$
	we have 
	$$r\geq \max\left\{30 \varphi_n(r),\left(\frac{72 V \tau}{n}\right)^{\frac{1}{2-\vartheta}}, 8GK_1\nor{w_*}\frac{(C\sqrt{2\tr \Sigma}+\E \nor{X}) (2+\tau)}{n} \right\}$$
\end{lemma}

We can finally prove the second part of Theorem~\ref{thm: fast rate A(lambda)} under exponential decay:

\begin{proof}[Proof of Theorem~\ref{thm: fast rate A(lambda)}, exponential decay]
		We follow exactly the proof of Theorem \ref{thm: adaptation thm 7.23} for polynomial decay presented above in the previous subsection, but using the estimate in Lemma~\ref{lem: exp decay} for $r$:
		\begin{align*}
			L(\wh \beta^{cl}_{\la,m}) -
			L(f_*)
			&\lesssim 
			\frac{\ln^2(1 / \lambda)}{n}+\sqrt{\frac{\alpha\mathcal{A}(\la)}{\lambda}}
			+\Big(\frac{\log(3/\delta)}{n}\Big)^{\frac{1}{2-\theta}}+\frac{\log(3/\delta)}{n} 
			\sqrt{\frac{\mathcal A(\la)}{\lambda}}+{\cal A}(\la). 
		\end{align*}
\end{proof}

\section{Proofs of Section~\ref{sec:other}}
\subsection{Square loss}\label{app: square}
We report in this section the proofs of Theorem~\ref{thm:square loss}.
As mentioned above, in the case where $w_*$ does not exists, the assumption of sub-gaussianity is necessary to get fast rates:

	\begin{proof}[Proof of Theorem~\ref{thm:square loss}]
		The proof follows the one of Theorem \ref{thm: fast rate A(lambda)} in Appendix \ref{app:theorem 4} with some differences coming from the fact that we are working now with the square loss. Since Theorem~\ref{thm: adaptation thm 7.23} works also with locally Lipschitz loss functions we have:
		\begin{align}
			\lambda\Vert \wh{\beta}_{\la,m} \Vert^2+L(\wh{\beta}_{\la,m}^{cl})-&L(f_*) \leq 7(\lambda\Vert \mathcal{P}_{\BB_m} w_\la\Vert^2 +L(\mathcal{P}_{\BB_m} w_\la)-L(f_*) )+K_3\frac{a^{2p}}{\lambda^p n}+2\frac{72V\log(3/\delta)}{n}+\nonumber\\
			&\;\;\;\;
			+16GK_1\nor{w_\la}\frac{(C\sqrt{2\tr \Sigma}+\E \nor{X}) (2+\log(3/\delta))}{n} 
			\nonumber\\
			&=7(L_\la(\mathcal{P}_{\BB_m} w_\la)-L_\la(w_\la)+L_\la(w_\la)-L(f_*) )+K_3\frac{a^{2p}}{\lambda^pn}+ \nonumber\\
			&\;\;\;\;+2\frac{72V\log(3/\delta)}{n}+16GK_1\frac{(C\sqrt{2\tr \Sigma}+\E \nor{X}) (2+\log(3/\delta))}{n} 
			\sqrt{\frac{\mathcal A(\la)}{\lambda}}\nonumber\\
			&=7\mathcal{A}(\lambda) +7(L_\la(\mathcal{P}_{\BB_m} w_\la)-L_\la(w_\la))+K_3\frac{a^{2p}}{\lambda^p n}+2\frac{72V\log(3/\delta)}{n}+\nonumber\\
			&\;\;\;\;+16GK_1\frac{(C\sqrt{2\tr \Sigma}+\E \nor{X}) (2+\log(3/\delta))}{n} \sqrt{\frac{\mathcal A(\la)}{\lambda}}
		\end{align}
		Using the fact that $L_\la$ is quadratic and expanding around the the minimum $w_\la$ we have
		\begin{equation}
			L_\la(\mathcal{P}_m w_\la)-L_\la(w_\la)=\|(\Sigma+\alpha)^{1/2}(I-\mathcal{P}_m)w_\la\|^2
		\end{equation}
		Using Lemma~\ref{lem:id_min_proj_lev_subgauss} we get the result
		\begin{align}
			\lambda\Vert \wh{\beta}_{\la,m} \Vert^2+L(\wh{\beta}_{\la,m}^{cl})-&L(f_*) \leq 7\mathcal{A}(\lambda) +7\|(\Sigma+\alpha)^{1/2}(I-\mathcal{P}_m)w_\la\|^2+K_3\frac{a^{2p}}{\lambda^p n}+2\frac{72V\log(3/\delta)}{n}+\nonumber\\
			&\;\;\;\;+16GK_1\frac{(C\sqrt{2\tr \Sigma}+\E \nor{X}) (2+\log(3/\delta))}{n} \sqrt{\frac{\mathcal A(\la)}{\lambda}}\nonumber\\
			&\lesssim 7\mathcal{A}(\lambda) +7\alpha\frac{\mathcal{A}(\la)}{\la}+K_3\frac{a^{2p}}{\lambda^p n}+2\frac{72V\log(3/\delta)}{n}+\nonumber\\
			&\;\;\;\;+16GK_1\frac{(C\sqrt{2\tr \Sigma}+\E \nor{X}) (2+\log(3/\delta))}{n} \sqrt{\frac{\mathcal A(\la)}{\lambda}}
		\end{align}
		Furthermore, if there exists $r \in (0,1]$ such that  
		${{\cal  A}(\la)\lesssim \la^{r}}$, then with the choice for ALS sampling
		\begin{align*}
			& \lambda \asymp
			n^{-\min\{\frac{2}{r+1},\frac{1}{r+p}\}} \\
			&\alpha \asymp n^{-\min\{\frac{2}{r+1},\frac{1}{r+p}\}}\\
			&m \gtrsim  n^{\min\{\frac{2p}{r+1},\frac{p}{ r+p}\}}\log n
		\end{align*}
		with high probability
		\begin{align*}
			L(\wh{\beta}_{\lambda,m}^{cl})-L(f_*)
			&\lesssim 
			n^{-\min\{\frac{2r}{r+1},\frac{r}{r+p}\}}.
		\end{align*}
	\end{proof}
 \subsection{Logistic Loss}\label{app: logistic}
Since logistic loss is not clippable,  we prove how the modification of the
definition of the clipping in \eqref{eq: clipping_logistic} and the
similar treatment of the projection term, up to constants, between
square and logistic losses asymptotically lead to the same excess risk
bounds. We start adjusting the proof of
Theorem~\ref{thm: adaptation thm 7.23}.

As explained in subsection \ref{sec: logistic}, one has  $h_f(X)-h_f^{cl}(X)+\frac{1}{n}\geq 0$. Therefore we can simply rewrite the splitting of the error~\eqref{eq: split 7.20} as 
 \begin{equation}
 	\la\nor{\wh w_\la}^2+\E h_{\wh w_\la^{cl}}\leq (\la\nor{\wh w_0}^2+\E h_{\wh w_0})+(\wh \E   h_{\wh w_0}-\E  h_{\wh w_0})+(\E h_{\wh w_\la^{cl}}-\wh \E  h_{\wh w_\la^{cl}})+\frac{1}{n}.
 \end{equation}
Clearly last term $1/n$ does not spoil the rate and we can proceed as
for square loss:
 \begin{align}
 	\lambda\Vert \wh{\beta}_{\la,m} \Vert^2+L(\wh{\beta}_{\la,m}^{cl})-&L(f_*) \leq 7(\lambda\Vert \mathcal{P}_{\BB_m} w_\la\Vert^2 +L(\mathcal{P}_{\BB_m} w_\la)-L(f_*) )+K_3\frac{a^{2p}}{\lambda^p n}+\frac{144V\log(3/\delta)}{n} +\nonumber\\
 	&\;\;\;\;
 +16GK_1\nor{w_\la}\frac{(C\sqrt{2\tr \Sigma}+\E \nor{X}) (2+\log(3/\delta))}{n} +\frac 1 n
 	\nonumber\\
 	&=7(L_\la(\mathcal{P}_{\BB_m} w_\la)-L_\la(w_\la)+L_\la(w_\la)-L(f_*) )+K_3\frac{a^{2p}}{\lambda^pn}+\frac{144V\log(3/\delta)}{n}+ \nonumber\\
 	&\;\;\;\;+16GK_1\frac{(C\sqrt{2\tr \Sigma}+\E \nor{X}) (2+\log(3/\delta))}{n}
 	\sqrt{\frac{\mathcal A(\la)}{\lambda}}+ \frac 1 n\nonumber\\
 	&=7\mathcal{A}(\lambda) +7(L_\la(\mathcal{P}_{\BB_m} w_\la)-L_\la(w_\la))+K_3\frac{a^{2p}}{\lambda^p n}+\frac{144V\log(3/\delta)}{n}+\nonumber\\
 	&\;\;\;\;+16GK_1\frac{(C\sqrt{2\tr \Sigma}+\E \nor{X}) (2+\log(3/\delta))}{n} \sqrt{\frac{\mathcal A(\la)}{\lambda}}+\frac 1 n
 \end{align}
 To deal with the projection term $L_\la(\mathcal{P}_{\BB_m} w_\la)-L_\la(w_\la)$ we do a Taylor expansion
 \begin{equation}
 	\label{eq: taylor logistic}
 	L_\la(\mathcal{P}_{\BB_m} w_\la)-L_\la(w_\la)=\frac{1}{2}\langle (HL)(w')(\mathcal{P}_{\BB_m}w_\la-w_\la),(\mathcal{P}_{\BB_m}w_\la-w_\la)\rangle
 \end{equation}
 where $w'=w_\la +t(\mathcal{P}_{\BB_m}w_\la-w_\la)$ with $t\in [0,1]$ and using the fact that $\nabla L_\la(w_\la)=0$. We can find the expression of the Hessian $H$ of $L$ in $w\in \X$ exploiting its definition
 \begin{align}
 	\label{eq: hessian logistic}
 	\langle (HL)(w) v, v\rangle&=\frac{d^{2}}{d t^{2}} L(w+t v)\arrowvert_{t=0}=\frac{d}{d t}\E\left[\ell'(\langle w+tv, X\rangle, Y) \langle v, X\rangle\right]\arrowvert_{t=0} \nonumber\\
 	&=\E\left[\ell''(\langle w+tv, X\rangle, Y) (\langle v, X\rangle)^2\right]\arrowvert_{t=0}\leq M\E \left[ \langle v, X\rangle^2\right]
 \end{align}
 where $M=\sup_{\tau\in\R, y\in \Y} \ell''(\tau, y)$ and $v\in\X$. For the logistic loss we have
 $$
 \ell''(\tau, y)=\sigma (y\tau)(1-\sigma(y\tau))\leq \frac{1}{4},\qquad \forall \tau\in \R, y\in\Y
 $$
 where $\sigma(\cdot)$ is the sigmoid which is upper bounded by $1$. So combining this result with \eqref{eq: hessian logistic} and considering $L_\la(\cdot)=L(\cdot)+\la\nor{\cdot}^2$ we get
 $$(HL_\la)(w)\leq \Sigma_\la.$$
 Finally we can rewrite \eqref{eq: taylor logistic} as
 \begin{equation}
 	L_\la(\mathcal{P}_{\BB_m} w_\la)-L_\la(w_\la)\leq \frac{1}{2} \nor{\Sigma_\la^{1/2}(\mathcal{P}_{\BB_m}w_\la-w_\la)}^2
 \end{equation}
 and proceed exactly as in the case of the square loss (see appendix \ref{app: square}).
 
 \section{Entropy Numbers and Exponential Decay} \label{app: entropy}
 
We analyze here the main steps needed to obtain the results for exponential decay in Theorem~\ref{thm:1} and Theorem~\ref{thm: fast rate A(lambda)}. 

 \subsection{Entropy numbers in Hilbert spaces}
 Let $\mathcal{H}$ and $\mathcal{K}$ be real Hilbert spaces. For all $n \in \mathbb{N}, n \geq 1$
 \begin{equation}
 	\label{eq: entropy1}
 \sup _{1 \leq k<\infty}\left(n^{-1 / k}\left(\Pi_{\ell=1}^{k} a_{\ell}(T)\right)^{1 / k}\right) \leq \varepsilon_{n}(T) \leq 14 \sup _{1 \leq k<\infty}\left(n^{-1 / k}\left(\Pi_{\ell=1}^{k} a_{\ell}(T)\right)^{1 / k}\right)
 \end{equation}
where $\eps_n(T)$ are the entropy numbers,
see (3.4.15) of \citep{carl1990entropy}.

Let $X$ be a random variable on a probability space $(\Omega, \mathcal{F}, \mathbb{P})$ taking value in a real Hilbert space $\mathcal{H}$ such that $\mathbb{E}\left[|\langle X, v\rangle|^{2}\right]$ is finite for all $v \in \mathcal{H}$. Define $$T: \mathcal{H} \rightarrow L_{2}(\Omega, \mathbb{P}) \quad T(v)(\omega)=\langle X(\omega), v\rangle$$
 so that $\Sigma=T^{*} T$ is (non-centered) covariance matrix. We assume that $\Sigma$ is a trace-class operator and the corresponding eigenvalues have an exponential decay
 $$
 \Sigma=\sum_{n=1}^{+\infty} \lambda_{n}(\Sigma) v_{n} \otimes v_{n} \quad \lambda_{n}(\Sigma) \simeq 2^{-2 a n}
 $$
 where $\left(v_{n}\right)_{n}$ is a base of $\mathcal{H}$.
 Since $\Sigma$ is trace-class, $S$ is compact, so that by \eqref{eq: entropy1}
 $$
 e_{n}(T) \simeq \sup _{1 \leq k<\infty} 2^{-(n-1) / k}\left(\Pi_{\ell=1}^{k} a_{n}(T)\right)^{1 / k}
 $$
 with $e_n(T)=\eps_{2^{n-1}}(T)$ the (dyadic) entropy numbers and where by \citep{carl1990entropy}
 $$
 a_{n}(T)=a_{n}(|T|)=\lambda_{n}(|T|)=\lambda_{n}(\Sigma)^{1 / 2} \simeq 2^{-a n}.
 $$
We have
 $$
 2^{-(n-1) / k}\left(\Pi_{\ell=1}^{k} 2^{-a \ell}\right)^{1 / k}=2^{-\left(\frac{n-1}{k}+\frac{a(k+1)}{2}\right)}.
 $$
 Observe that the minimum on $(0,+\infty)$ of the function
 $$
 f(x)=\left(\frac{n-1}{x}+\frac{a x}{2}\right)
 $$
 is $f(\sqrt{2(n-1) / a})=\sqrt{2 a(n-1)},$ then
 $$
 e_{n}(T) \simeq 2^{-\sqrt{a n}}.
 $$
 
 \subsection{Entropy numbers of $\F_r$}
 Given the above calculation we want to upper bound the entropy number of $\mathcal{\F}_{r}$, we recall here some definitions:
 $$\mathcal{\X}_{r}:=\left\{f \in \mathcal{\X}: \Upsilon(f)+L(f^{cl})-L(f_{*}) \leq r\right\} \qquad r>r^*$$
 $$\mathcal{F}_r:=\left\{\ell \circ f^{cl}-\ell \circ f_*: f \in \mathcal{\X}_{r}\right\} \qquad r>r^*$$
 Using the above discussion we obtain
 $$e_i(\F_r)
 \leq G e_i(\X_r)\leq G\sqrt{\frac{r}{\la}}e_i(\mathcal{B}_\X)=G\sqrt{\frac{r}{\la}} 2^{-c \sqrt i}
 $$
 
 \subsection{Bound the Rademacher Complexity of $\F_r$} 
 Now we are ready to upper bound the empirical Rademacher Complexity $\wh{\mathfrak{R}}$ of $\F_r$:
 \begin{lemma}
 	\begin{equation}
 		\wh{\mathfrak{R}}\left(\F_r\right) \leq  \sqrt{\frac{\ln 16}{n}}\log\left(\frac{1}{\la}\right)(3\rho+2c_3\sqrt{r})
 	\end{equation}
 	where $\rho=\sup _{f \in \mathcal{F}_{r}}\|f\|_{L_{2}(D)}$ and $\|f\|_{L_{2}(D)}:=\left(\frac{1}{m} \sum_{i} f^{2}\left(x_{i}\right)\right)^{1 / 2}$.
 	\begin{proof}
 		Using Theorem 7.13 in \citep{steinwart2008support}, we have
 		$$
 		\wh{\mathfrak{R}}\left(\F_r\right) \leq \sqrt{\frac{\ln 16}{n}}\left(\sum_{i=1}^{\infty} 2^{i / 2} e_{2^{i}}\left(\F_r \cup\{0\},\|\cdot\|_{L_{2}(D)}\right)+\sup _{f \in \F_r}\|f\|_{L_{2}(D)}\right)
 		$$
 		It is easy to see that $e_{i}\left(\F_r \cup\{0\}\right) \leq e_{i-1}\left(\F_r\right)$ and $e_{0}\left(\F_r\right) \leq \sup _{f \in \F_r}\|f\|_{L_{2}(D)} .$ Since $e_{i}\left(\F_r\right)$ is
 		a decreasing sequence with respect to $i,$ together with the lemma above, we know that
 		$$
 		e_{i}\left(\F_r\right) \leq \min \left\{\sup _{f \in \F_r}\|f\|_{L_{2}(D)}, \sqrt{\frac{2 r}{\lambda}} 2^{-c\sqrt{i}}\right\}
 		$$
 		Even though the second one decays exponentially, it may be much greater than the first term when $2 r / \lambda$ is huge for small $i$ s. To achieve the balance between these two bounds, we use the first one for first $T$ terms in the sum and the second one for the tail. So
 		$$
 		\wh{\mathfrak{R}}\left(\F_r\right) \leq \sqrt{\frac{\ln 16}{n}}\left(\sup _{f \in \F_r}\|f\|_{L_{2}(D)} \sum_{i=0}^{T-1} 2^{i / 2}+\sqrt{\frac{2 r}{\lambda}} \sum_{i=T}^{\infty} 2^{i / 2} 2^{-c\sqrt{2^i-1}}\right)
 		$$
 		The first sum is $\frac{\sqrt{2}^{T}-1}{\sqrt{2}-1} .$ When $T$ is large enough, the second sum is upper bounded by the integral
 		\begin{align}
 			\int_{T}^{\infty} 2^{x / 2} 2^{-c\sqrt{2^i-1}} \mathrm{~d} x &\leq \int_{T}^{\infty} 2^{x / 2} 2^{-c_2\sqrt{2^i}} \mathrm{~d} x \leq \frac{2^{-c_2\sqrt{2^{T}}+1}}{c_2 \log ^{2}(2)}\\
 			&\leq c_3 2^{-c_2\sqrt{2^{T}}}
 		\end{align}
 		To make the form simpler, we bound $\frac{\sqrt{2}^{T}-1}{\sqrt{2}-1}$ by $3 \cdot 2^{T / 2}$, and denote $\sup _{h \in \F_r}\|h\|_{L_{2}(D)}$ by $\rho$. Taking $T$ to be
 		$$
 		\log _{2}\left(c_4^2\log _{2}^2\left(\frac{1}{\lambda}\right)\right),
 		$$
 		with $c_4$ such that $c_2c_4>1/2$, 
 		we get the upper bound of the form
 		$$
 		\wh{\mathfrak{R}}\left(\F_r\right) \leq \sqrt{\frac{\ln 16}{n}}\left(3\rho\log\left(\frac{1}{\la}\right)+c_3\sqrt{\frac{2 r}{\lambda}} \la^{c_2c_4} \right)\leq \sqrt{\frac{\ln 16}{n}}\log\left(\frac{1}{\la}\right)(3\rho+2c_3\sqrt{r})
 		$$
 	\end{proof}
 \end{lemma}
Now we can directly compute the upper bound for the population Rademacher Complexity $\mathfrak{R}\left(\F_r\right)$ by taking expectation over $D \sim P^{m}$:
 \begin{lemma}
 	\begin{equation}
 		\mathfrak{R}\left(\F_r\right) \leq C_{1} \sqrt{\frac{V}{n}}\log _{2}\left(\frac{1}{ \lambda}\right) \sqrt{r}+C_{2} \frac{\log^2_{2}(1 / \lambda)}{n}
 	\end{equation}
 	where $C_{1}$ and $C_{2}$ are two absolute constants.
 	\begin{proof}
 		\begin{equation}
 			\mathfrak{R}\left(\F_r\right) =\E [\wh{\mathfrak{R}}\left(\F_r\right)] \leq \sqrt{\frac{(\ln 16)}{n}} \log _{2} \left(\frac{1}{\lambda}\right) \left(3 \mathbb{E} \sup _{f \in \F_r}\|f\|_{L_{2}(D)}+2c_3 \sqrt{r}\right)
 		\end{equation}
 		By Jensen's inequality and Corollary A.8.5 in \citep{steinwart2008support},
 		%
 		we have
 		$$
 		\begin{aligned}
 			\mathbb{E} \sup _{f \in \F_r}\|f\|_{L_{2}(D)} & \leq\left(\mathbb{E} \sup _{f \in \F_r}\|f\|_{L_{2}(D)}^{2}\right)^{1 / 2} 
 			\leq\left(\mathbb{E} \sup _{f \in \F_r} \frac{1}{m} \sum_{i=1}^{m} f^{2}\left(x_{i}, y_{i}\right)\right)^{1 / 2} \\
 			& \leq\left(\sigma^{2}+8 \mathfrak{R}\left(\F_r\right)\right)^{1 / 2}
 		\end{aligned}
 		$$
 		where $\sigma^{2}:=\mathbb{E} f^{2}$. When $\sigma^{2}>\mathfrak{R}\left(\F_r\right),$ we have
 		\begin{align}
 			\mathfrak{R}\left(\F_r\right)& \leq \sqrt{\frac{\ln 16}{n}} \log _{2} \left(\frac{1}{\lambda}\right)(9 \sigma+2c_3 \sqrt{r})\\
 			&	\leq \sqrt{\frac{\ln 16}{n}} \log _{2} \left(\frac{1}{\lambda}\right)(9\sqrt{V r^\theta}+2c_3 \sqrt{r}) \\
 			&\leq c_5 \sqrt{\frac{V}{n}} \log _{2}\left(\frac{1}{\lambda}\right)\sqrt{r}
 		\end{align}
 		The second inequality is because $\mathbb{E} f^{2} \leq V (\mathbb{E} f)^\theta$ and $\mathbb{E} f \leq r$ for $f \in \F_r$.
 		When $\sigma^{2} \leq \mathfrak{R}\left(\F_r\right),$ we have
 		$$
 		\begin{aligned}
 			\mathfrak{R}\left(\F_r\right) & \leq  \sqrt{\frac{\ln 16}{n}} \log _{2} \left(\frac{1}{\lambda}\right)\left(9 \sqrt{\mathfrak{R}\left(\F_r\right)}+2c_3 \sqrt{r}\right) \\
 			& \leq (9+2c_3) c_3  \sqrt{\frac{\ln 16}{n}} \log _{2} \left(\frac{1}{\lambda}\right) \sqrt{r}+(9+2c_3)^{2} \frac{(\ln 16) \log^2_{2}(1 / \lambda)}{n}
 		\end{aligned}
 		$$
 		The last inequality can be obtained by dividing the formula into two cases, either $\mathfrak{R}\left(\F_r\right)<r$ or $\mathfrak{R}\left(\F_r\right) \geq r$ and then take the sum of the upper bounds of two cases.
 		Combining all these inequalities, we finally obtain an upper bound
 		$$
 		\mathfrak{R}\left(\F_r\right) \leq C_{1} \sqrt{\frac{V}{n}}\log _{2}\left(\frac{1}{ \lambda}\right) \sqrt{r}+C_{2} \frac{\log^2_{2}(1 / \lambda)}{n}
 		$$
 		where $C_{1}$ and $C_{2}$ are two absolute constants.
 	\end{proof}
 \end{lemma}

\section{Known results}
\label{app: known results}
For sake of completeness we recall the following known results, we
freely use in the paper. 

The following two results provide a tight bound on the effecticbe
dimension under the assumption of  a
polynomial decay or an exponential decay of the eigenvalues $\sigma_j$
of $\Sigma$ from \citep{caponnetto2007optimal}. We report the
proofs for sake of completeness.
\begin{proposition}[Proposition 3 in \citep{caponnetto2007optimal}]
            \label{prop: eig polynom decay}
If for some $\gamma\in\R^+$ and  $1<\beta<+\infty$
\[
\sigma_i \leq \gamma i^{-\beta}
\]
then 
				\begin{equation}
				d_\alpha\leq \gamma\frac{\beta}{\beta-1}\alpha^{-1/\beta}
                              \end{equation}
		\begin{proof}
			Since the function $\sigma/(\sigma+\alpha)$ is increasing in $\sigma$ and using the spectral theorem $\Sigma=UDU^*$ combined with the fact that $\tr (UDU^*)=\tr (U(U^* D))=\tr D$
			\begin{equation}
			d_\alpha=\tr (\Sigma(\Sigma+\alpha I)^{-1})=\sum_{i=1}^\infty \frac{\sigma_i}{\sigma_i+\alpha}\leq \sum_{i=1}^\infty \frac{\gamma}{\gamma+i^\beta\alpha}
			\end{equation}
			The function $\gamma/(\gamma+x^\beta\alpha)$ is positive and decreasing, so 
			\begin{align}
			d_\alpha&\leq \int_0^\infty\frac{\gamma}{\gamma+x^\beta\alpha}dx\nonumber\\
			&=\alpha^{-1/\beta}\int_0^\infty\frac{\gamma}{\gamma+\tau^\beta}d\tau\nonumber\\
			&\leq \gamma\frac{\beta}{\beta-1}\alpha^{-1/\beta}
			\end{align}
			since $\int_0^\infty(\gamma+\tau^\beta)^{-1}\leq \beta/(\beta-1)$.
		\end{proof}
	\end{proposition}
	
	\begin{proposition}[Exponential eigenvalues decay]\label{prop:Exponential eigenvalues decay}
		\label{prop: eig exp decay}
                  If for some $\gamma,\beta \in\R^+
                  \sigma_i\leq \gamma e^{-\beta i}$ then
		\begin{equation}
		d_\alpha\leq \frac{\ln(1+\gamma/\alpha)}{\beta}
		\end{equation}
		\begin{proof}
			\begin{align}
			\label{exp decay}
			d_\alpha&=\sum_{i=1}^\infty \frac{\sigma_i}{\sigma_i+\alpha}=\sum_{i=1}^\infty \frac{1}{1+\alpha/\sigma_i}\leq\sum_{i=1}^\infty \frac{1}{1+\alpha' e^{\beta i}}\leq\int_0^{+\infty} \frac{1}{1+\alpha' e^{\beta x}}dx
			\end{align}
			where $\alpha'=\alpha/\gamma$. Using the change of variables $t=e^{\beta x}$ we get
			\begin{align}
			(\ref{exp decay})&=\frac{1}{\beta}\int_1^{+\infty} \frac{1}{1+\alpha' t}\;\frac{1}{t}dt=\frac{1}{\beta}\int_1^{+\infty}\Big[\frac{1}{t}- \frac{\alpha'}{1+\alpha' t}\Big]dt=\frac{1}{\beta}\Big[ \ln t -\ln(1+\alpha't)\Big]_1^{+\infty}\nonumber\\
			&=\frac{1}{\beta}\Big[ \ln \Big(\frac{t}{1+\alpha't}\Big)\Big]_1^{+\infty}=\frac{1}{\beta}\Big[\ln(1/\alpha')+\ln(1+\alpha')\Big]
			\end{align}
			So we finally obtain
			\begin{equation}
			d_\alpha\leq \frac{1}{\beta}\Big[\ln(\gamma/\alpha)+\ln(1+\alpha/\gamma)\Big]=\frac{\ln(1+\gamma/\alpha)}{\beta}
			\end{equation}
		\end{proof}
	\end{proposition}
The following result provides a bound on the entropy number and it is  the content of Theorem 15 in
\citep{steinwart2009optimal}.  We recall that, given a bounded operator $A$ between two
Hilbert spaces $\H_1$ and $H_2$, we denote by $e_j(A)$ the (dyadic) entropy
numbers of $A$ and by $\wh{P}_\X=\frac{1}{n}\sum_{i=1}^n\delta_{x_i}$ the empirical (marginal)
measure associated with the input data $x_i,\ldots,x_n$. Regard the
data matrix $\wh{X}$ as the inclusion operator $\operatorname{id}:\X\to L_2(\wh{P})$
\[(\operatorname{id} w)(x_i)=\scal{w}{x_i} \qquad i=1,\ldots,n\]
\begin{lemma}\label{entropy}  
Let $p\in (0,1)$.  Then
	\begin{equation}
	\E_{\wh{P}} [e_j(\operatorname{id}:\X\to L_2(\wh{P}) )]\sim j^{-\frac{1}{2p}} 
	\end{equation}
	if and only if
	\begin{equation}
	\label{decad lemma equiv}
	\sigma_j \sim j^{-\frac{1}{p}}
	\end{equation}
      \end{lemma}

As regard results in Section~\ref{sec: 0-1 loss}, from \citep{bartlett2006convexity} we report the following lemma:
\begin{lemma}
	\label{lem: from0-1_to_surr}
	For any nonnegative loss function $\phi$, any measurable $f: \X \rightarrow \mathbb{R},$ and any probability distribution on $\X \times\{\pm 1\}$
	$$
	\psi\left(L_{0-1}(f)-L_{0-1}^{*}\right) \leq L_{\phi}(f)-L_{\phi}^{*}.
	$$
	In particular, for square,  hinge and logistic losses we can write 
	\begin{itemize}
		\item square loss: $L_{0-1}(f)-L_{0-1}^*\leq \sqrt{L_{square}(f)-L_{square}^*}$,
		\item hinge loss: $L_{0-1}(f)-L_{0-1}^*\leq L_{hinge}(f)-L_{hinge}^*$,
		\item logistic loss: $L_{0-1}(f)-L_{0-1}^*\leq 2\sqrt{L_{logistic}(f)-L_{logistic}^*}$.
	\end{itemize}
\end{lemma}
Under the assumption of low noise we can improve the above bounds in Lemma~\ref{lem: from0-1_to_surr}:
\begin{lemma}[Theorem 3 in \citep{bartlett2006convexity}]
	\label{lem: class risk wirh gamma}
	Suppose that $P$ has noise exponent $0\leq\gamma \leq 1$, and that $\phi$ is classification-calibrated (which is the case for square, hinge and logistic losses). Then there is a $c>0$ such that for any $f: \mathcal{X} \rightarrow \mathbb{R}$
	$$
	c\left(L_{0-1}(f)-L_{0-1}^*\right)^{\gamma} \psi\left(\frac{\left(L_{0-1}(f)-L_{0-1}^*\right)^{1-\gamma}}{2 c}\right) \leq L_\phi(f)-L_\phi^{*}
	$$
	where $\psi(x)=x^2$ when $\phi$ is the square loss, $\psi(x)=x$ when $\phi$ is the hinge loss and $\psi(x)\geq \frac{x}{2}$ when $\phi$ is the logistic loss.
\end{lemma} 
We copy also this results from \citep{steinwart2008support}, linking the variance bound in Assumption~\ref{ass:berstein} with low noise condition in Assumption~\ref{ass: low-noise} for hinge loss:
\begin{lemma}\label{lem: lownoise_bernst}
	[Theorem 8.24 \citep{steinwart2008support}] (Variance bound for the hinge loss). Let $\mathrm{P}$ be a distribution on $X \times Y$ that has noise exponent $\gamma \in[0, 1] .$ Moreover, let $f_{*}: X \rightarrow[-1,1]$ be a fixed Bayes decision function for the hinge loss $\ell$. Then, for all measurable $f: X \rightarrow \mathbb{R},$ we have
	$$
	\mathbb{E}\left(\ell \circ f^{cl}-\ell \circ f_{*}\right)^{2} \leq 6 c\left(\mathbb{E}\left(\ell \circ f^{cl}-\ell \circ f_{*}\right)\right)^\gamma
	$$
	where $c$ is the constant appearing in \eqref{ass:noise cond}.
\end{lemma}

\section{Experiments: datasets and tuning}\label{appexp}
Here we report  further information on the used datasets and the set up used for parameter tuning, plus some additional tables of results.
\begin{table}[h]
	\caption[caption table]{Comparison between ALS and uniform sampling. To achieve similar accuracy, uniform sampling usually requires larger $m$ than ALS sampling. Therefore, even if it does not need leverage scores computations, Nystr\"om-Pegasos with uniform sampling can be more expensive both in terms of memory and time (in seconds). }
	\centering
	\label{tab:comparison}
	\begin{tabular}{lllllllll}
		\toprule
		\multicolumn{1}{c}{}   &\multicolumn{3}{c} {Nystr\"om-Pegasos (ALS)} &\multicolumn{3}{c} {Nystr\"om-Pegasos (Uniform)}             \\ 
		\cmidrule(r){1-1}	\cmidrule(r){2-4}\cmidrule(r){5-7}
		Datasets        & c-err & t train  & t pred & c-err &t train & t pred\\
		
		\cmidrule(r){1-1}	\cmidrule(r){2-4}\cmidrule(r){5-7}
		SUSY  &$20.0\% \pm 0.2 \%$&$ 608\pm 2$& $134\pm 4$ &$20.1\% \pm 0.2 \%$&$ 592\pm 2$& $129\pm 1$\\
		Mnist bin     & $2.2\% \pm 0.1 \%$ &$ 1342\pm 5 $& $ 491 \pm 32 $ &  $2.3\% \pm 0.1 \%$&$ 1814\pm 8 $& $ 954 \pm 21 $ \\
		Usps         & $3.0\%\pm 0.1  \%$ & $ 19.8 \pm 0.1 $ &$7.3 \pm 0.3 $ & $3.0\%\pm 0.2  \%$ & $ 66.1 \pm 0.1 $ &$48 \pm 8 $ \\
		Webspam      & $1.3\% \pm 0.1\%$    & $2440 \pm 5$& $376 \pm 18$& $1.3\% \pm 0.1\%$    & $4198 \pm 40$& $1455 \pm 180$ \\
		a9a  &$15.1\%\pm 0.2\%$ & $29.3\pm 0.2$& $1.5\pm 0.1$ &$15.1\%\pm 0.2\%$ & $30.9\pm 0.2$& $3.2\pm 0.1$\\
		CIFAR  &$19.2\%\pm 0.1\%$ & $2408\pm 14$& $820\pm 47$ &$19.0\%\pm 0.1\%$ & $2168\pm 19$& $709\pm 13$\\
		\bottomrule
	\end{tabular}
\end{table}

For Nystr\"om SVM with Pegaos we tuned the kernel parameter $\sigma$ and $\lambda$ regularizer with a simple grid search ($\sigma\in [0.1,20]$, $\lambda\in [10^{-8},10^{-1}]$, initially with a coarse grid and then more refined around the best candidates).  An analogous procedure has been used for K-SVM with its parameters $C$ and $\gamma$. 
The details of the considered data sets and the chosen parameters for our algorithm in Table ~\ref{tab:results} and ~\ref{tab:comparison} are the following: \\
\textbf{SUSY} (Table ~\ref{tab:results} and ~\ref{tab:comparison}, $n=5\times10^6$, $d=18$): we used a Gaussian kernel with $\sigma=4$, $\lambda=3\times 10^{-6}$ and $m_{ALS}=2500$, $m_{uniform}=2500$.\\
\textbf{Mnist binary} (Table ~\ref{tab:results} and ~\ref{tab:comparison}, $n=7\times10^4$, $d=784$): we used a Gaussian kernel with $\sigma=10$, $\lambda=3\times 10^{-6}$ and $m_{ALS}=15000$, $m_{uniform}=20000$.\\
\textbf{Usps} (Table ~\ref{tab:results} and ~\ref{tab:comparison}, $n=9298$, $d=256$): we used a Gaussian kernel with $\sigma=10$, $\lambda=5\times 10^{-6}$ and $m_{ALS}=2500$, $m_{uniform}=4000$.\\
\textbf{Webspam} (Table ~\ref{tab:results} and ~\ref{tab:comparison}, $n=3.5\times 10^5$, $d=254$): we used a Gaussian kernel with $\sigma=0.25$, $\lambda=8\times 10^{-7}$ and $m_{ALS}=11500$, $m_{uniform}=20000$.\\
\textbf{a9a} (Table ~\ref{tab:results} and ~\ref{tab:comparison}, $n=48842$, $d=123$): we used a Gaussian kernel with $\sigma=10$, $\lambda=1\times 10^{-5}$ and $m_{ALS}=800$, $m_{uniform}=1500$.\\
\textbf{CIFAR} (Table ~\ref{tab:results} and ~\ref{tab:comparison}, $n=6\times 10^4$, $d=400$): we used a Gaussian kernel with $\sigma=10$, $\lambda=2\times 10^{-6}$ and $m_{ALS}=20000$, $m_{uniform}=20000$.\\
In figure~\ref{fig: cov} we visualize the eigenvalues decay of the empirical covariance matrix for some of the datasets considered.

\begin{table}[h]
	\caption[caption table]{Comparison between Nystr\"om-Pegasos (hinge loss) and Nystr\"om-KRR (square loss) when using uniform sampling. We report the respective classification errors fixing the number of Nystr\"om centers. }
	\centering
	\label{tab:disc}
	\begin{tabular}{lllllllll}
		\toprule
		\multicolumn{1}{c}{}   &\multicolumn{2}{c} {Nystr\"om-Pegasos (Uniform)} &\multicolumn{2}{c} {Nystr\"om-KRR (Uniform)}             \\ 
		\cmidrule(r){1-1}	\cmidrule(r){2-3}\cmidrule(r){4-5}
		Datasets        & c-err   & $m$ & c-err  & $m$\\
		
		\cmidrule(r){1-1}	\cmidrule(r){2-3}\cmidrule(r){4-5}
		SUSY  &$20.1\%\pm 0.2 \%$& $2500$ &$19.8\%\pm 0.2\%$& $2500$\\
		Mnist bin     & $2.3\% \pm 0.1 \%$& $ 20000 $ &  $2.5\%\pm 0.1\%$& $ 20000$ \\
		Usps         & $3.0\%\pm 0.2  \%$  &$4000 $ & $3.1\%\pm0.1\%$  &$4000 $ \\
		Webspam    &$1.3\% \pm 0.1\%$    & $20000$  & $1.4\%\pm 0.1\%  $  & $20000$ \\
		a9a  &$15.1\%\pm 0.2\%$ &  $1500$ &$14.9\%\pm 0.1\%$ &  $1500$\\
		CIFAR  &$19.0\%\pm 0.1\%$ & $20000$ &$19.2\%\pm 0.1\%$ & $20000$\\
		\bottomrule
	\end{tabular}
\end{table}

\begin{figure*}[h!]
	\centering
	\includegraphics[width=7.5cm]{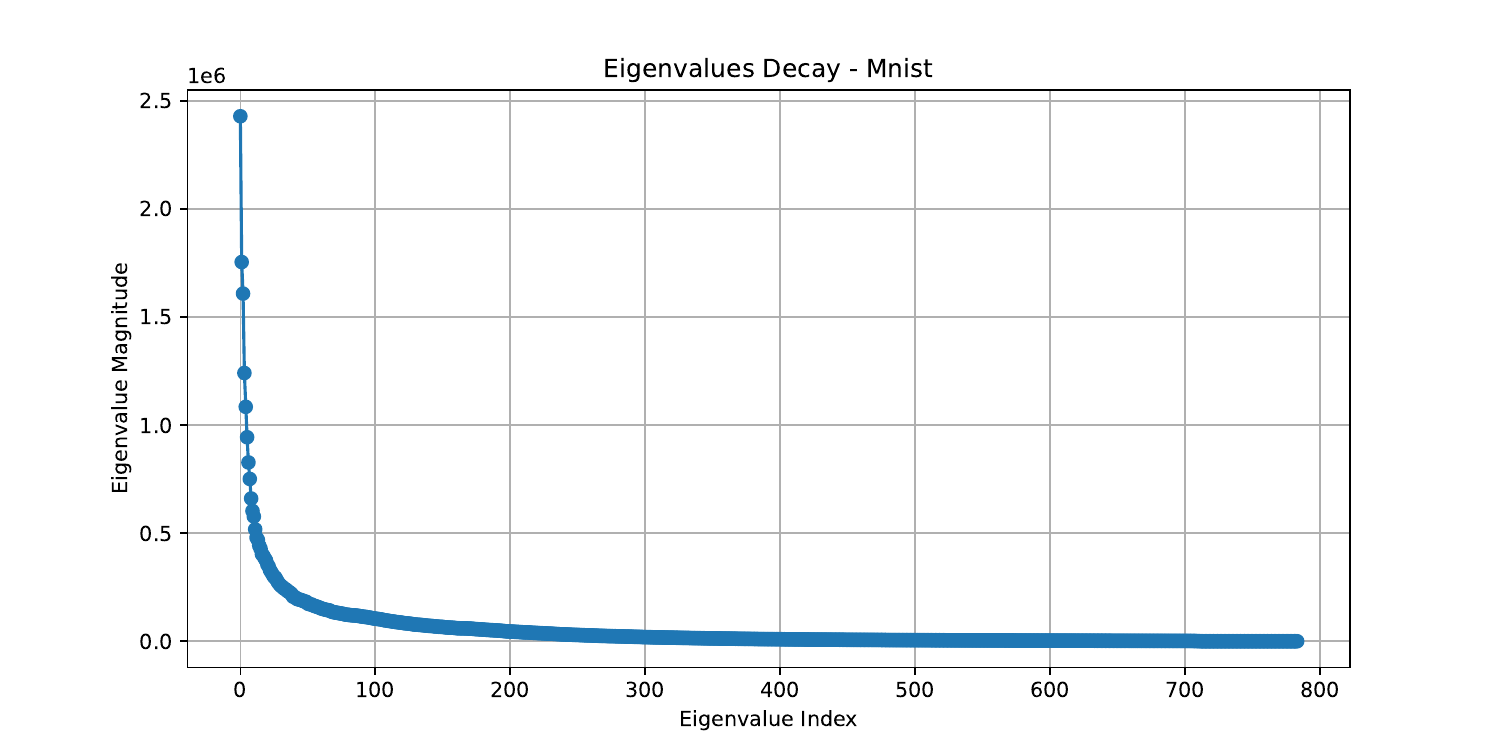}
	\includegraphics[width=7.5cm]{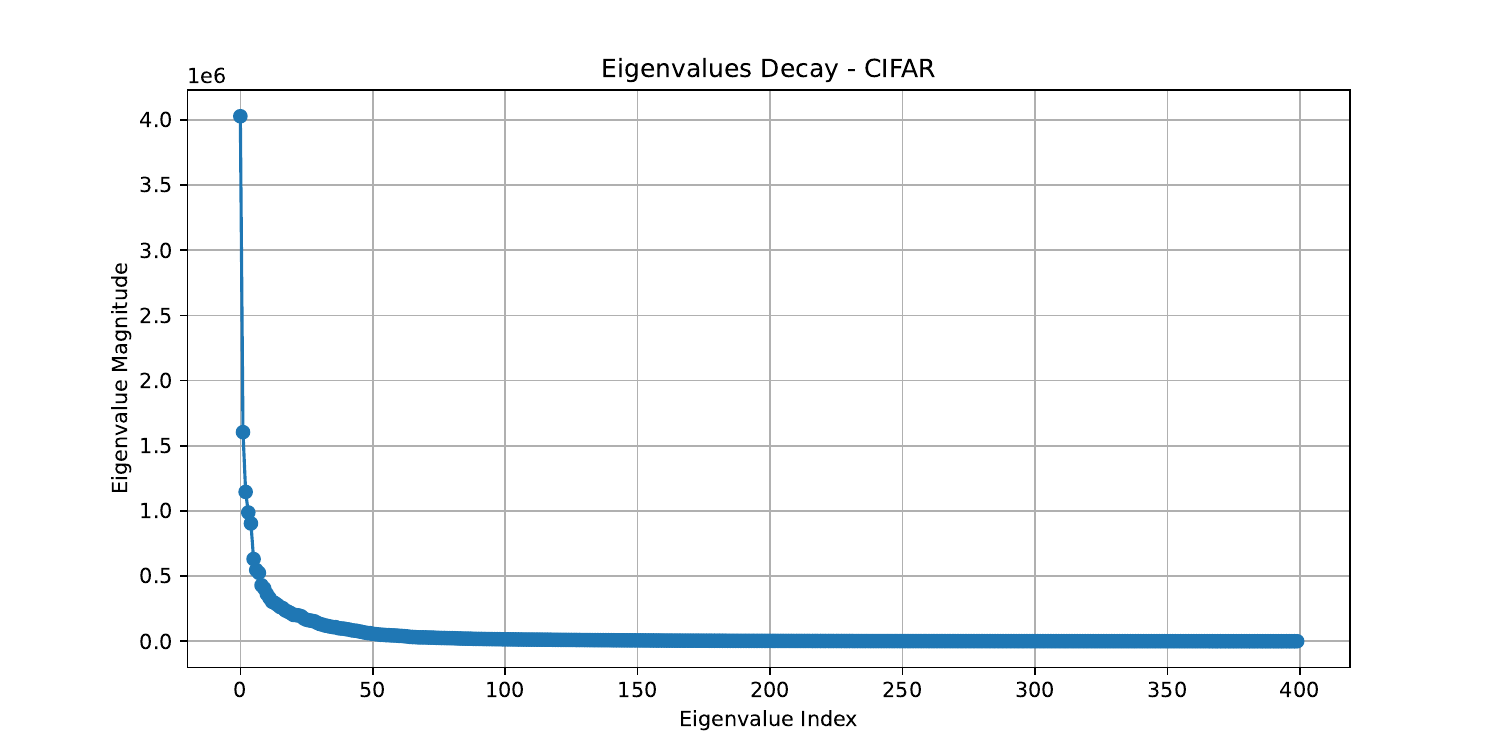}
	\includegraphics[width=7.5cm]{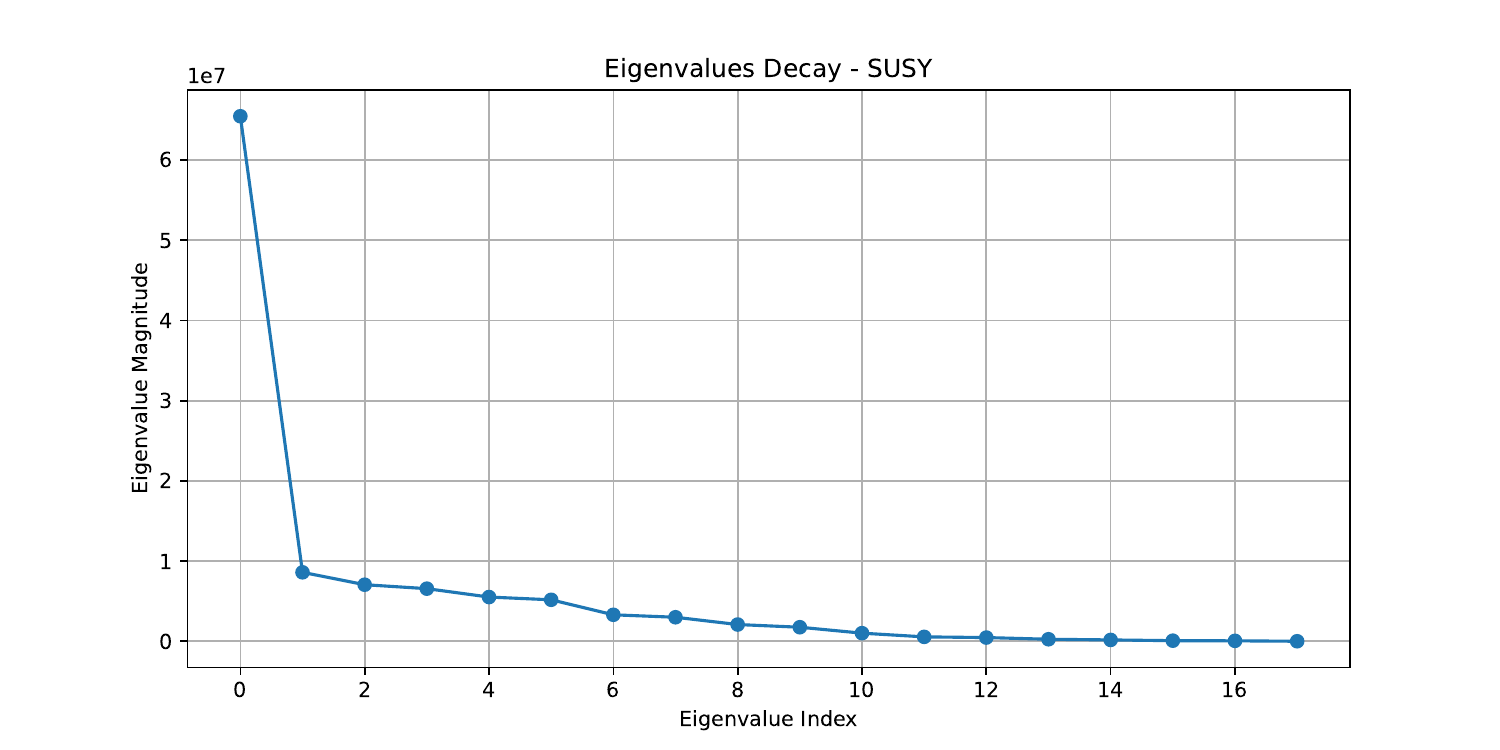}
	\caption{Eigenvalues decay of the empirical covariance matrix for Mnist binary, CIFAR and SUSY datasets.}
	\label{fig: cov}
\end{figure*}

\end{document}